\documentclass[12pt,letter,abstract=true,parskip=half]{scrartcl}
\usepackage[margin=1in]{geometry}
\usepackage{microtype} 
\usepackage[section]{placeins}

\usepackage{enumitem,etoolbox}
\setlist[itemize]{label={\small\textbullet}}
\setlist{noitemsep,nosep} 
\AtBeginEnvironment{itemize}{\setlength\parskip{0pt}}
\usepackage{booktabs}
\usepackage{makecell}

\usepackage{lmodern}
\usepackage[T1]{fontenc}
\usepackage[mathscr]{eucal}

\usepackage{tikz}
\usetikzlibrary{cd}
\usetikzlibrary{calc}
\usetikzlibrary{positioning}
\usetikzlibrary{shapes.geometric}
\newcommand{\smallbox}[1]{\scalebox{0.8}{#1}}
\newcommand{\smallhbox}[1]{\hbox{\smallbox{#1}}}

\usepackage[maxnames=10,style=alphabetic]{biblatex}
\renewbibmacro{in:}{}

\usepackage{hyperref}
\usepackage[capitalize,noabbrev]{cleveref}
\hypersetup{bookmarksnumbered=true,bookmarksopen=true}

\usepackage{amsmath,amssymb,bm,mathtools}
\newcommand{\Number}[1]{\mathbb{#1}}
\newcommand{\R}{\Number{R}}
\newcommand{\C}{\Number{C}}
\newcommand{\Bool}{\Number{B}}

\newcommand{\cat}[1]{\mathscr{#1}}
\newcommand{\opposite}{\mathrm{op}}
\newcommand{\oppositeTwo}{\mathrm{co}}
\newcommand{\CatName}[1]{\mathbf{#1}}
\newcommand{\Cat}{\CatName{Cat}}
\newcommand{\Set}{\CatName{Set}}
\newcommand{\Rel}{\CatName{Rel}}
\newcommand{\FinRel}{\CatName{FinRel}}
\newcommand{\Cpo}{\CatName{CPO}}
\newcommand{\Ab}{\CatName{Ab}}
\newcommand{\Vect}{\CatName{Vect}}
\newcommand{\Hilb}{\CatName{Hilb}}
\newcommand{\BiRel}{\CatName{BiRel}}
\newcommand{\DistBiRel}{\CatName{DistBiRel}}
\newcommand{\BoolMat}{\CatName{Mat}(\Bool)}
\newcommand{\VectRel}{\CatName{VectRel}}

\newcommand{\Copy}{\Delta}
\newcommand{\Merge}{\nabla}
\newcommand{\Create}{\square}
\newcommand{\Delete}{\lozenge}
\newcommand{\varCopy}{\blacktriangle}
\newcommand{\varMerge}{\blacktriangledown}
\newcommand{\varCreate}{\blacksquare}
\newcommand{\varDelete}{\blacklozenge}
\DeclareMathOperator{\Map}{Map}
\DeclareMathOperator{\RelC}{Rel}

\usepackage[heavycircles]{stmaryrd}
\usepackage{prftree}
\newcommand{\DL}[1]{\mathcal{#1}}
\newcommand{\given}[1][]{\:#1\vert\:}
\newcommand{\sem}[1]{\left\llbracket#1\right\rrbracket}
\newcommand{\theory}[1]{\mathbb{T}}
\newcommand{\tvec}[1]{\underline{#1}}
\newcommand{\pair}[2]{\langle#1,#2\rangle}
\newcommand{\copair}[3]{\mathop{\delta}\left(#1,#2,#3\right)}
\DeclareMathOperator{\Cl}{Cl}
\DeclareMathOperator{\Lang}{Lang}

\usepackage{amsthm}
\theoremstyle{plain}
\newtheorem*{theorem}{Theorem}
\newtheorem*{proposition}{Proposition}
\newtheorem*{lemma}{Lemma}
\theoremstyle{definition}
\newtheorem*{definition}{Definition}

\theoremstyle{remark}
\newtheorem*{remark}{Remark}

\title{Knowledge Representation in Bicategories of Relations}
\author{Evan Patterson \\
  {\large Department of Statistics, Stanford University}}
\date{\vspace{-1\baselineskip}}

\begin{document}
\maketitle

\begin{abstract}
  We introduce the \emph{relational ontology log}, or \emph{relational olog}, a
  knowledge representation system based on the category of sets and relations.
  It is inspired by Spivak and Kent's \emph{olog}, a recent categorical
  framework for knowledge representation. Relational ologs interpolate between
  ologs and description logic, the dominant formalism for knowledge
  representation today. In this paper, we investigate relational ologs both for
  their own sake and to gain insight into the relationship between the algebraic
  and logical approaches to knowledge representation. On a practical level, we
  show by example that relational ologs have a friendly and intuitive---yet
  fully precise---graphical syntax, derived from the string diagrams of monoidal
  categories. We explain several other useful features of relational ologs not
  possessed by most description logics, such as a \emph{type system} and a rich,
  flexible notion of \emph{instance data}. In a more theoretical vein, we draw
  on categorical logic to show how relational ologs can be translated to and
  from logical theories in a fragment of first-order logic. Although we make
  extensive use of categorical language, this paper is designed to be
  self-contained and has considerable expository content. The only prerequisites
  are knowledge of first-order logic and the rudiments of category theory.
\end{abstract}

\section{Introduction} \label{sec:introduction}

The representation of human knowledge in computable form is among the oldest and
most fundamental problems of artificial intelligence. Several recent trends are
stimulating continued research in the field of knowledge representation (KR).
The birth of the Semantic Web \cite{berners-lee2001} in the early 2000s has led
to new technical standards and motivated new machine learning techniques to
automatically extract knowledge from unstructured text \cite{nickel2016}. In the
scientific community, successful knowledge bases like the Gene Ontology
\cite{ashburner2000} have inspired a proliferation of ontologies across biology
and biomedicine \cite{musen2009}. This development belongs to a general trend
towards greater openness and interconnectivity in science. Optimists dream of a
future science in which all scientific knowledge is open, online, and
interpretable by machines \cite{nielsen2012}.

Description logic is the dominant formalism for knowledge representation today.
In particular, the language OWL (Web Ontology Language), a W3C standard
underlying the Semantic Web, is a description logic \cite{grau2008}. Description
logics are logical calculi designed specifically for knowledge representation.
They lie somewhere between propositional logic and first-order predicate logic,
striking a trade-off between computational tractability and expressivity.

In parallel with the invention of description logic and the Semantic Web, a
mostly disjoint community of mathematicians, physicists, and computer scientists
have discovered that category theory, popularly known for its abstruseness, is
useful not just for describing abstract mathematical structures, but for
modeling such diverse real-world phenomena as databases, programming languages,
electrical circuits, and quantum mechanics
\cite{spivak2012b,lambek1988,baez2015b,abramsky2004}. The ethos of this research
program is that category theory can serve as a general-purpose modeling language
for science and engineering. Having internalized this perspective, it is but a
short step to contemplate a general-purpose knowledge representation system
based on category theory. In this spirit, Spivak and Kent have recently
introduced the \emph{ontology log} (or \emph{olog}), a simple and elegant
categorical framework for knowledge representation \cite{spivak2012a}.

An objective of this paper is to understand the relationship between the logical
and algebraic approaches to knowledge representation. To that end, we introduce
a third knowledge representation formalism that interpolates between description
logic and ontology logs. We call it the \emph{relational ontology log}, or
\emph{relational olog}. Spivak and Kent's ologs, which we sometimes call
\emph{functional ologs} to avoid confusion, are based on $\Set$, the archetypal
category of sets and functions. Relational ologs are based on $\Rel$, the
category of sets and relations. That may sound a small difference, since
functions and relations are often interchangeable, but it leads to very
different modes of expression. Functional ologs achieve their expressivity
through categorical limits and colimits (products, pullbacks, pushforwards,
etc.), while relational ologs rely mainly on relational algebra (intersections,
unions, etc.). In this sense, relational ologs are actually closer to
description logic than to functional ologs.

Practitioners of description logic will find in relational ologs several useful
features not possessed by most existing KR systems, including OWL. Some of these
features are awkward to handle in a purely logical system; all emerge
automatically from the categorical framework. First, functors allow
\emph{instance data} to be associated with an ontology in a mathematically
precise way. Instance data can be interpreted as a relational or graph database
or can take more exotic forms. Second, relational ologs are by default
\emph{typed}. We argue that types, if used judiciously, can mitigate the
maintainability challenges posed by the open world semantics of description
logic. Finally, relational ologs have a friendly and intuitive---yet fully
precise---\emph{graphical syntax}, derived from the string diagrams of monoidal
categories. We expect that this graphical language will appeal to technical and
non-technical users alike.

\paragraph{How to read this paper} We have tried to write a paper that is
accessible to a diverse audience. All Remarks and Appendices are technical and
can be skipped on a first reading. The mathematical prerequisites are limited as
follows. We assume the reader is familiar with the syntax and semantics of
first-order logic. No prior knowledge of description logic is required. We
expect the reader to know the ``big three'' concepts of category
theory---category, functor, and natural transformation---but we do not assume
knowledge of categorical logic or monoidal categories and their graphical
languages. References for further reading are provided where appropriate.

Readers who prefer to begin with an extended example may proceed immediately to
\cref{sec:foaf}, referring to \cref{sec:Rel} as needed to understand the
graphical notation. The core of the paper, explaining the
categorical-relational approach to knowledge representation, is
\cref{sec:Rel,sec:BiRel,sec:olog,sec:instances,sec:types}. The other sections
develop extensions of our methodology and make connections to other branches of
mathematics and knowledge representation.

\paragraph{Organization of paper} In the next section, we review description
logic as a computationally tractable subset of first-order logic and describe
several widely used description logics. In \cref{sec:Rel}, we introduce $\Rel$,
the category of sets and relations, and use it to illustrate the general
concepts of monoidal categories and their graphical languages. We also make
initial contact with the basic notions of description logic. Motivated by
$\Rel$, in \cref{sec:BiRel} we present the \emph{bicategory of relations}, a
categorical abstraction of relational algebra invented by Carboni and Walters.
\cref{sec:olog} defines a \emph{relational olog} to be a finitely presented
bicategory of relations and illustrates with an extended example.
\cref{sec:instances,sec:types} discuss the implications of instance data and
types for knowledge representation. In \cref{sec:logic}, we take a sojourn into
\emph{categorical logic}, proving that regular logic is the internal language of
bicategories of relations. This result establishes a formal connection between
relational ologs and a fragment of typed first-order logic. In
\cref{sec:expressive-olog}, we introduce the \emph{distributive relational
olog}, an extension of the relational olog with high expressivity. In the final
\cref{sec:conclusion}, we comment on the philosophy of categorical knowledge
representation and suggest directions for future research. The two Appendices
bring mathematical rigor to the informal discussion of categorical logic in the
main text.

\section{Description logic} \label{sec:DL}

Early knowledge representation systems, based on semantic networks or frames,
often lacked a formal semantics. The intended meanings of the elements of such
systems were defined only implicitly or operationally by the inference
algorithms that manipulated them. As a result, researchers found it difficult to
reason generally about these systems, independent of any specific
implementation. Arguments were advanced that knowledge representation should be
grounded in formal logic \cite{woods1975}. First-order logic, ever the
``default'' logical system, seems like a natural place to start.

Description logic (DL) is motivated by the deficiencies of first-order logic as
a foundation for knowledge representation. Chief among these is computational
intractability: first-order logic, while quite expressive, is undecidable. The
basic description logics are subsets of first-order logic designed to be
decidable (although not always in polynomial time). The tradeoff between
expressivity and tractability was emphasized by the earliest papers on
description logic \cite{brachman1984}. Another point, less frequently mentioned,
is that description logic is simpler and more user friendly than first-order
logic. As we will see, its syntax suppresses variables, both bound and free, and
imposes strict limits on the logical sentences that can be formed. Given that
most users of knowledge representation systems are domain experts in scientific
or business fields, not professional mathematicians, it is important that KR
formalisms be easily interpretable and maintainable. A knowledge base consisting
of a collection of arbitrary first-order sentences will probably not meet this
requirement.

\subsection{Review of description logic}

In this section, we briskly review description logic. General introductions to
description logic include the survey \cite{krotzsch2012} and the textbook
chapter \cite[Ch.\ 9]{brachman2004} by Brachman and Levesque. A comprehensive
reference is the \emph{Description Logic Handbook} \cite{baader2007}. For the
perspectives of the bioinformatics and Semantic Web communities, see
\cite{robinson2011} and \cite{hitzler2009}, respectively.

Description logic uses a special nomenclature to specify the features possessed
by a given system. The base system, from which most others are derived, is
called $\DL{AL}$ (for Attributive Concept Language). Given a collection of
\emph{atomic concepts}, denoted $A$, and \emph{atomic roles}, denoted $R$ or
$S$, the \emph{concept} descriptions of $\DL{AL}$ are well-formed terms of the
grammar:
\begin{align*}
  C, D \quad::=\quad
    & A \ | & \text{(atomic concept)} \\
    & \top \ | &\text{(universal concept)} \\
    & \bot \ | &\text{(bottom concept)} \\
    & \neg A \ & \text{(atomic negation)} \\
    & C \sqcap D \ | &\text{(intersection)} \\
    & \forall R.C \ | &\text{(value restriction)} \\
    & \exists R.\top &\text{(limited existential quantification)}.
\end{align*}
Note that negating arbitrary concepts is not allowed in $\DL{AL}$. Concepts
and roles are interpreted as unary and binary predicates in first-order logic:
\begin{align*}
  (\neg A)(x) &\qquad\text{iff}\qquad \neg A(x) \\
  (C \sqcap D)(x) &\qquad\text{iff}\qquad C(x) \wedge D(x) \\
  (\forall R.C)(x) &\qquad\text{iff}\qquad \forall y.(R(x,y) \to C(y)) \\
  (\exists R.\top)(x) &\qquad\text{iff}\qquad \exists y. R(x,y)
\end{align*}
A \emph{terminological box} or \emph{TBox} is a collection of 
\emph{terminological axioms} of form
\begin{equation*}
  C \sqsubseteq D \qquad\text{or}\qquad C \equiv D,
\end{equation*}
interpreted as the first-order sentences
\begin{equation*}
  \forall x. (C(x) \to D(x)) \qquad\text{or}\qquad
  \forall x. (C(x) \leftrightarrow D(y)).
\end{equation*}
An \emph{assertional box} or \emph{ABox} is a collection of \emph{assertional
axioms} of form
\begin{equation*}
  C(a) \qquad\text{or}\qquad R(a,b),
\end{equation*}
where $a,b$ are names of individuals. A \emph{knowledge base} or \emph{ontology}
in description logic consists of a TBox and an ABox. Given the above
translations into first-order logic, there is an obvious notion of an
\emph{interpretation} or \emph{model} of a knowledge base. Thus description
logic inherits a model-theoretic semantics from first-order logic.

More expressive description logics are obtained by adjoining to $\DL{AL}$
additional concept and role constructors, identified by script letters like
$\DL{C}$ and $\DL{U}$. The literature describes countless such extensions;
\cref{table:DL} lists the most important ones. As a warning, a few identifiers
(like $\DL{F}$ and $\DL{R}$) are not used consistently across the literature.

\begin{table}
  \centering
  \begin{tabular}{cllr}
    \toprule
    ID & Name & DL Syntax & FOL Interpretation \\
    \midrule
    $\DL{C}$ & Concept negation & $\neg C$ & $\neg C(x)$ \\
    $\DL{U}$ & Concept union & $C \sqcup D$ & $C(x) \vee D(x)$ \\
    $\DL{E}$ & Full existential quantification &
      $\exists R.C$ & $\exists y.(R(x,y) \wedge C(y))$ \\
    $\DL{H}$ & Role axioms 
      & $R \sqsubseteq S$ & $\forall x,y.(R(x,y) \to S(x,y))$ \\
      & & $R \equiv S$ & $\forall x,y.(R(x,y) \leftrightarrow S(x,y))$ \\                 
    $\DL{O}$ & Nominals (concept literals) & 
      $\{a_1,\dots,a_n\}$ & $\{a_1,\dots,a_n\}$ \\
    $\DL{I}$ & Inverse roles & $R^-$ & $R^-(x,y) \leftrightarrow R(y,x)$ \\
    $\DL{F}$ & Functional role & $\leq 1 R$ & see below \\
    $\DL{N}$ & Number restriction & $\geq n R$ & see below \\
      & & $= n R$ & \\
      & & $\leq n R$ & \\
    $\DL{Q}$ & Qualified number restriction & $\geq n R.C$ & see below \\
      & & $= n R.C$ & \\
      & & $\leq n R.C$ & \\
    N/A & Role intersection & $R \sqcap S$ & $R(x,y) \wedge S(x,y)$ \\
    N/A & Role union & $R \sqcup S$ & $R(x,y) \vee S(x,y)$ \\
    N/A & Role composition & $R \circ S$ & $\exists y. R(x,y) \wedge S(y,z)$ \\
    $\DL{R}$ & Regular role inclusion &
      $R_1 \circ \cdots \circ R_n \sqsubseteq S$ & see below \\
    $(\DL{D})$ & Concrete domains (data types) & varies & varies \\
    \bottomrule
  \end{tabular}
  \caption{Summary of $\DL{AL}$ language extensions}
  \label{table:DL}
\end{table}

Several DL constructs in \cref{table:DL} deserve elaboration. The qualified
number restriction ($\DL{Q}$) term $\geq n R.C$ (respectively $\leq n R.C$)
denotes the class of elements related by $R$ to at least $n$ (respectively at
most $n$) elements of class $C$. In first-order logic,
\begin{align*}
  (\geq n R.C)(x) &\qquad\text{iff}\qquad
    \exists y_1,\dots,y_n.\left(
      \bigwedge_{1 \leq i \leq n} (R(x,y_i) \wedge C(y_i))
      \wedge \bigwedge_{1 \leq i < j \leq n} y_i \neq y_j'
    \right) \\ 
  (\leq n R.C)(x) &\qquad\text{iff}\qquad
    \forall y_1,\dots,y_{n+1}.\left(
      \bigwedge_{1 \leq i \leq n} (R(x,y_i) \wedge C(y_i))
      \rightarrow \bigvee_{1 \leq i < j \leq n+1} y_i = y_j'
    \right).
\end{align*}
Number restriction ($\DL{N}$) and functional roles ($\DL{F}$) are special cases
of qualified number restriction. Concrete domains ($(\DL{D})$) refer to data
types, such as natural numbers or real numbers, and operations on them, such as
addition and multiplication. We return to the topic of data types in
\cref{sec:types}.

Most descriptions logics do not allow arbitrary intersection, union, or
composition of roles. However, composition-based regular role inclusion
($\DL{R}$) is widely used. System $\DL{R}$ allows axioms of form $R_1 \circ
\cdots \circ R_n \sqsubseteq S$, where $R_1,\dots,R_n$ are atomic roles,
provided there are no cycles between axioms. This acyclicity requirement, which
we will not make precise, leads to favorable computational properties. Note that
$\DL{R}$ is sometimes taken to include additional, ad hoc features like
reflexivity, ``local'' reflexivity, irreflexivity, and disjoint roles
\cite{horrocks2005,horrocks2006}.

A few description logics are privileged in theory or practice. The minimal
language $\DL{AL}$ is too inexpressive for most applications. The central
language in the theory of description logic is $\DL{ALC}$. It is logically
equivalent to $\DL{ALUE}$, although the shorter name $\DL{ALC}$ is preferred. In
a break with the standard nomenclature, the language $\DL{S}$ is $\DL{ALC}$ plus
transitive roles. The Web Ontology Languages are derived from system $\DL{S}$.
For example, OWL 1 Lite corresponds to $\DL{SHIF(D)}$, OWL 1 DL to
$\DL{SHION(D)}$, and OWL 2 DL to $\DL{SROIQ(D)}$.

\subsection{Structure of description logic}

To put the subsequent developments in context, we make a few observations about
the structure of description logic. Since description logic is not a single
logical system, but rather a large federation of systems, it is difficult to
make broad generalizations. Nevertheless, some general themes can be discerned.

An obvious syntactic difference between description logic and first-order logic
is that the former is \emph{point-free} while the latter is not. By
``point-free'' we mean that the concept and role constructors of description
logic suppress all variables, free and bound. First-order logics without
variables \emph{do} exist---Tarksi, for example, studied such systems in his
last major work \cite{tarski1987}---but, outside of description logic, they are
rare in research and in practice. In this respect, relational ologs are like
description logic: both the textual and graphical syntaxes of relational ologs
are point-free.

Description logics characteristically impose strict limitations on how concepts
and roles may be combined. Thus, not all first-order sentences are expressible
in description logic. The same is true of relational ologs: we shall see that
when relational ologs are interpreted as first-order theories
(\cref{sec:logic}), not all first-order sentences are expressible.

Moreover, there are structural similarities between the first-order sentences
that \emph{are} expressible in the two formalisms. In description logic,
terminological axioms $C \sqsubseteq D$ and $R \sqsubseteq S$ typically
translate into first-order sentences of form
\begin{equation*}
  \forall x_1 \cdots \forall x_n (\varphi \to \psi),
\end{equation*}
where $\varphi$ and $\psi$ are formulas containing only the connectives and
quantifiers $\wedge, \vee, \top, \bot, \exists$. (Depending on the language,
exceptions can arise from value restrictions $\forall R.C$ and number
restrictions like $\leq n R.C$. However, these constructors are acceptable in
axioms of form $\forall R.C \sqsubseteq \top$ or $\forall R.C \sqsubseteq \bot$.
Concept negations $\neg C$ also present exceptions.) The logical system just
described is called \emph{coherent logic}. The weaker system of \emph{regular
logic} is obtained when $\varphi$ and $\psi$ are further restricted to the
connectives and quantifiers $\wedge,\top,\exists$. We shall see that regular
logic and coherent logic are closely connected to relational ologs
(\cref{sec:logic,sec:expressive-olog}).

\section{The category of relations} \label{sec:Rel}

In this section we introduce $\Rel$, the category of sets and relations.
Although the reader is doubtless familiar with sets and relations, we think it
helpful to start the development in this very concrete setting. We will
introduce monoidal categories and their graphical languages by equipping $\Rel$
with various categorical structures, such as a monoidal product, diagonals and
codiagonals, and a dagger operator. These structures on $\Rel$ motivate the more
abstract ``categories of relations'' needed for knowledge representation
(\cref{sec:BiRel}). We will also make initial contact with description logic.

Our presentation draws on the physics-oriented survey by Coecke and Paquette
\cite{coecke2010}, where $\Rel$ is viewed as a ``quantum-like'' category, in
contrast to the ``classical-like'' category $\Set$. The excellent surveys
\cite{baez2010} and \cite{selinger2010} also provide more detail about monoidal
categories and their applications and graphical languages. General introductions
to category theory, in order of increasing sophistication, are
\cite{lawvere2009,spivak2014,awodey2010,leinster2014,riehl2016,maclane1998}.

\begin{definition}
  The \emph{category of sets and relations}, denoted by $\Rel$, is the category
  whose objects are sets and whose morphisms $R: X \to Y$ are subsets $R
  \subseteq X \times Y$. The \emph{composition} of $R: X \to Y$ and $S: Y \to
  Z$, written $R \cdot S: X \to Z$ or $RS: X \to Z$, is given by
  \begin{equation*}
    x RS z \qquad\text{iff}\qquad \exists y \in Y: xRy \wedge yRz,
  \end{equation*}
  where $xRy$ means that $(x,y) \in R$. For any set $X$, the \emph{identity}
  morphism $1_X$ is the diagonal relation:
  \begin{equation*}
    x(1_{X})x' \qquad\text{iff}\qquad x = x'.
  \end{equation*}
\end{definition}

The notion of composition of relations is natural and important. Notice that
when $R$ and $S$ are (graphs of) functions, $RS$ is the usual composition of
functions. Also, the identity morphism is the usual identity function. As a
result, $\Set$, the category of sets and functions, is a subcategory of $\Rel$.

\begin{remark}
  As illustrated by the definition, we compose morphisms in left-to-right or
  ``diagrammatic'' order. We make this choice for consistency with the graphical
  syntax, which is read from left to right. It is also consistent with the
  notation $xRy$ for $(x,y) \in R$.
\end{remark}

Unlike $\Set$, the category $\Rel$ is a special kind of a 2-category.

\begin{definition}
  A category $\cat{C}$ is a \emph{locally posetal 2-category} if between any two
  morphisms $f, g: A \to B$ with common domain and codomain, there exists at most
  one \emph{2-morphism}, written
  \begin{equation*}
    \begin{tikzcd}
      A \arrow[r, bend left=35, "f"{above}, ""{name=f, below}]
        \arrow[r, bend right=35, "g"{below}, ""{name=g, above}]
        & B
      \arrow[Rightarrow, from=f, to=g]
    \end{tikzcd}
  \end{equation*}
  or more succintly $f \Rightarrow g$, together with operations of
  \emph{vertical composition},
  \begin{equation*}
    \begin{tikzcd}
      A \arrow[r, bend left=60, "f"{above}, ""{name=f, below}]
        \arrow[r, "g"{name=g, anchor=center, fill=white}]
        \arrow[r, bend right=60, "h"{below}, ""{name=h, above}]
        & B
      \arrow[Rightarrow, from=f.center, to=g]
      \arrow[Rightarrow, from=g, to=h.center]
    \end{tikzcd}
    \qquad\leadsto\qquad
    \begin{tikzcd}
      A \arrow[r, bend left=35, "f"{above}, ""{name=f, below}]
        \arrow[r, bend right=35, "h"{below}, ""{name=h, above}]
        & B
      \arrow[Rightarrow, from=f, to=h]
    \end{tikzcd},
  \end{equation*}
  and \emph{horizontal composition},
  \begin{equation*}
    \begin{tikzcd}
      A \arrow[r, bend left=35, "f"{above}, ""{name=f, below}]
        \arrow[r, bend right=35, "g"{below}, ""{name=g, above}]
      & B \arrow[r, bend left=35, "h"{above}, ""{name=h, below}]
          \arrow[r, bend right=35, "k"{below}, ""{name=k, above}]
      & C
      \arrow[Rightarrow, from=f, to=g]
      \arrow[Rightarrow, from=h, to=k]
    \end{tikzcd}
    \qquad\leadsto\qquad
    \begin{tikzcd}
      A \arrow[r, bend left=35, "f \cdot h"{above}, ""{name=fh, below}]
        \arrow[r, bend right=35, "g \cdot k"{below}, ""{name=gk, above}]
        & C
      \arrow[Rightarrow, from=fh, to=gk]
    \end{tikzcd},
  \end{equation*}
  such that each hom-set $\cat{C}(A,B)$ is a poset (partially ordered set)
  under the relation $\Rightarrow$.
\end{definition}

In $\Rel$, we stipulate that $R \Rightarrow S$ if there is a set containment $R
\subseteq S$. Vertical composition simply says that set containment is
transitive. More interestingly, horizontal composition says that containment is
preserved by composition of relations: if $R \subseteq S$ and $T \subseteq U$,
then $R \cdot T \subseteq S \cdot U$. Furthermore, the hom-sets are posets
because $R \subseteq R$ and also $R = S$ whenever $R \subseteq S$ and $S
\subseteq R$. Thus $\Rel$ is a locally posetal 2-category. In this context, the
symbol $\Rightarrow$ has a happy double meaning: we can read $\Rightarrow$ as a
generic 2-morphism or as \emph{logical implication}. In the terminology of
description logic, 2-morphisms are subsumptions.

\subsection{Monoidal category}

We will make $\Rel$ into a monoidal category by equipping it with the Cartesian
product. We first state the general definition of a monoidal category.

\begin{definition}
  A (strict) \emph{monoidal category} $(\cat{C}, \otimes, I)$ is a category
  $\cat{C}$ together with a functor ${\otimes: \cat{C} \times \cat{C} \to
  \cat{C}}$, called the \emph{monoidal product}, and an object $I$, called the
  \emph{monoidal unit}, such that $(\otimes,I)$ behaves like a monoid on the
  objects and morphisms of $\cat{C}$, in the following sense. For objects
  $A,B,C$, we have
  \begin{equation*}
    A \otimes (B \otimes C) = (A \otimes B) \otimes C, \qquad
    A \otimes I = I \otimes A = I,
  \end{equation*}
  and for morphisms $f,g,h$, we have
  \begin{equation*}
    f \otimes (g \otimes h) = (f \otimes g) \otimes h, \qquad
    f \otimes 1_I = 1_I \otimes f = f.
  \end{equation*}
\end{definition}

More explicitly, functorality of the monoidal product $\otimes$ means that for
any objects $A,B$,
\begin{equation*}
  1_A \otimes 1_B = 1_{A \otimes B}
\end{equation*}
and for any morphisms $f: A \to B$, $g: B \to C$, $h: D \to E$, $k: E \to F$,
\begin{equation*}
  (f \cdot g) \otimes (h \cdot k) = (f \otimes h) \cdot (g \otimes k).
\end{equation*}

Let us immediately introduce the graphical language of \emph{string diagrams}
that is associated with any monoidal category. In this language, objects are
represented by wires and morphisms are represented by boxes with incoming and
outgoing wires. A generic morphism $f: A \to B$ is represented as
\begin{center}
  \includegraphics{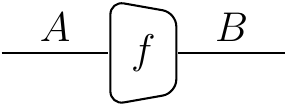}.
\end{center}
The composite $fg: A \to C$ of $f: A \to B$ and $g: B \to C$ is
\begin{center}
  \includegraphics{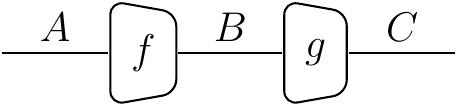}
\end{center}
and the monoidal product $f \otimes g: A \otimes B \to C \otimes D$ of
 $f: A \to B$ and $g: C \to D$ is
\begin{center}
  \includegraphics{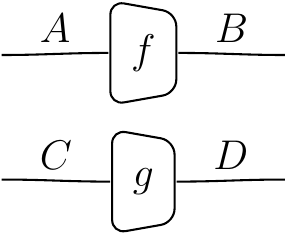}.
\end{center}
Identity morphisms are represented specially as a bare wire:
\begin{equation*}
  1_{A} \quad=\quad \vcenter{\hbox{\includegraphics{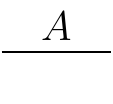}}}.
\end{equation*}
Each time we introduce a new monoidal structure, we will augment the graphical
language accordingly.

String diagrams are among the most beautiful aspects of the theory of monoidal
categories. Unlike the diagrams and flowcharts found throughout the engineering
literature, which have no formal meaning, string diagrams provide a formal
calculus for reasoning in monoidal categories. More precisely, \emph{coherence
theorems} guarantee that string diagrams constitute a \emph{sound} and
\emph{complete} calculus for equational reasoning with morphisms in a monoidal
category. Coherence theorems are emphasized by Selinger's comprehensive survey
\cite{selinger2010}. All string diagrams in this paper are drawn by the author's
(highly experimental) library for computational category theory
\cite{catlab2017}.

The Cartesian product makes $\Rel$ into a monoidal category in the most
straightforward way. (Later we will see that it is not the only interesting
monoidal product on $\Rel$.) On objects, define
\begin{equation*}
  X \otimes Y := X \times Y = \{(x,y): x \in X, y \in Y\}
\end{equation*}
and given morphisms $R: X \to Y$ and $S: Z \to W$, define $R \otimes S: X
\otimes Z \to Y \otimes W$ by
\begin{equation*}
  (x,z) (R \otimes S) (y,w) \qquad\text{iff}\qquad xRy \wedge zSw.
\end{equation*}
The monoidal unit is any singleton set, which we write as $I = \{*\}$.

\begin{remark}
  Technically, $\Rel$ is not a \emph{strict} monoidal category, as defined
  above, because the Cartesian product is not strictly associative: $X \times (Y
  \times Z) \neq (X \times Y) \times Z$. Of course, there is a natural
  isomorphism $X \times (Y \times Z) \cong (X \times Y) \times Z$, mapping
  $(x,(y,z))$ to $((x,y),z)$, that allows us to identify these two sets.
  Similarly, $X \times I \neq X$ but there is a natural isomorphism $X \times I
  \cong X$ that identifies $(x,*)$ with $x$. Such considerations lead to the
  general definition of a monoidal category, where strict associativity and
  units are replaced with \emph{associator} and \emph{unitor} natural
  isomorphisms, subject to some coherence conditions. However, the coherence
  theorem for monoidal categories \cite{maclane1963} ensures that a general
  monoidal category is monoidally equivalent to some strict monoidal category,
  called its \emph{strictification}. As a result, we adopt the common practice
  of suppressing associators and unitors, effectively replacing every monoidal
  category by its strictification. Note that the graphical language implicitly
  performs this strictification. Incidentally, the abstract categories of
  relations we construct in \cref{sec:logic} will actually be strict monoidal
  categories.
\end{remark}

As in most monoidal categories in encountered in practice, in $\Rel$ the order
of inputs and outputs can be freely exchanged.

\begin{definition}
  A monoidal category $(\cat{C}, \otimes, I)$ is a  \emph{symmetric monoidal
  category} if there is a natural family of isomorphisms
  \begin{equation*}
    \sigma_{A,B}: A \otimes B \to B \otimes A, \qquad A, B \in \cat{C},
  \end{equation*}
  called \emph{braidings}, satisfying $\sigma_{A,B}^{-1} = \sigma_{B,A}$.
\end{definition}

In the graphical language, braidings are represented by crossed wires:
\begin{equation*}
  \sigma_{A,B} \quad=\quad \vcenter{\hbox{\includegraphics{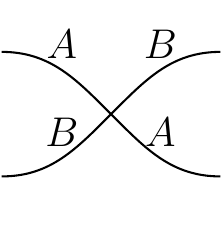}}}.
\end{equation*}
The braidings in $\Rel$ are defined by
\begin{equation*}
  (x, y)\,\sigma_{X,Y}\,(y', x') \qquad\text{iff}\qquad
    x = x' \wedge y = y'.
\end{equation*}

In $\Rel$, unlike in $\Set$, there is a fundamental duality between inputs and
outputs under which any input can be turned into an output and vice versa. This
duality is captured abstractly by the following definition.

\begin{definition}
  A symmetric monoidal category $(\cat{C}, \otimes, I)$ is a \emph{compact
  closed category} if for every object $A \in \cat{C}$, there is an object
  $A^*$, the \emph{dual} of $A$, and a pair of morphisms $\eta_A: I \to A^*
  \otimes A$ and $\epsilon_A: A \otimes A^* \to I$, the \emph{unit} and
  \emph{counit} respectively, which satisfy the \emph{triangle} or
  \emph{zig-zag} identities:
  \begin{equation*}
    \begin{tikzcd}[row sep=large, column sep=large]
      A \arrow[r, "1_A \otimes \eta_A"]
        \arrow[rd, "1_A"{below}]
        & A \otimes A^* \otimes A
          \arrow[d, "\epsilon_A \otimes 1_A"] \\
      & A
    \end{tikzcd}
    \hspace{1in}
    \begin{tikzcd}[row sep=large, column sep=large]
      A^* \arrow[r, "\eta_A \otimes 1_{A^*}"]
          \arrow[rd, "1_{A^*}"{below}]
        & A^* \otimes A \otimes A^*
          \arrow[d, "1_{A^*} \otimes \epsilon_A"] \\
      & A^*
    \end{tikzcd}.
  \end{equation*}
\end{definition}

The prototypical example of a compact closed category is $(\Vect_k,\otimes)$,
the category of finite-dimensional vector spaces (over a fixed field $k$) and
linear maps, equipped with the tensor product. The monoidal unit is $I = \C$,
the one-dimensional vector space. As expected, the dual $A^*$ of a vector space
$A$ is the space of linear maps $A \to \C$. The unit $\eta_A: I \to A^* \otimes
A$ maps $c$ to $c 1_{A}$ and the counit $\epsilon_A: A \otimes A^* \to I$ is the
trace operator.

$\Rel$ is a \emph{self-dual} compact closed category. That is, every object is
its own dual (${X^* = X}$). The unit $\eta_X: I \to X \otimes X$ and counit
$\epsilon_X: X \otimes X \to I$ are defined by
\begin{equation*}
  (*)\,\eta_X\,(x,x')
    \qquad\text{iff}\qquad x = x'
    \qquad\text{iff}\qquad (x,x')\,\epsilon_X\,(*).
\end{equation*}
In the graphical language, these morphisms are represented as ``bent wires'':
\begin{equation*}
  \eta_X \quad=\quad \vcenter{\hbox{\includegraphics{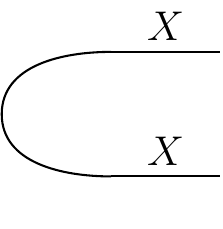}}}
  \hspace{1in}
  \epsilon_X \quad=\quad \vcenter{\hbox{\includegraphics{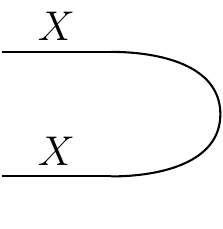}}}.
\end{equation*}
The zig-zag identities assert that ``zig-zags can be straightened out'':
\begin{equation*}
  \vcenter{\hbox{\includegraphics{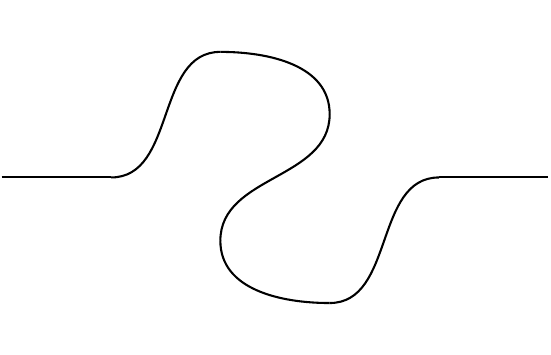}}}
  \quad=\quad
  \vcenter{\hbox{\includegraphics{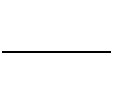}}}
  \quad=\quad
  \vcenter{\hbox{\includegraphics{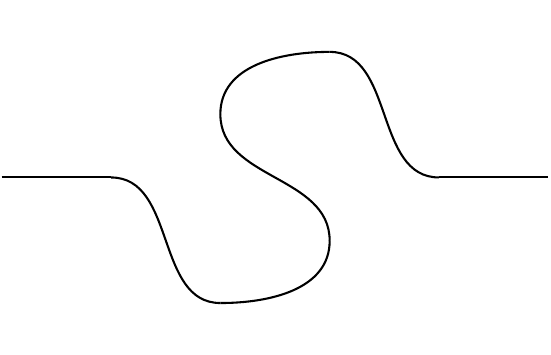}}}.
\end{equation*}

\begin{remark}
  String diagrams for compact closed categories typically include arrowheads on
  the wires to distinguish objects $\xrightarrow{\ A\ }$ from their duals
  $\xrightarrow{\ A^* \ }$, which are drawn as reversed arrows $\xleftarrow{\ A \ }$.
  Because we work in self-dual categories, we can safely omit the arrowheads.
\end{remark}

We have amassed enough structure to specify relations of arbitrary arity.
Relations $B: I \to I$ whose domain and codomain are the monoidal unit
constitute the degenerate case of arity zero. Since $I = \{*\}$ is the singleton
set, the hom-set $\Rel(I,I)$ has only two members: the identity relation $1_I = \{(*,*)\}$
and the empty relation $\emptyset$. Defining $\top := 1_I$ and $\bot :=
\emptyset$, we interpret relations $B: I \to I$ as \emph{booleans}.

Next, we have unary and binary relations. A unary relation $C: X \to I$ is
called a \emph{class} or \emph{concept} in the description logic literature. Its
elements have the form $(x,*)$, where $x \in X$. In the graphical calculus,
wires of type $I$ are not drawn at all, so a concept $C: X \to I$ is represented
as
\begin{equation*}
  \includegraphics{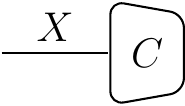}.
\end{equation*}
A binary relation $R: X \to Y$, or \emph{role} in description logic jargon, is
depicted as
\begin{equation*}
  \includegraphics{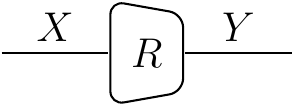},
\end{equation*}
as we have seen.

Finally, we can easily express higher-order relations. A relation of arity $n$
is a morphism of form $R: X_1 \otimes \cdots \otimes X_n \to I$. For instance,
here is a ternary relation $R: X \otimes Y \otimes Z \to I$:
\begin{equation*}
  \includegraphics{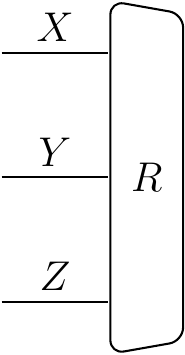}
\end{equation*}
Apparently, there are two conventions for representing a binary relation: as a
morphism $R: X \to Y$ or a morphism $R: X \otimes Y \to I$. By bending wires, we
can pass freely between the two representations, via the transformations
\begin{equation*}
  \vcenter{\hbox{\includegraphics{img/role}}}
  \qquad\leadsto\qquad
  \vcenter{\hbox{\includegraphics{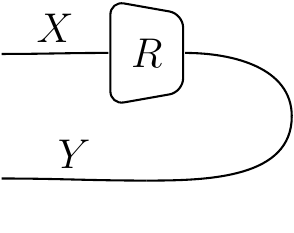}}}
\end{equation*}
and
\begin{equation*}
  \vcenter{\hbox{\includegraphics{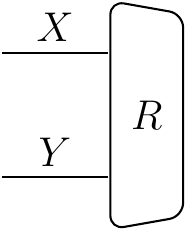}}}
  \qquad\leadsto\qquad
  \vcenter{\hbox{\includegraphics{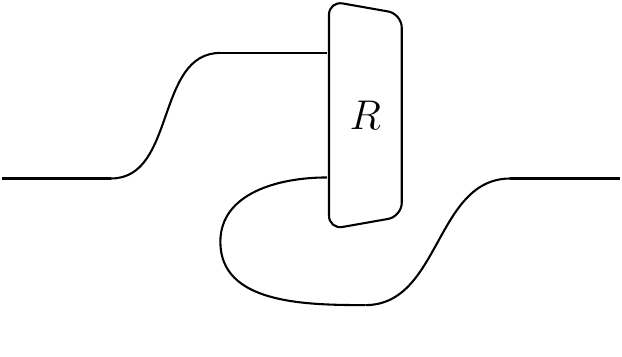}}},
\end{equation*}
which, by the zig-zag identities, are mutually inverse. Most description logics
do not support relations of arity greater than two, in part because the
point-free textual syntax becomes quite awkward \cite[\S 5.7]{baader2007}. The
graphical language of monoidal categories enables graceful and intuitive
composition even when relations have multiple inputs and outputs.

\subsection{Dagger category}

Every relation $R: X \to Y$ in $\Rel$ has an \emph{opposite} relation
$R^\dagger: Y \to X$, also known as the \emph{converse} or \emph{inverse}
relation, defined by
\begin{equation*}
  y R^\dagger x \qquad\text{iff}\qquad xRy.
\end{equation*}
This structure is axiomatized by the following definition.

\begin{definition}
  A \emph{dagger category} is a category $\cat{C}$ equipped with a contravariant
  functor $(-)^\dagger: \cat{C}^\opposite \to \cat{C}$ that is the identity on
  objects and is involutive, i.e., $((-)^\dagger)^\dagger = 1_{\cat{C}}$.

  More explicitly, a dagger category is a category $\cat{C}$ such that to every
  morphism $f: A \to B$ there corresponds a morphism $f^\dagger: B \to A$,
  and the correspondence satisfies
  \begin{equation*}
    {1_A}^\dagger = 1_A, \qquad
    (f g)^\dagger = g^\dagger f^\dagger, \qquad\text{and}\qquad
    (f^\dagger)^\dagger = f.
  \end{equation*}
\end{definition}

When $\cat{C}$ is a monoidal category, one typically asks that the dagger
respect the monoidal structure.

\begin{definition}
  A symmetric monoidal category $\cat{C}$ is a \emph{dagger symmetric monoidal
  category} if $\cat{C}$ is a dagger category and $(-)^\dagger$ is a symmetric
  monoidal functor, i.e.,
  \begin{equation*}
    (f \otimes g)^\dagger = f^\dagger \otimes g^\dagger
    \qquad\text{and}\qquad
    \sigma_{A,B}^\dagger = \sigma_{B,A} = \sigma_{A,B}^{-1}.
  \end{equation*}
  A compact closed category $\cat{C}$ is a \emph{dagger compact category} if
  $\cat{C}$ is a dagger symmetric monoidal category and for each object $A \in
  \cat{C}$, there is a commutative diagram
  \begin{equation*}
    \begin{tikzcd}
      I \arrow[r, "\epsilon_A^\dagger"] \arrow[dr, "\eta_A"{below left}]
        & A \otimes A^* \arrow[d, "\sigma_{A,A^*}"{right}] \\
      & A^* \otimes A
    \end{tikzcd}.
  \end{equation*}
\end{definition}

\begin{remark}
  Dagger compact categories have been introduced and studied in the context of
  quantum computing \cite{selinger2007}. The prototypical example is $\Hilb$,
  the category of finite-dimensional Hilbert spaces and linear maps, equipped
  with the tensor product. Here $f^\dagger$ is the usual adjoint (or Hermitian
  conjugate) of the linear map $f$. For this reason, in a general dagger
  category $\cat{C}$, the morphism $f^\dagger$ is often called the
  \emph{adjoint} of $f$. To avoid confusion with the 2-categorical notion of
  adjoint, invoked in \cref{sec:BiRel}, we do not use this terminology.
\end{remark}

In the graphical language, $f^\dagger$ is represented by taking the ``mirror
image'' of $f$:
\begin{equation*}
  \vcenter{\hbox{\includegraphics{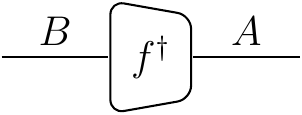}}} \quad:=\quad
  \vcenter{\hbox{\includegraphics{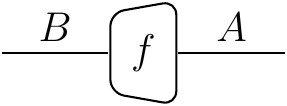}}}.
\end{equation*}
Equivalently, given any morphism $f$ drawn as a string diagram---which we read
from left to right, as usual---we get $f^\dagger$ by reading the same diagram
from right to left.

With the dagger operation defined above, $\Rel$ is a dagger compact category.
However, unlike most dagger compact categories, $\Rel$ is self-dual. Thus there
is potentially a second way to transform a morphism $X \to Y$ into a morphism $Y
\to X$: bend both the input and the output wires. These two operations are
actually the same, which we can express graphically as:
\begin{equation*}
  \vcenter{\hbox{\includegraphics{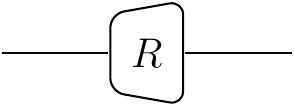}}} \quad=\quad
  \vcenter{\hbox{\includegraphics{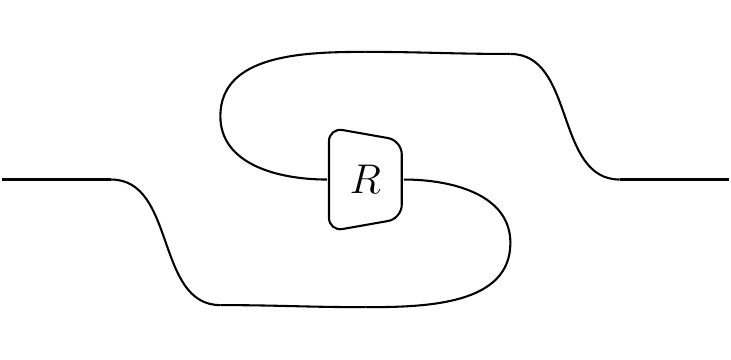}}}.
\end{equation*}
This equation is a generalization of the zig-zag identity: if we imagine
``straightening out'' the right-hand side by pulling on the ends of the input
and output wires, we obtain the left-hand side. Mathematically speaking, the
dagger structure on $\Rel$ is superfluous since it can be reduced to the compact
closed structure. We choose to make the dagger structure explicit because
inverse relations occur frequently in practice and the associated graphical
language is succinct and intuitive.

\subsection{Diagonals and codiagonals}

In our final topic of this section, we show that the category of relations has
operations for ``copying'' and ``deleting'' data and, dually, for ``merging''
and ``creating'' data. Using these operations, we can express intersections of
classes and relations and, more generally, logical operations involving
conjunction. We also obtain an important characterization of the ``functional
relations'', or \emph{maps}, in $\Rel$.

The ``copying'' and ``merging'' operations are defined by internal comonoids and
monoids, respectively, in $\Rel$. We recall the general definition of a
(co)monoid in a monoidal category.

\begin{definition}
  Let $(\cat{C}, \otimes, I)$ be a monoidal category. An \emph{internal monoid}
  in $\cat{C}$ is an object $M \in \cat{C}$ together with a
  \emph{multiplication} morphism $\mu: M \otimes M \to M$ and a \emph{unit}
  morphism $\eta: I \to M$ such that
  \begin{equation*}
    \begin{tikzcd}
      M \otimes M \otimes M
          \arrow[r, "1_M \otimes \mu"]
          \arrow[d, "\mu \otimes 1_M"]
        & M \otimes M \arrow[d, "\mu"{right}]\\
      M \otimes M \arrow[r, "\mu"{below}]
        & M
    \end{tikzcd}
    \hspace{0.5in}
    \begin{tikzcd}
      I \otimes M \arrow[r, "\eta \otimes 1_M"{above}] \arrow[rd, equal]
        & M \otimes M \arrow[d, "\mu"]
        & M \otimes I \arrow[l, "1_M \otimes \eta"{above}] \arrow[ld, equal] \\
      & M &
    \end{tikzcd}.
  \end{equation*}
  Dually, an \emph{internal comonoid} in $\cat{C}$ is an internal monoid in
  $\cat{C}^\opposite$. In concrete terms, an internal comonoid is an object $C$
  together with a \emph{comultiplication} morphism $\delta: C \to C \otimes C$
  and a \emph{counit} morphism $\epsilon: C \to I$ such that
  \begin{equation*}
    \begin{tikzcd}
      C \arrow[r, "\delta"{above}]
        \arrow[d, "\delta"{left}]
        & C \otimes C \arrow[d, "\delta \otimes 1_C"{right}]\\
      C \otimes C \arrow[r, "1_C \otimes \delta"{below}]
        & C \otimes C \otimes C
    \end{tikzcd}
    \hspace{0.5in}
    \begin{tikzcd}
      & C \arrow[d, "\delta"] & \\
      I \otimes C \arrow[ru, equal]
        & C \otimes C
          \arrow[r, "\epsilon \otimes 1_C"{below}]
          \arrow[l, "1_C \otimes \epsilon"{below}]
        & C \otimes I \arrow[lu, equal]
    \end{tikzcd}.
  \end{equation*}
  
  If $\cat{C}$ is a symmetric monoidal category, we say that an internal monoid
  $(M, \mu, \eta)$ is \emph{commutative} if
  \begin{equation*}
    \begin{tikzcd}
      M \otimes M \arrow[r, "\sigma_{M,M}"] \arrow[rd, "\mu"{below left}]
        & M \otimes M \arrow[d, "\mu"] \\
        & M
    \end{tikzcd}.
  \end{equation*}
  Dually, an internal comonoid $(C, \delta, \epsilon)$ is \emph{cocommutative} if
  \begin{equation*}
    \begin{tikzcd}
      C \arrow[r, "\delta"] \arrow[rd, "\delta"{below left}]
        & C \otimes C \arrow[d, "\sigma_{C,C}"] \\
        & C \otimes C
    \end{tikzcd}.
  \end{equation*}
  Note that an internal monoid in $\Set$ is just a monoid in the usual sense,
  i.e., a set $M$ equipped an associative binary operation $\mu$ and an identity
  element $\eta$. Likewise, an internal commutative monoid in $\Set$ is just a
  commutative monoid.
\end{definition}

We define a family of internal (co)monoids in $\Rel$ as follows. For each set
$X$, define $\Copy_X: X \to X \otimes X$ by
\begin{equation*}
  x\,\Copy_X\,(x', x'') \qquad\text{iff}\qquad x = x' \wedge x = x'',
\end{equation*}
and define $\Delete_X: X \to I$ by $\Delete_X = \{(x,*): x \in X\}$ (so that
$x(\Delete_X)*$ holds for every $x \in X$). It is easily verified that
$(X,\Copy_X,\Delete_X)$ is a cocommutative comonoid in $\Rel$. By taking the
opposite relations
\begin{equation*}
  \Merge_X := \Copy_X^\dagger: X \otimes X \to X
  \qquad\text{and}\qquad
  \Create_X := \Delete_X^\dagger: I \to X,
\end{equation*}
we also obtain for each set $X$ a commutative monoid ($X, \Merge_X, \Create_X)$
in $\Rel$. We think of $\Copy_X$ as ``copying'' or ``duplicating,'' $\Merge_X$
as ``merging,'' $\Delete_X$ as ``deleting'' or ``erasing,'' and $\Create_X$ as
``creating.'' These interpretations will be manifest from the graphical
language, to be demonstrated shortly. However, for the graphical language to be
consistent, the family of (co)monoids must satisfy certain \emph{coherence
axioms}, which ensure that they interact properly with the monoidal product.
These axioms are captured by the following definition.

\begin{definition}[\cite{selinger1999}]
  A \emph{monoidal category with diagonals} is a symmetric monoidal category
  $\cat{C}$ together with a family of morphisms $\Copy_A: A \to A \otimes A$
  and $\Delete_A: A \to I$, not necessarily natural in objects $A$, such that each
  triple $(A,\Copy_A,\Delete_A$) is a cocommutative comonoid in $\cat{C}$ and obeys
  the \emph{coherence axioms}
  \begin{equation*}
    \Delete_I = 1_I, \qquad
    \Delete_{A \otimes B} = \Delete_A \otimes \Delete_B, \qquad
    \Copy_{A \otimes B} = 
      (\Copy_A \otimes \Copy_B) (1_A \otimes \sigma_{A,B} \otimes 1_B).
  \end{equation*}
  Dually, a \emph{monoidal category with codiagonals} is a symmetric monoidal
  category $\cat{C}$ together with a family of morphisms $\Merge_A: A \otimes A
  \to A$ and $\Create_A: I \to A$ such that each triple $(A,\Merge_A,\Create_A)$ is
  commutative monoid in $\cat{C}$ and obeys the coherence axioms
  \begin{equation*}
    \Create_I = 1_I, \qquad
    \Create_{A \otimes B} = \Create_A \otimes \Create_B, \qquad
    \Merge_{A \otimes B} = 
      (1_A \otimes \sigma_{B,A} \otimes 1_B) (\Merge_A \otimes \Merge_B).
  \end{equation*}
\end{definition}

\begin{remark}
  When the monoidal category $\cat{C}$ is, like $\Rel$, not strict, we also need
  coherence axioms asserting that $\Copy_I$ and $\Merge_I$ are the unitors
  realizing the isomorphism $I \cong I \otimes I$.
\end{remark}

The graphical language of a monoidal category with diagonals is
\begin{equation*}
  \Copy_A \quad=\quad \vcenter{\hbox{\includegraphics{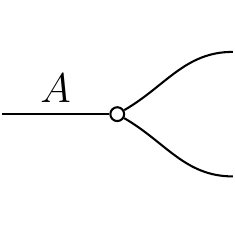}}},
  \hspace{1in}
  \Delete_A \quad=\quad \vcenter{\hbox{\includegraphics{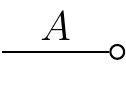}}}.
\end{equation*}
Similarly, the graphical language of a monoidal category with codiagonals is
\begin{equation*}
  \Merge_A \quad=\quad \vcenter{\hbox{\includegraphics{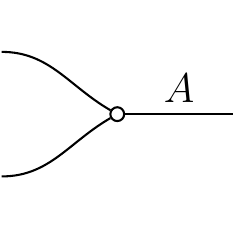}}},
  \hspace{1in}
  \Create_A \quad=\quad \vcenter{\hbox{\includegraphics{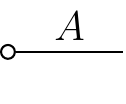}}}.
\end{equation*}
By the coherence axioms, we can express the diagonal morphisms for a product
$A \otimes B$ in the graphical language as
\begin{equation*}
  \Copy_{A \otimes B} \quad=\quad \vcenter{\hbox{\includegraphics{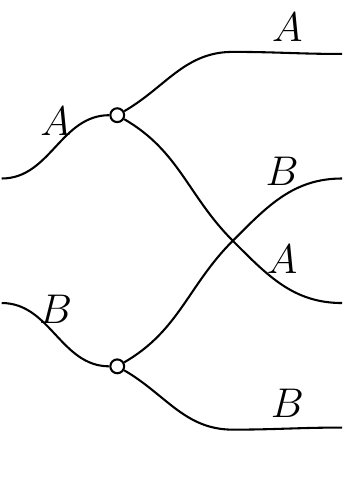}}},
  \hspace{1in}
  \Delete_{A \otimes B} \quad=\quad \vcenter{\hbox{\includegraphics{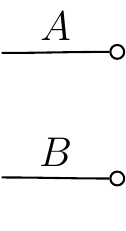}}}.
\end{equation*}
Of course, there is a dual picture for the codiagonal morphisms.

Under the above definitions, $\Rel$ is a monoidal category with diagonals and
codiagonals. A quick calculation shows that the intersection $R \cap S $ of two
relations $R,S: X \to Y$ with common domain $X$ and codomain $Y$ is $\Copy_X (R
\otimes S) \Merge_Y$ or, in graphical language,
\begin{equation*}
  R \cap S \quad=\quad \vcenter{\hbox{\includegraphics{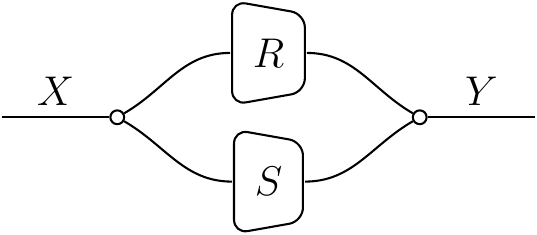}}}.
\end{equation*}
As a special case, the intersection of two classes $C, D: X \to I$ is
\begin{equation*}
  C \cap D \quad=\quad \vcenter{\hbox{\includegraphics{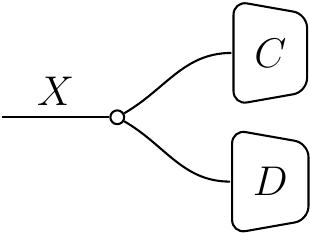}}}.
\end{equation*}

We can also express many of the basic concept constructors in description logic.
For any relation $R: X \to Y$ and class $C: Y \to I$, the ``limited''
existential quantification $\exists R.\top$ is the class
\begin{equation*}
  \includegraphics{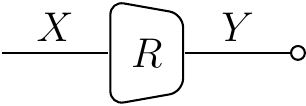}
\end{equation*}
and the ``full'' existential quantification $\exists R.C$ is the class
\begin{equation*}
  \includegraphics{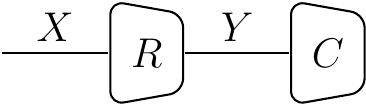}.
\end{equation*}
In contrast to description logic, we can retain access to the domain or codomain
of the relation $R$ while restricting its values. For instance, given classes
$C: X \to I$ and $D: Y \to I$, the relation
\begin{equation*}
  \includegraphics{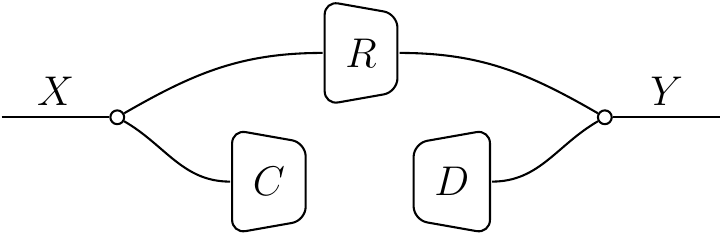}
\end{equation*}
consists of all pairs $(x,y) \in X \times Y$ satisfying  $xRy \wedge {xC*}
\wedge {yD*}$. The value restriction concept constructor $\forall R.C$ cannot be
expressed as a single morphism, but we can achieve the same effect by declaring
a subsumption of two different morphisms:
\begin{equation*}
  \vcenter{\hbox{\includegraphics{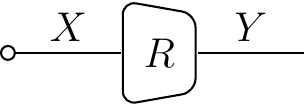}}}
  \qquad\implies\qquad
  \vcenter{\hbox{\includegraphics{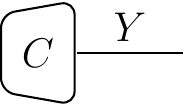}}}.
\end{equation*}
This 2-morphism asserts that
$\forall x \in X.\forall y \in Y.(xRy \rightarrow yC*)$.

Finally, we can express a typed variant of the ``universal role'' in
description logic. The local maximum
\begin{equation*}
  \top_{X,Y} \quad:=\quad \vcenter{\hbox{\includegraphics{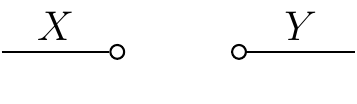}}}
\end{equation*}
is the (unique) maximum element of the poset $\Rel(X,Y)$, namely $X \times Y$.
It generalizes the top element $\top = \top_{I,I} = 1_I$ of the booleans
$\Rel(I,I)$.

\paragraph{Maps} The diagonal structure on $\Rel$ leads to an abstract
characterization of the relations that are functions, i.e., the relations $R: X
\to Y$ with the property that for every $x \in X$, there exists a unique $y \in
Y$ such that $xRy$. This matter is closely connected to the naturality, or lack
thereof, of the diagonal in $\Rel$. In general, a diagonal in a symmetric
monoidal category $\cat{C}$ is \emph{natural} if for every morphism $f: A \to
B$, we have $f \Copy_B = \Copy_A (f \otimes f)$ and $f \Delete_B = \Delete_A$,
or graphically
\begin{equation*}
  \vcenter{\hbox{\includegraphics{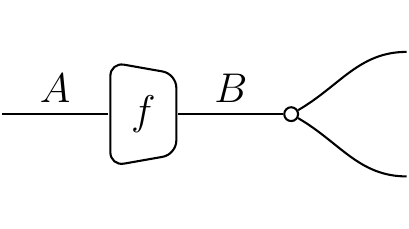}}}
  \quad=\quad
  \vcenter{\hbox{\includegraphics{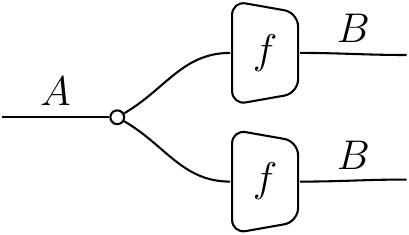}}}
\end{equation*}
and
\begin{equation*}
  \vcenter{\hbox{\includegraphics{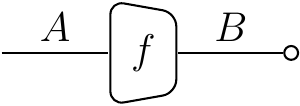}}}
  \quad=\quad
  \vcenter{\hbox{\includegraphics{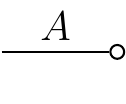}}}.
\end{equation*}
The first equation has the interpretation that applying $f$, then copying the
output is the same as copying the input, then applying $f$ to both copies; the
second that applying $f$, then deleting the output is the same as deleting the
input. When both equations hold for a morphism $f$, we say that $f$ is a
\emph{comonoid homomorphism}. In a general category, we expect the equations to
hold for morphisms that ``behave like functions.''

The diagonal in $\Rel$ is not natural because not all relations are functions.
However, for any relation $R: X \to Y$, there are 2-morphisms
\begin{align*}
  \vcenter{\hbox{\includegraphics{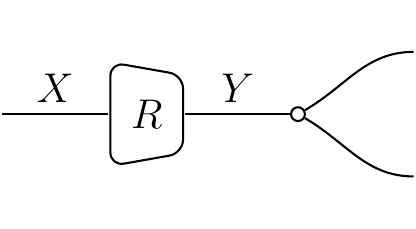}}}
  \quad&\implies\quad
  \vcenter{\hbox{\includegraphics{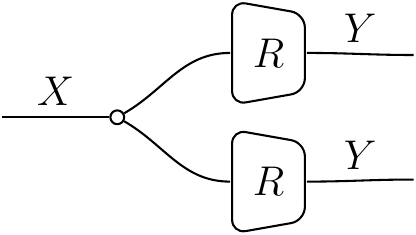}}} \\
  \vcenter{\hbox{\includegraphics{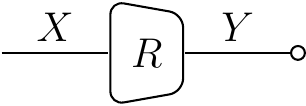}}}
  \quad&\implies\quad
  \vcenter{\hbox{\includegraphics{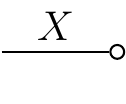}}}.
\end{align*}
We say that every morphism in $\Rel$ is a \emph{lax comonoid homomorphism}.
Explicitly, the 2-morphisms are the inclusions
\begin{align*}
  \{(x,y,y): xRy\} &\subseteq \{(x,y,y'): xRy \wedge xRy'\} \\
  \{(x,*): \exists y \in Y. xRy\} &\subseteq \{(x,*): x \in X\}.
\end{align*}
When the first inclusion is an equality, $R$ is a \emph{partial function}; when
the second is an equality, $R$ is \emph{total}; when both are equalities,
$R$ is a \emph{function} or a \emph{map}. In other words, the comonoid
homomorphisms in $\Rel$ are exactly the relations that are functions.

Of course, for every concept about the diagonal, there is a dual concept about
the codiagonal, whose details we omit. In $\Rel$, we obtain abstract
characterizations of the \emph{injective}, \emph{surjective}, and
\emph{bijective} relations. By combining the diagonal and codiagonal structures,
we can characterize the injective functions, surjective functions, etc.

\paragraph{Interactions between structures} To conclude this section, we
consider how the diagonals and codiagonals of $\Rel$ interact with each other
and with the previous structures. In fact, the self-dual compact closed
structure is reducible to the (co)diagonals. The unit and counit morphisms are
given by
\begin{equation*}
  \vcenter{\hbox{\includegraphics{img/unit}}}
  \quad=\quad
  \vcenter{\hbox{\includegraphics{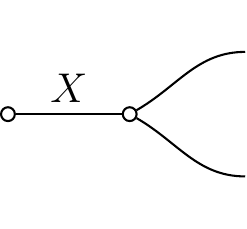}}},
  \hspace{0.5in}
  \vcenter{\hbox{\includegraphics{img/counit}}}
  \quad=\quad
  \vcenter{\hbox{\includegraphics{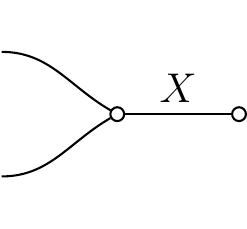}}}.
\end{equation*}
We have seen that the dagger is, in turn, reducible to the compact closed
structure. Like the dagger operation, bending arrows is useful enough to merit
its own textual and graphical syntax.

The internal monoids and comonoids in $\Rel$ combine to form internal
\emph{Frobenius algebras} (sometimes called \emph{Frobenius monoids})
\cite{kock2004}. That is, for each object $X$, there is a monoid
$(X,\Merge_X,\Create_X)$ and a comonoid $(X,\Copy_X,\Delete_X)$ satisfying the
\emph{Frobenius equations}
\begin{equation*}
  (1_X \otimes \Copy_X) (\Merge_X \otimes 1_X)
    = \Merge_X \Copy_X
    = (\Copy_X \otimes 1_X) (1_X \otimes \Merge_X)
\end{equation*}
or, graphically,
\begin{equation*}
  \vcenter{\hbox{\includegraphics{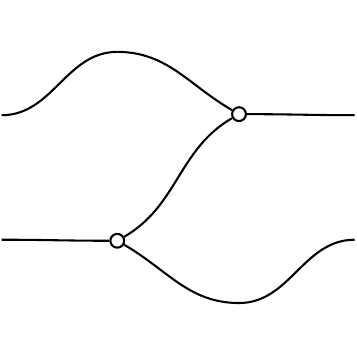}}}
  \quad=\quad
  \vcenter{\hbox{\includegraphics{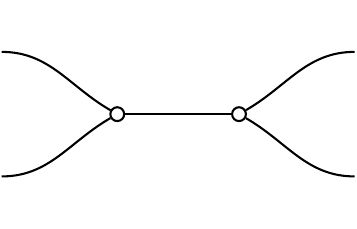}}}
  \quad=\quad
  \vcenter{\hbox{\includegraphics{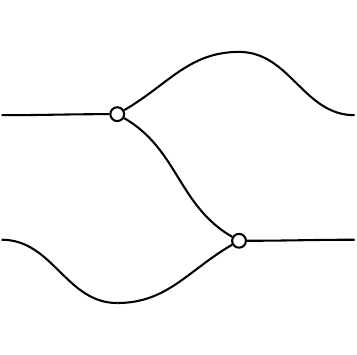}}}.
\end{equation*}
The monoid and comonoids are also \emph{special}, meaning that $\Copy_X
\Merge_X = 1_X$ or
\begin{equation*}
  \vcenter{\hbox{\includegraphics{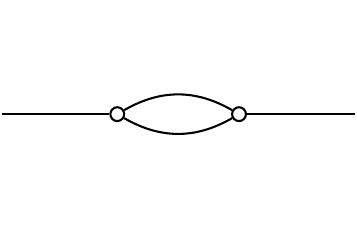}}}
  \quad=\quad
  \vcenter{\hbox{\includegraphics{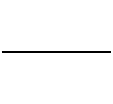}}}.
\end{equation*}
Finally, by definition, we have $\Merge_X = \Copy_X^\dagger$ and $\Create_x =
\Delete_X^\dagger$. These properties can be summarized by saying that
$(X,\Copy_X,\Delete_X,\Merge_X,\Create_X)$ is a \emph{special
$\dagger$-Frobenius monoid} \cite{baez2015a}.

\section{Abstract categories of relations} \label{sec:BiRel}

The category $\Rel$ of sets and relations cannot stand alone as a formalism for
knowledge representation. A knowledge representation system must be
implementable on a computer, which requires that each knowledge base admit a
finite description. Yet $\Rel$, far from being a finitary object, has as objects
every possible set and as morphisms every possible relation! Moreover, there is
no formal system for specifying equations or subsumptions that should hold
between relations. To enable a finite description of categories that ``behave
like'' the category of sets and relations, we must \emph{axiomatize} the salient
structures of $\Rel$. The previous section provides some clues about how to
achieve this axiomatization.

In fact, there are two different notions of an ``abstract'' category of
relations in the category theory literature. The best known is Freyd's
\emph{allegory}, popularized by Freyd and Scedrov \cite{freyd1990} and utilized
in Johnstone's treatise on topos theory \cite{johnstone2002}. There have been a
few efforts to apply allegories to real-world phenomena, e.g., in circuit design
\cite{brown1994a,brown1994b}, logic programming \cite{gallegoarias2012}, and
database modeling \cite{zielinski2013}. Allegories take intersections and the
dagger (called ``reciprocation'') as primitive, characterizing the former by
axioms like reflexivity, commutativity, and, most distinctively, the
\emph{modular law}
\begin{equation*}
  RS \cap T \subseteq (R \cap TS^\dagger)S,
\end{equation*}
where $R \subseteq S$ is, by definition, equivalent to $R \cap S = R$. (The
reader can check that this rather strange law \emph{does} hold in $\Rel$.) The
second notion is the \emph{bicategory of relations}, introduced by Carboni and
Walters \cite{carboni1987,carboni2008}. In bicategories of relations, the
monoidal structures are primitive, while intersections and the dagger are
derived concepts. The two notions are ostensibly quite different, but it can be
shown that the categories of unitary pretabular allegories and of bicategories
of relations are equivalent, in fact isomorphic
\cite{knijnenburg1994,lawler2015}. Thus, the choice of axiomatization is mostly
a matter of preference.

In this paper, we shall take bicategories of relations as our preferred notion
of an ``abstract'' category of relations. An advantage of this choice is that
the graphical language of monoidal categories is immediately available.

\begin{definition}[\cite{carboni1987}]
  A \emph{bicategory of relations} is a locally posetal 2-category $\cat{B}$
  that is also a symmetric monoidal category $(\cat{B},\otimes,I$) with
  diagonals $(X,\Copy_X,\Delete_X)_{X \in \cat{B}}$, such that
  \begin{itemize}
    \item every morphism $R: X \to Y$ is a lax comonoid homomorphism,
    \begin{equation*}
      R \cdot \Copy_Y \implies \Copy_X (R \otimes R), \qquad
      R \cdot \Delete_Y \implies \Delete_X;
    \end{equation*}
    \item the duplication morphisms $\Copy_X$ and deletion morphisms $\Delete_X$
    have right adjoints $\Merge_X := \Copy_X^*$ and $\Create_X := \Delete_X^*$;
    \item the pairs of morphisms $(\Copy_X, \Merge_X)$ obey the Frobenius
    equations.
  \end{itemize}
\end{definition}

We denote by $\BiRel$ the category of (small) bicategories of relations and
structure-preserving functors.

\begin{remark}
  Our definition differs from Carboni and Walter's definition in one respect.
  They ask not for diagonals but only for internal cocommutative comonoids,
  subject to the requirement that they are the \emph{unique} cocommutative
  comonoids with right adjoints. However, it appears that the only use of this
  uniqueness axiom is to derive the coherence axioms \cite[Remark 1.3
  (ii)]{carboni1987}. We think it simpler to just assert the coherence axioms to
  begin with. By omitting the uniqueness axiom, we ensure that the theory of
  bicategories of relations is essentially algebraic (see below).
\end{remark}

Every structure invoked in the definition has been introduced in \cref{sec:Rel},
with the exception of adjoints. In this paper we use ``adjoint'' in the sense of
2-categories \cite{lack2010}. Thus, in a locally posetal 2-category $\cat{C}$, a
morphism $f: A \to B$ is \emph{left adjoint} to $g: B \to A$ (and $g$ is
\emph{right adjoint} to $f$), written $f \dashv g$, if
\begin{equation*}
  1_A \implies fg \qquad\text{and}\qquad gf \implies 1_B.
\end{equation*}
If a morphism $f$ has a right adjoint $g$, then it is unique, for if $g'$ is
another right adjoint, then $g' \Rightarrow g'fg \Rightarrow g$ and, by
symmetry, $g \Rightarrow g'$, so that $g = g'$. Similarly, left adjoints are
unique when they exist. In $\Rel$, a relation $R: X \to Y$ has a right adjoint
$R^*: Y \to X$ if and only if $R$ is a function, in which case $R^* =
R^\dagger$. Together with the discussion in \cref{sec:Rel}, this proves that
$\Rel$ is a bicategory of relations. We shall meet other interesting
bicategories of relations in \cref{sec:instances,sec:logic}.

\begin{table}
  \centering
  \begin{tabular}{llc}
    \toprule
    Structure & Name & Notation and definition \\
    \midrule
    category & composition & $R \cdot S$ \\
    monoidal category & product & $R \otimes S$ \\
      & braiding & $\sigma_{X,Y}$ \\ 
    diagonal & copy & $\Copy_X$ \\
      & delete & $\Delete_X$ \\
    codiagonal & merge & $\Merge_X := \Copy_X^* = \Copy_X^\dagger$ \\
      & create & $\Create_X := \Delete_X^* = \Delete_X^\dagger$ \\
    compact closed & unit & $\eta_X := \Create_X \cdot \Copy_X$ \\
      & counit & $\epsilon_X := \Merge_X \cdot \Delete_X$ \\
    dagger & dagger & $R^\dagger :=
      (\eta_X \otimes 1_Y) (1_X \otimes R \otimes 1_Y) (1_X \otimes \epsilon_Y)$ \\
    logical & intersection & $R \cap S := \Copy_X (R \otimes S) \Merge_Y$ \\
      & true & $\top := \Delete_I \cdot \Create_I = 1_I$ \\
      & local maximum & $\top_{X,Y} := \Delete_X \cdot \Create_Y$ \\
    \bottomrule
  \end{tabular}
  \caption{Summary of morphisms in a bicategory of relations}
  \label{table:BiRel}
\end{table}

Carboni and Walters derive from the axioms of a bicategory of relations all the
categorical structures discussed in \cref{sec:Rel}. The situation is perfectly
analogous to that of $\Rel$. For the reader's convenience, we summarize the
results in \cref{table:BiRel}, using the textual syntax for brevity.

The characterization of maps in $\Rel$ also generalizes to an arbitrary
bicategory of relations. A morphism $R: X \to Y$ in a bicategory of relations
$\cat{B}$ is a \emph{map} if it has a right adjoint $R^*: Y \to X$. Equivalent
conditions are that $R$ is a comonoid homomorphism or that $R$ is left adjoint
to $R^\dagger$ \cite[Lemma 2.5]{carboni1987}. The collection of maps in
$\cat{B}$ is closed under composition and monoidal products and hence forms a
symmetric monoidal category, which we denote by $\Map(\cat{B})$. In the
motivating example, $\Map(\Rel) = \Set$. The diagonal on $\cat{B}$ is natural
when restricted to $\Map(\cat{B})$, making $\Map(\cat{B})$ into a cartesian
category. In fact, $\Map(\cat{B})$ is the largest subcategory of $\cat{B}$ that
is cartesian. Thus, in the terminology of \cite{selinger1999}, $\Map(\cat{B})$
is the \emph{focus} of $\cat{B}$.

\section{Relational ologs} \label{sec:olog}

A categorical framework for knowledge representation, generalizing the category
of sets and relations, emerges almost automatically from the abstractions
developed in the previous section. An ontology in this framework is called a
``relational olog,'' after Spivak and Kent \cite{spivak2012a}. We will define a
relational olog to be any bicategory of relations that admits a finite
description; more precisely, a relational olog is a \emph{finitely presented}
bicategory of relations. Intuitively, a finitely presented bicategory of
relations is the ``generic'' or ``free'' bicategory of relations that contains a
specified finite collection of basic objects, morphisms, and 2-morphisms. It is
analogous to other free constructions in algebra, such as a free vector space or
a finitely presented group. Another, more relevant example is a functional olog
that does not involve limits or colimits, which is just a finitely presented
category.

\begin{definition}
  A \emph{relational ontology log} (or \emph{relational olog}) is a finitely
  presented bicategory of relations.
  
  In more detail, a relational olog is a bicategory of relations $\cat{B}$
  presented by
  \begin{itemize}
    \item a finite set of \emph{basic types} or \emph{object generators};
    
    \item a finite set of \emph{basic relations} or \emph{morphism generators}
    of form $R: X \to Y$, where $X,Y$ are object expressions;
    
    \item a finite set of \emph{subsumption axioms} or \emph{2-morphism
    generators} of form $R \Rightarrow S$, where $R,S$ are well-formed morphism
    expressions with the same domain and codomain.
  \end{itemize}
\end{definition}

Note that while our definition does not explicitly include \emph{equality
axioms}, equality of morphisms can be reduced to subsumption. In the sequel,
axioms of form $R = S$ are understood to be shorthand for the two axioms $R
\Rightarrow S$ and $S \Rightarrow R$.

We hope that the meaning of the definition is intuitively clear but let us be
somewhat more precise about our terminology. By ``well-formed morphism
expressions'' we mean expressions constructed from the morphism generators and
the syntax of bicategories of relations (see \cref{table:BiRel}) such that
domains and codomains are respected in all compositions. Similarly, ``object
expressions'' are expressions constructed from the object generators and the
syntax of monoidal categories ($\otimes$ and $I$). A bicategory of relations
$\cat{B}$ is ``presented by'' a given collection of generators $\cat{B}_0$ if
$\cat{B}$ contains (an isomorphic copy of) $\cat{B}_0$ and if for every other
bicategory of relations $\cat{B}'$ containing $\cat{B}_0$, there exists a unique
functor $F: \cat{B} \to \cat{B}'$ preserving the structure of $\BiRel$ and the
generators $\cat{B}_0$. As usual, this universal property guarantees the
\emph{uniqueness} of $\cat{B}$ up to isomorphism. For readers concerned about
the \emph{existence} of $\cat{B}$ we make the following technical remark.

\begin{remark}
  The preceding definition can be made fully rigorous by formulating the axioms
  of a bicategory of relations as an essentially algebraic theory. Roughly
  speaking, an \emph{essentially algebraic theory} is an algebraic theory that
  allows some operations to be partially defined, provided the domain of
  definition is characterized by equations between total operations
  \cite{freyd1972,piessens1997}. A motivating example is the theory of
  categories, where composition of morphisms is partially defined. It is well
  known that definition by finite presentation, also known as the ``method of
  generators and relations,'' works in any essentially algebraic theory. There
  are several methods for constructing a free model from the syntax of the
  theory. \emph{Generalized algebraic theories}, a reformulation of essentially
  algebraic theories using dependent type theory, provide a particularly elegant
  solution \cite{cartmell1986,pitts1995}.
  
  Alternatively, our foray into categorical logic (\cref{sec:logic}) yields an
  entirely different and fully explicit construction of relational ologs. This
  construction is based not on universal algebra or dependent type theory but on
  the proof theory of a subset of first-order logic.
\end{remark}

\subsection{Example: Friend of a friend} \label{sec:foaf}

While the formal definition of a relational olog is somewhat abstract, the
specification of a particular relational olog is, as a practical matter, simple
and intuitive, thanks to the graphical language of monoidal categories. To
illustrate, we specify a relational olog in a toy domain affectionately called
``friend of a friend'' (or ``FOAF'') \cite{foaf2014,davis2010}. This domain,
involving people, organizations, and their presences online, is often used to
showcase the Semantic Web technologies (RDF and OWL). We take the formal
specification of FOAF as an inspiration only, making no attempt to replicate its
interface or general philosophy.

The basic types of the olog are ``Person'', ``Organization'', ``Number'', and
``String''. We shall introduce the basic relations as we need them. Here are
some essential relations for our ontology:
\begin{align*}
  \vcenter{\smallhbox{\includegraphics{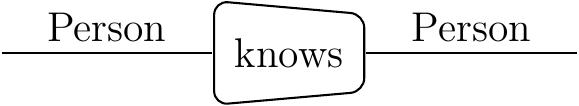}}}, &\hspace{0.5in}
  \vcenter{\smallhbox{\includegraphics{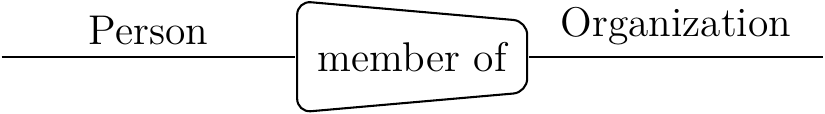}}}, \\
  \vcenter{\smallhbox{\includegraphics{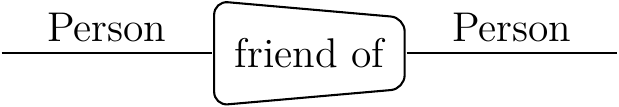}}}, &\hspace{0.5in}
  \vcenter{\smallhbox{\includegraphics{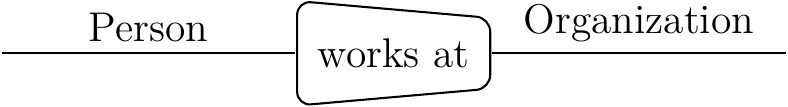}}}.
\end{align*}
The obligatory ``friend of a friend'' relation is just the composite
\begin{equation*}
  \smallbox{\includegraphics{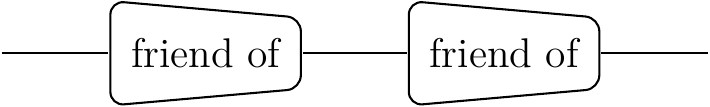}}.
\end{equation*}
(All relations are typed but when the types are clear from context we shall
suppress the type labels.) Some of the basic relations are subsumed by others.
For instance, if Alice is a friend of Bob, then Alice knows Bob; thus, the
``friend of'' relation is subsumed by the ``knows'' relation:
\begin{equation*}
  \vcenter{\smallhbox{\includegraphics{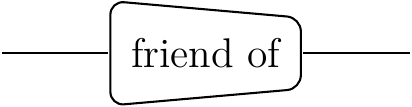}}}
  \quad\implies\quad
  \vcenter{\smallhbox{\includegraphics{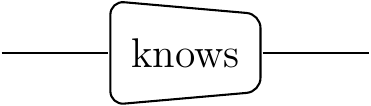}}}.
\end{equation*}
Likewise, the ``works at'' relation is subsumed by the ``member of'' relation
(diagram omitted). Presumably, if Alice knows Bob, then Bob also knows Alice, so
we should declare that the ``knows'' relation is \emph{symmetric}:
\begin{equation*}
  \vcenter{\smallhbox{\includegraphics{img/foaf-knows}}}
  \quad=\quad
  \vcenter{\smallhbox{\includegraphics{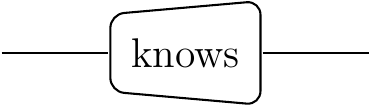}}}.
\end{equation*}
Most people would also say that the ``friend of'' relation is symmetric.

We can attach some basic data to each person, such as their name and age:
\begin{gather*}
  \vcenter{\smallhbox{\includegraphics{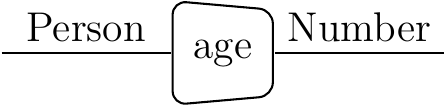}}}, \\
  \vcenter{\smallhbox{\includegraphics{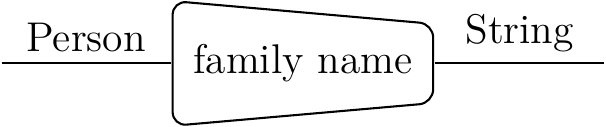}}}, \hspace{0.5in}
  \vcenter{\smallhbox{\includegraphics{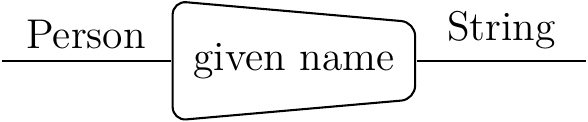}}}.
\end{gather*}
We declare that these relations are (total) \emph{functions}; for instance,
\begin{equation*}
  \vcenter{\smallhbox{\includegraphics{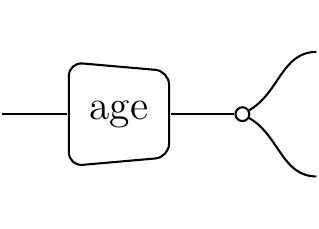}}}
  \quad=\quad
  \vcenter{\smallhbox{\includegraphics{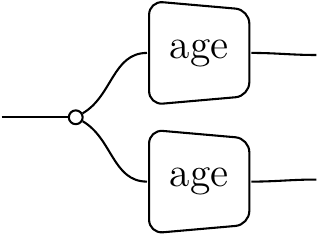}}},
  \hspace{0.5in}
  \vcenter{\smallhbox{\includegraphics{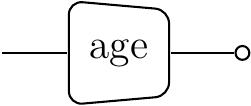}}}
  \quad=\quad
  \vcenter{\smallhbox{\includegraphics{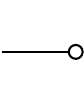}}}.
\end{equation*}
In RDF and OWL, functional relations whose codomains are primitive data types
are called ``properties'' and are treated specially.

For extra flavor, we complement the ``friend of'' relation with an ``enemy of''
relation. We can then define the notorious relation of ``frenemy'' as the
\emph{intersection} of friend and enemy:
\begin{equation*}
  \vcenter{\smallhbox{\includegraphics{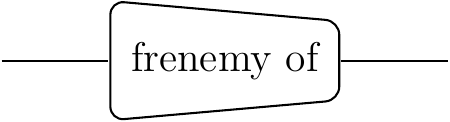}}}
  \quad:=\quad
  \vcenter{\smallhbox{\includegraphics{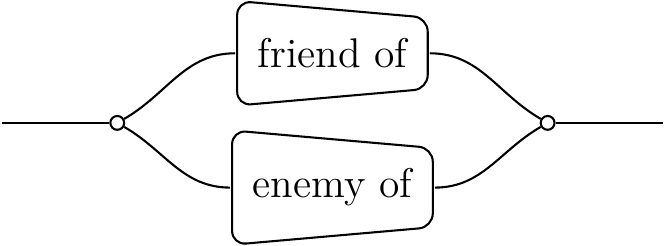}}}.
\end{equation*}

Next, we model some basic family relationships. Having introduced a ``child of''
relation, the ``parent of'' relation is just its \emph{inverse}:
\begin{equation*}
  \vcenter{\smallhbox{\includegraphics{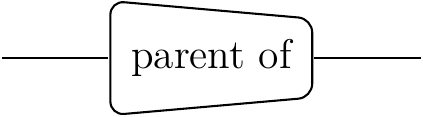}}}
  \quad:=\quad
  \vcenter{\smallhbox{\includegraphics{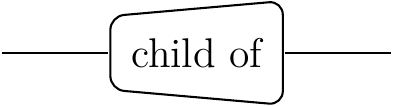}}}.
\end{equation*}
An ``ancestor of'' relation should possess several properties. First, it should
subsume ``parent of'':
\begin{equation*}
  \vcenter{\smallhbox{\includegraphics{img/foaf-parent-of}}}
  \quad\implies\quad
  \vcenter{\smallhbox{\includegraphics{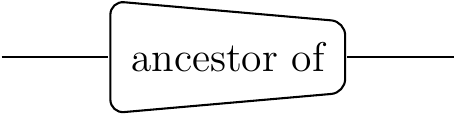}}}.
\end{equation*}
It should be \emph{transitive},
\begin{equation*}
  \vcenter{\smallhbox{\includegraphics{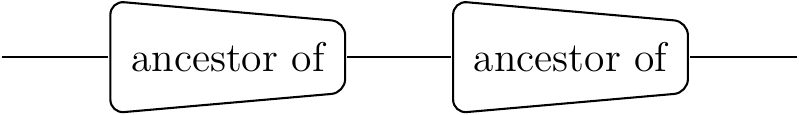}}}
  \quad\implies\quad
  \vcenter{\smallhbox{\includegraphics{img/foaf-ancestor-of}}},
\end{equation*}
(an ancestor of an ancestor is an ancestor) and \emph{reflexive},
\begin{equation*}
  \vcenter{\smallhbox{\includegraphics{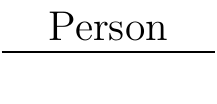}}}
  \quad\implies\quad
  \vcenter{\smallhbox{\includegraphics{img/foaf-ancestor-of}}},
\end{equation*}
(by convention, we regard every person as their own ancestor). Finally, the
ancestor relation should be \emph{antisymmetric},
\begin{equation*}
  \vcenter{\smallhbox{\includegraphics{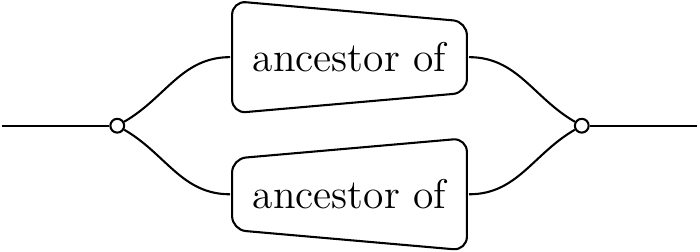}}}
  \quad\implies\quad
  \vcenter{\smallhbox{\includegraphics{img/foaf-person-id}}},
\end{equation*}
(if two people are both ancestors and descendants of each other, then they are
the same person). We can now \emph{deduce}, rather an declare as an axiom, that
the relation ``grandparent of'' is subsumed by ``ancestor of'':
\begin{align*}
  \vcenter{\smallhbox{\includegraphics{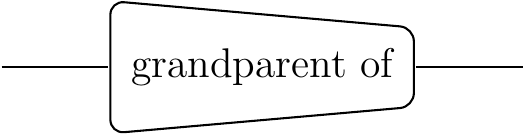}}}
  \quad&:=\quad \vcenter{\smallhbox{\includegraphics{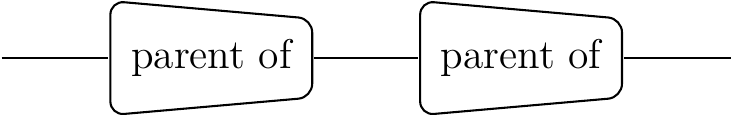}}} \\
  \quad&\implies\quad \vcenter{\smallhbox{\includegraphics{img/foaf-ancestor-transitive}}} \\
  \quad&\implies\quad \vcenter{\smallhbox{\includegraphics{img/foaf-ancestor-of}}}.
\end{align*}
In summary, ``ancestor of'' is a partial order that subsumes ``parent of.'' It
would be more precise to declare that ``ancestor of'' is the partial order
\emph{generated} by ``parent of,'' but that cannot be expressed in a relational
olog.

So far we have seen only relatively simple, binary relations. Let us now
consider more complex compound relations and relations of arity different than
two. The class (unary relation) of employed people can be defined as the class
of people who work at some organization:
\begin{equation*}
  \vcenter{\smallhbox{\includegraphics{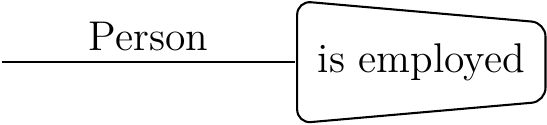}}}
  \quad:=\quad
  \vcenter{\smallhbox{\includegraphics{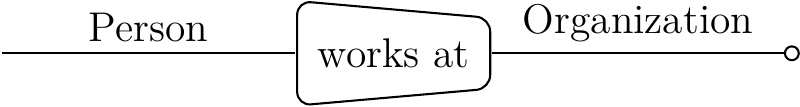}}}.
\end{equation*}
We declare a ternary relation ``salary'' with signature
\begin{equation*}
  \smallbox{\includegraphics{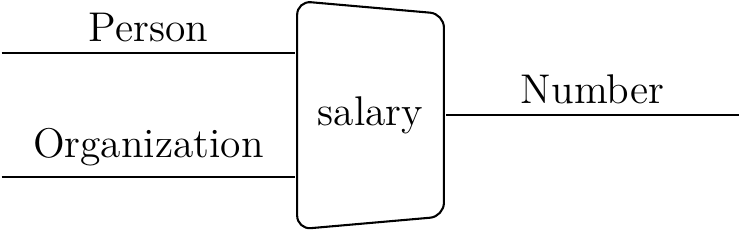}}.
\end{equation*}
We assert that ``salary'' is a \emph{partial function} (diagram omitted). Its
domain of definition is characterized by
\begin{equation*}
  \vcenter{\smallhbox{\includegraphics{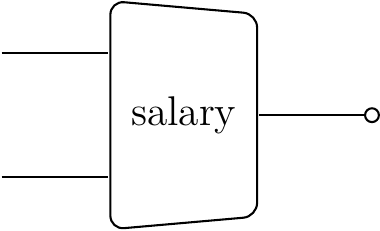}}}
  \quad=\quad
  \vcenter{\smallhbox{\includegraphics{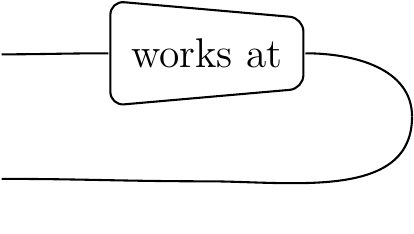}}}.
\end{equation*}
Alternatively, we can take this equation as the definition of the ``works at''
relation: a person works at an organization if and only if they draw a salary
from that organization. As another example, a ``colleague'' is a person whom you
know and with whom you share a membership at some organization:
\begin{equation*}
  \vcenter{\smallhbox{\includegraphics{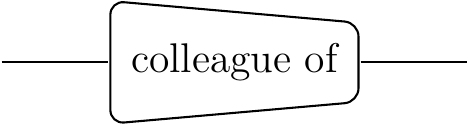}}}
  \quad:=\quad
  \vcenter{\smallhbox{\includegraphics{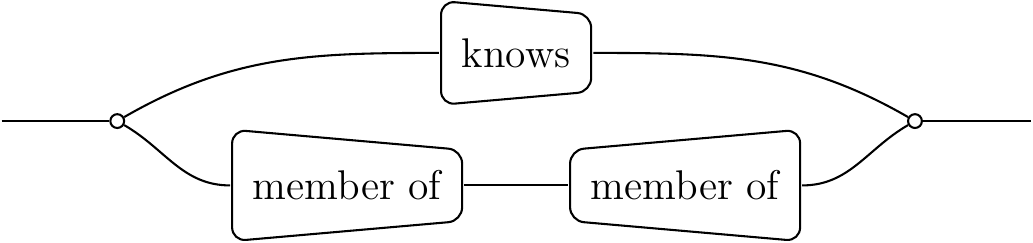}}}.
\end{equation*}
A simple calculation, using the symmetry of ``knows,'' proves that ``colleague
of'' is a symmetric relation. More fancifully, a romantic ``love triangle'' is
the ternary relation
\begin{equation*}
  \smallbox{\includegraphics{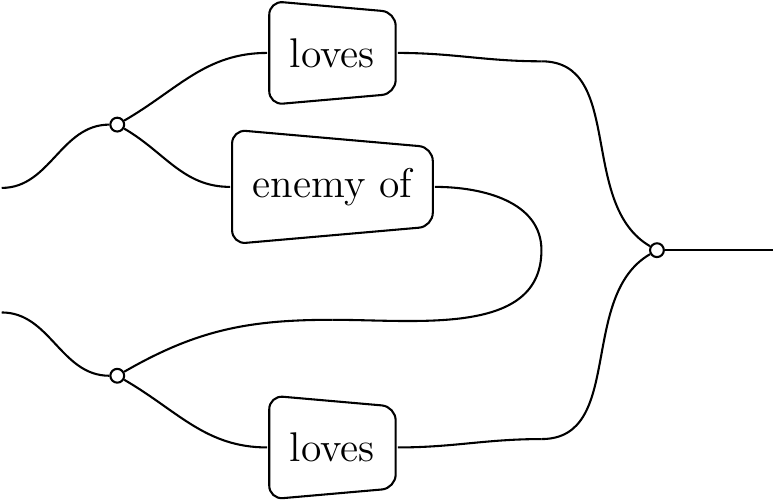}}.
\end{equation*}
Thus, a love triangle consists of two people, mutually enemies, who both love a
third person. Assuming ``enemy of'' is symmetric, this relation is symmetric in
its first two arguments. As long as we're indulging in Shakespearean themes, we
can also define the quaternary relation of ``intergenerational family feud'':
\begin{equation*}
  \smallbox{\includegraphics{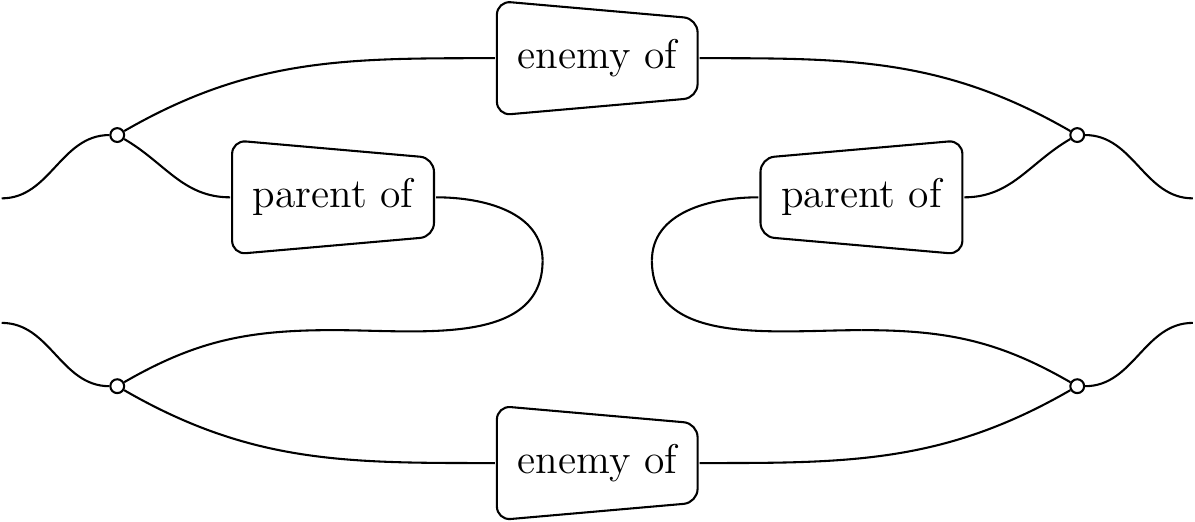}}.
\end{equation*}
Such a feud consists of two parent-child pairs, where the parents are enemies
and the children are also enemies.

Although it might be entertaining to continue along these lines, we shall stop
here. We hope we have convinced the reader that relational ologs are both
expressive and intuitive. With a little practice, it becomes easy to write down
complex relations and read them at a glance. However, there are certain natural
constraints that cannot be expressed in a relational olog, as developed so far.
For instance, if we took our ontology more seriously, we might prefer to dismiss
the possibility of ``frenemies'' and declare that ``friend of'' and ``enemy of''
are disjoint relations. At present we cannot express this constraint because we
cannot express the empty relation. Nor can we express unions, so we cannot
declare that, for example, the ``parent of'' relation is the union of the
``mother of'' and ``father of'' relations. In \cref{sec:expressive-olog} we
explain how to overcome these limitations.

\section{Instance data} \label{sec:instances}

A distinguishing feature of categorical knowledge representation, compared to
the logical paradigm, is a rich and flexible notion of instance data. The idea
of instance data is simply that of \emph{functorality}. To be precise,
\emph{instance data} for a relational olog $\cat{B}$ in an arbitrary bicategory
of relations $\cat{D}$ is a structure-preserving functor $D: \cat{B} \to
\cat{D}$. We call $\cat{D}$ the \emph{data category} for the instance data $D$.
Unsurprisingly, the ``standard'' data category is $\Rel$, the category of sets
and relations. We study this important case and several others below.

In knowledge representation systems based on description logic, instances are
represented by named constants \emph{within the logical system}, usually called
``individuals.'' There are several advantages to the categorical notion of
instance data. First, there is a clean separation between universal concepts,
stored in the olog $\cat{B}$, and instantiations of these concepts, stored in
the functor $D: \cat{B} \to \cat{D}$. In description logic, this separation is
only partly achieved by partitioning the axioms of the knowledge base into a
``TBox'' and an ``ABox'' (see \cref{sec:DL}). Besides its aesthetic appeal, the
separation of universal and particular knowledge has important practical
benefits. In modern ``big data'' applications involving a large number of
individuals, storing instance data in a suitable database, rather than as
logical sentences, becomes a practical necessity. Of course, one can define ad
hoc schemes for translating between the logical system and the database system.
The point is that functors provide a simple, mathematically precise notion of
``translation'' between systems.

Another advantage, less easily achieved by ad hoc devices, is that we can define
``non-standard'' instance data by using data categories besides $\Rel$. This
possibility arises because ologs, unlike logical theories, are algebraic
structures and hence come equipped with a general notion of structure-preserving
maps, namely functors. From this point of view, instance data for relational
ologs is closely connected to \emph{functorial semantics} in categorical logic.
We shall return to categorical logic in \cref{sec:logic}.

Let us add, parenthetically, that it \emph{is} possible to represent individuals
inside a relational olog. An \emph{individual} of type $X$ is a map $c: I \to
X$, since a function from the singleton set $I = \{*\}$ to $X$ picks out an
element of $X$. In our view, individuals should be used sparingly to represent
concepts that are inherently singletons. For example, there is at any given time
only one Dalai Lama, so it would be reasonable to represent the Dalai Lama as an
individual of type ``Person.'' It should suffice to include most ``ordinary''
people only as instance data. In general, the olog should contain only universal
concepts, even if they are singletons, while the instance data contains all
particular knowledge. (We grant that when building an ontology it is not always
easy to distinguish between universal and particular, but often the difference
is clear enough.)

In this section, we consider four different kinds of instance data for
relational ologs. The first and second are interpreted as relational databases
and graph databases, respectively. With minor modification, these two concepts
apply equally well to functional ologs and are treated by Spivak and Kent
\cite{spivak2012a}. The other two kinds of instance data are specific to
relational ologs. From the matrix calculus of relations, an extension of the
familiar boolean algebra, we derive instance data in the category of boolean
matrices. Finally, we consider ``non-standard'' instance data in the category of
linear relations. This data category can be used to model linear dynamical
systems.

\subsection{Relational databases} \label{sec:relational-db}

The default category for instance data, suitable for most applications, is the
category of sets and relations. Thus, without further qualification,
\emph{instance data} for a relational olog $\cat{B}$ is a structure-preserving
functor $D: \cat{B} \to \Rel$.

A structure-preserving functor $D: \cat{B} \to \Rel$ is defined by the following
data. Each basic type $X$ of $\cat{B}$ is mapped to a set $D(X)$ and each basic
relation $R: X \to Y$ of $\cat{B}$ is mapped to a subset $D(R)$ of $D(X) \times
D(Y)$. The set $D(X)$ contains the \emph{instances} of type $X$ and the subset
$D(R)$ tabulates the instances of type $X$ that are in relation $R$ with
instances of type $Y$. By functorality, this data determines the action of $D$
on every object and every morphism of $\cat{B}$, since $\cat{B}$ is generated by
the basic types and relations. Moreover, in order for $D$ to be well-defined, it
must preserve all the subsumption axioms of $\cat{B}$. Thus, for every
subsumption axiom $R \Rightarrow S$, we require that $D(R) \subseteq D(S)$. If
the mapping $D$ satisfies these properties, it defines valid instance data for
the olog.

Instance data is straightforwardly interpreted as a relational (SQL) database
\cite{codd1970}. In the idealized database interpretation, there is a
single-column table $D(X)$ for each basic type $X$, which defines the primary
key of each instance of type $X$, and a multi-column table $D(R)$ for each basic
relation $R$, whose columns are foreign keys associated with the domain and
codomain types of $R$. The tables $D(R)$ are called ``association tables'' or
``junction tables'' in SQL jargon. Association tables are the standard way of
representing many-to-many relationships in a relational database. The primary
key of the association table is the product of the foreign keys of the columns.

An example should make this clear. Instance data for a fragment of the ``friend
of a friend'' ontology (\cref{sec:foaf}) is shown below.
\begin{center}
  \begin{tabular}{|l|}
    \multicolumn{1}{c}{Person} \\ \Xhline{2\arrayrulewidth}
    \textbf{ID} \\ \Xhline{2\arrayrulewidth}
    P1 \\ \hline 
    P2 \\ \hline 
    P3 \\ \hline 
    P4 \\ \hline
  \end{tabular}
  \qquad
  \begin{tabular}{|l|}
    \multicolumn{1}{c}{Organization} \\ \Xhline{2\arrayrulewidth}
    \textbf{ID} \\ \Xhline{2\arrayrulewidth}
    O1 \\ \hline 
    O2 \\ \hline
  \end{tabular}
  \qquad
  \begin{tabular}{|l|l|l|}
    \multicolumn{3}{c}{salary} \\ \Xhline{2\arrayrulewidth}
    Person & Organization & Number \\ \Xhline{2\arrayrulewidth}
    P1 & O1 & 30,000 \\ \hline 
    P3 & O2 & 40,000 \\ \hline
  \end{tabular}
  \\
  \begin{tabular}{|l|l|}
    \multicolumn{2}{c}{friend of} \\ \Xhline{2\arrayrulewidth}
    Person 1 & Person 2 \\ \Xhline{2\arrayrulewidth}
    P1 & P2 \\ \hline 
    P2 & P1 \\ \hline
  \end{tabular}
  \qquad
  \begin{tabular}{|l|l|}
    \multicolumn{2}{c}{knows} \\ \Xhline{2\arrayrulewidth}
    Person 1 & Person 2 \\ \Xhline{2\arrayrulewidth}
    P1 & P2 \\ \hline 
    P2 & P1 \\ \hline
    P3 & P4 \\ \hline
    P4 & P3 \\ \hline
  \end{tabular}
\end{center}
As required by functorality, the ``friend of'' and ``knows'' tables are
symmetric and the ``friend of'' table is a subset of the ``knows'' table. In
practice we expect to deviate slightly from the idealized database
interpretation to obtain a more compact database schema. Most importantly,
instead of representing \emph{maps} as individual tables, e.g.,
\begin{center}
  \begin{tabular}{|l|}
    \multicolumn{1}{c}{Person} \\ \Xhline{2\arrayrulewidth}
    \textbf{ID} \\ \Xhline{2\arrayrulewidth}
    P1 \\ \hline
    P2 \\ \hline
    P3 \\ \hline
    P4 \\ \hline
  \end{tabular}
  \quad
  \begin{tabular}{|l|l|}
    \multicolumn{2}{c}{age} \\ \Xhline{2\arrayrulewidth}
    Person & Number \\ \Xhline{2\arrayrulewidth}
    P1 & 21 \\ \hline
    P2 & 37 \\ \hline
    P3 & 22 \\ \hline
    P4 & 54 \\ \hline
  \end{tabular}
  \quad
  \begin{tabular}{|l|l|}
    \multicolumn{2}{c}{family name} \\ \Xhline{2\arrayrulewidth}
    Person & String \\ \Xhline{2\arrayrulewidth}
    P1 & Doe \\ \hline
    P2 & Smith \\ \hline
    P3 & Williams \\ \hline
    P4 & Jones \\ \hline
  \end{tabular}
  \quad
  \begin{tabular}{|l|l|}
    \multicolumn{2}{c}{given name} \\ \Xhline{2\arrayrulewidth}
    Person & String \\ \Xhline{2\arrayrulewidth}
    P1 & Alice \\ \hline
    P2 & Bob \\ \hline
    P3 & Carol \\ \hline
    P4 & David \\ \hline
  \end{tabular}
  \quad,
\end{center}
it would be more conventional to combine the maps with common domain into a
single table, e.g.,
\begin{center}
  \begin{tabular}{|l|l|l|l|}
    \multicolumn{4}{c}{Person} \\ \Xhline{2\arrayrulewidth}
    \textbf{ID} & age & family name & given name \\ \Xhline{2\arrayrulewidth}
    P1 & 21 & Doe & Alice \\ \hline
    P2 & 37 & Smith & Bob \\ \hline
    P3 & 22 & Williams & Carol \\ \hline
    P4 & 54 & Jones & David \\ \hline
  \end{tabular}
  \quad.
\end{center}
We could represent partial maps similarly, using \texttt{NULL} to indicate
undefined values.

By formalizing the association of instance data with an olog, it becomes
possible to migrate data in a precise, principled way. Suppose that, in light of
new information or a changing world, we decide to update the concepts in our
ontology $\cat{B}$, yielding a new ontology $\cat{B}'$. Ideally we can translate
the concepts of $\cat{B}$ into concepts of $\cat{B}'$ by means of a functor $F:
\cat{B} \to \cat{B}'$. The functor $F$ can then be used to migrate the original
data $D: \cat{B} \to \Rel$ to updated data $D': \cat{B}' \to \Rel$. This
paradigm is called \emph{functorial data migration} and is investigated by
Spivak and collaborators in a series of papers
\cite{spivak2012b,spivak2012a,spivak2015,spivak2016}. Functorial data migration
has been developed for functional ologs. At least one data migration functor,
the pullback functor, has an obvious analogue for relational ologs. It is an
open question whether the other data migration functors, the left and right
pushforward functors, admit analogues. This question, while important, is not
pursued further here.

\subsection{Graph databases} \label{sec:graph-db}

Graph databases \cite{angles2008,rodriguez2010} provide a natural storage model
for the instance data of an ontology. In fact, the development of graph
databases can be traced back to the semantic networks and frame systems of the
early era of knowledge representation \cite[Fig.\ 1]{angles2008}. Even today it
is sometimes suggested that graph databases \emph{are} knowledge representation
systems. That is not so: graph databases offer a generic data storage model that
need not impose any logical constraints on the data. Still, the confusion exists
precisely because graph databases are so well-aligned with the practice of
knowledge representation. In this section, we explain how instance data $D:
\cat{B} \to \Rel$ can be interpreted as a graph database.

Unlike relational databases, which are practically synonymous with the
Structured Query Language (SQL), graph databases are an emerging technology with
no universally accepted data model or query language. Our construction maps
easily onto Apache TinkerPop3 \cite{tinkerpop3}, an open standard for graph
databases with moderate vendor adoption. The Resource Description Framework
(RDF), a core component of the Semantic Web, can also be regarded as a graph
database \cite{angles2005}, especially when coupled with a graph query language
like SPARQL \cite{sparql2008}.

The interpretation of instance data as a graph database involves a construction
called the ``category of elements.''

\begin{definition}  
  Let $\cat{B}$ be a bicategory of relations. The \emph{category of elements} of
  a structure-preserving functor $F: \cat{B} \to \Rel$, denoted $\int_{\cat{B}}
  F$ or $\int F$, has as objects, the pairs
  \begin{equation*}
    (X,x) \qquad\text{where}\qquad X \in \cat{B}, \quad x \in F(X),
  \end{equation*}
  and as morphisms $(X,x) \to (Y,y)$, the morphisms in $\cat{B}$
  \begin{equation*}
    R: X \to Y \qquad\text{such that}\qquad x\,F(R)\,y.
  \end{equation*}
  Composition and identity morphisms are inherited from $\cat{B}$.
\end{definition}

\begin{remark}
  Technically, this definition is not included in the usual notion of a
  \emph{category of elements} \cite[\S 2.4]{riehl2016}, which applies to
  functors $F: \cat{C} \to \Set$, nor in the more general \emph{Grothendieck
  construction} \cite[\S 1.10]{jacobs1999}, which applies to functors $F:
  \cat{C} \to \Cat$. However, it is evidently the same idea, so we will use the
  same terminology.
\end{remark}

A category of elements $\int_{\cat{B}} F$ is itself a bicategory of relations,
with its structure inherited from both $\cat{B}$ and $\Rel$. The monoidal
product is defined by
\begin{equation*}
  (X,x) \otimes (Y,y) := (X \otimes Y, (x,y)), \qquad
  I_{\int F} := (I_{\cat{B}}, *).
\end{equation*}
The diagonal maps are $\Copy_{(X,x)} := \Copy_X$ and $\Delete_{(X,x)} :=
\Delete_X$. These morphisms behave as expected because $F$ is
structure-preserving. For example, there is a copying morphism $\Copy_X: (X,x)
\to (X \otimes X, (x',x''))$ in $\int F$ if and only if $x = x'$ and $x = x''$.
Finally, the 2-morphisms of $\int F$ are just the 2-morphisms of $\cat{B}$. (It
is tempting to declare that $R \Rightarrow S$ in $\int F$ whenever $F(R)
\Rightarrow F(S)$ in $\Rel$, but under that definition $\int F$ is not
necessarily locally posetal.)

Given instance data $D: \cat{B} \to \Rel$, we think of $\int D$ as a graph
database as follows. The vertices of the graph are the objects $(X,x)$ of $\int
D$. The vertex labels (or vertex types) are the objects $X$ of $\cat{B}$, given
by the canonical projection functor $\int_{\cat{B}} D \to \cat{B}$. The directed
edges of the graph are morphisms $(X,x) \xrightarrow{R} (Y,y)$ of $\int D$. The
edge labels (or edge types) are the morphisms $R$ of $\cat{B}$, again given by
the projection functor $\int_{\cat{B}} D \to \cat{B}$. As an example, the
instance data for the FOAF ontology yields the graph database:

\begin{center}
  \begin{tikzpicture}[
      obj/.style={draw,circle,align=center,scale=0.8},
      >=latex, elabel/.style={scale=0.7}]
    \node[obj] (P1) {P1 \\ {\footnotesize Person}};
    \node[obj, right=20ex of P1] (P2) {P2 \\ {\footnotesize Person}};
    \node[obj, below=7.5ex of P1] (P3) {P3 \\ {\footnotesize Person}};
    \node[obj, right=20ex of P3] (P4) {P4 \\ {\footnotesize Person}};
    
    \draw[->] (P1.0) to[bend left=20] node[elabel,above] {knows} (P2.180);
    \draw[->] (P2.180) to[bend left=20] node[elabel,below] {knows} (P1.0);
    \draw[->] (P1.30) to[bend left] node[elabel,above] {friend of} (P2.150);
    \draw[->] (P2.210) to[bend left] node[elabel,below] {friend of} (P1.330);
    
    \draw[->] (P3.0) to[bend left=20] node[elabel,above] {knows} (P4.180);
    \draw[->] (P4.180) to[bend left=20] node[elabel,below] {knows} (P3.0);
  \end{tikzpicture}
\end{center}

As with relational databases, the idealized graph database interpretation may
require modification to accommodate real-world database systems. The size of the
graph can be considerably reduced by representing maps with ``primitive type''
codomains (e.g., ``Number'' or ``String'') as ``vertex properties,'' a feature
supported by most graph databases. Symmetric relations, such as ``knows'' and
``friend of,'' can be represented by one undirected edge instead of two directed
edges. Another issue, not arising with relational databases, is the
representation of relations whose domain or codomain is a product of basic
types, such as the ``salary'' relation. If the database included vertices only
for basic types, we would need directed \emph{hyper}edges \cite{gallo1993}, a
feature not supported by most graph databases. The solution is to include
vertices for product types and edges for the projection morphisms. In fact,
this encoding is accomplished automatically by the monoidal product in the
category of elements.

\subsection{Boolean matrices} \label{sec:BoolMat}

We now consider instance data derived from the matrix calculus of relations.
Unlike relational and graph databases, the matrix calculus has no analogue for
functional ologs. It is a special case of \emph{categorical matrix calculus},
which can be performed in any biproduct category \cite[\S 3.5.5]{coecke2010},
\cite{harding2009}.

Let $\Bool = \{0,1\}$ be the commutative ``rig'' (commutative ring without
negatives) of booleans, whose operations are defined by
\begin{gather*}
  0+0=0, \quad 0+1=1+0=1, \quad 1+1=1 \\
  0 \cdot 0 = 0, \quad 0 \cdot 1 = 1 \cdot 0 = 0, \quad 1 \cdot 1 = 1.
\end{gather*}
That is, addition in $\Bool$ is logical disjunction and multiplication in
$\Bool$ is logical conjunction.

\begin{definition}
  The \emph{category $\BoolMat$ of boolean matrices} has as objects the natural
  numbers and as morphisms $m \to n$ the $m \times n$ matrices over $\Bool$.
  Composition is defined by matrix multiplication and the identity morphisms are
  the identity matrices.
\end{definition}

We interpret a matrix $R \in \Bool^{m \times n}$ as a relation with domain $[m] =
\{1,\dots,m\}$ and codomain $[n] = \{1,\dots,n\}$, where individual $i \in [m]$
is in relation $R$ with individual $j \in [n]$ if and only if $R_{i,j} = 1$. As
expected, composition in $\BoolMat$ is given by existential quantification:
\begin{equation*}
  (R \cdot S)_{i,k} = 1 \qquad\text{iff}\qquad
    \exists j: R_{i,j} = 1 \wedge S_{j,k} = 1.
\end{equation*}

The category of boolean matrices is a bicategory of relations. There is a
2-morphism $R \Rightarrow S$ if and only if $R \leq S$ (elementwise). The
monoidal product is the tensor product of matrices:
\begin{equation*}
  R \otimes S := \begin{pmatrix}
    R_{1,1} S & \cdots & R_{1,n} S \\
    \vdots & \ddots & \vdots \\
    R_{m,1} S & \cdots & R_{m,n} S
  \end{pmatrix}.
\end{equation*}
The diagonals are defined by
\begin{equation*}
  \Copy_n := \begin{pmatrix}
    e_1 e_1^\top & \cdots & e_n e_n^\top
  \end{pmatrix} \in \Bool^{n \times n^2}
  \qquad\text{and}\qquad
  \Delete_n := \begin{pmatrix}
    1 \\ \vdots \\ 1
  \end{pmatrix} \in \Bool^{n \times 1},
\end{equation*}
where $e_i$ is the $i$th standard basis vector. With these definitions, the
dagger is simply the matrix transpose, $R^\dagger = R^\top$. Given matrices $R,S
\in \Bool^{m,n}$, a quick calculation shows that local intersections are given
by the elementwise (Hadamard) product:
\begin{equation*}
  R \cap S := \Copy_m (R \otimes S) \Merge_n
    = \begin{pmatrix}
      R_{1,1} S_{1,1} & \cdots & R_{1,n} S_{1,n} \\
      \vdots & \ddots & \vdots \\
      R_{m,1} S_{m,1} & \cdots & R_{m,n} S_{m,n}
    \end{pmatrix}
    =: R \odot S.
\end{equation*}
Thus we recover the usual intersection of relations.

Anticipating \cref{sec:expressive-olog}, we equip $\BoolMat$ with a
\emph{second} monoidal product, the direct sum of matrices:
\begin{equation*}
  R \oplus S := \begin{pmatrix} R & 0 \\ 0 & S \end{pmatrix}.
\end{equation*}
Define a codiagonal with respect to the direct sum by
\begin{equation*}
  \varMerge_n := \begin{pmatrix} I_n \\ I_n \end{pmatrix} \in \Bool^{2n \times n}
  \qquad\text{and}\qquad
  \varCreate_n := () = \text{$0 \times n$ matrix},
\end{equation*}
where $I_n$ is the $n \times n$ identity matrix. We recover unions of relations
from the formula
\begin{equation*}
  R \cup S := \varCopy_m (R \oplus S) \varMerge_n = R+S.
\end{equation*}
This construction will be revisited and generalized in
\cref{sec:expressive-olog}.

Matrix data---instance data in the category of boolean matrices---for a
relational olog $\cat{B}$ is a structure-preserving functor $D: \cat{B} \to
\BoolMat$. Each basic type $X$ of $\cat{B}$ is mapped to a natural number $D(X)$
and each basic relation $R: X \to Y$ is mapped to a $D(X) \times D(Y)$ matrix
$D(R)$ over $\Bool$, such that all the subsumption axioms of $\cat{B}$ are
satisfied. For example, the instance data for the FOAF ontology becomes
\begin{equation*}
  D(\text{friend of}) = \begin{pmatrix}
    0 & 1 & 0 & 0 \\
    1 & 0 & 0 & 0 \\
    0 & 0 & 0 & 0 \\
    0 & 0 & 0 & 0
  \end{pmatrix}, \qquad
  D(\text{knows}) = \begin{pmatrix}
    0 & 1 & 0 & 0 \\
    1 & 0 & 0 & 0 \\
    0 & 0 & 0 & 1 \\
    0 & 0 & 1 & 0
  \end{pmatrix}.
\end{equation*}
The two matrices are symmetric because the corresponding relations in $\cat{B}$
are. Note that relations with (theoretically) infinite domain or codomain, such
as ``family name,'' ``given name,'' and ``salary,'' cannot be represented as
matrices.

The matrix representation of relations plays an important role in data analysis
applications. A symmetric relation $R$ with the same domain and codomain, such
as the ``friend of'' relation, is often regarded as an undirected graph, with
$D(R)$ its \emph{adjacency matrix}. This simple observation is the starting
point of the spectral analysis of network data, a rich and active area of
statistical research \cite{mahoney2016}. Matrix data offers yet another example
of how instance data can be used to connect an ontology to another computational
system in a mathematically precise way.

\subsection{Linear relations} \label{sec:VectRel}

Relational and graph databases are both realized by functors $D: \cat{B} \to
\Rel$, and matrix data can be regarded as a repackaged functor $D: \cat{B} \to
\FinRel$, where $\FinRel$ is the category of \emph{finite} sets and relations.
On the basis of these examples one might suppose that, in general, instance data
amounts to a functor into $\Rel$. That would be mistaken. In this section, we
describe a counterexample of practical significance, the category of
\emph{linear} relations. This category has been studied by Baez and Erbele
\cite{baez2015a}, and independently by Bonchi, Soboci{\'n}ski, and Zanasi
\cite{bonchi2014}, as a model of signal flow diagrams in control theory. Note
that the material in this section is peripheral to the main development of the
paper and can be skipped without loss of continuity.

\begin{definition}
  The \emph{category of linear relations}, denoted $\VectRel_k$, is the category
  whose objects are finite-dimensional vector spaces (over a fixed field $k$)
  and whose morphisms $L: U \to V$ are \emph{linear relations}, which are vector
  subspaces
  \begin{equation*}
    L \subseteq U \oplus V.
  \end{equation*}
  Composition and identity morphisms are defined as in $\Rel$; thus, given
  linear relations $L: U \to V$ and $M: V \to W$, the composite $LM: U \to W$ is
  \begin{equation*}
    LM = \left\{(u,w)\ \middle|\ 
      \exists v \in V: (u,v) \in L \wedge (v,w) \in M\right\}.
  \end{equation*}
\end{definition}

The category of linear relations is a bicategory of relations. The 2-morphisms
are subspace inclusions. The monoidal product is the direct sum (which we always
write as $\oplus$, not $\otimes$) and the monoidal unit is the zero vector
space. The diagonal is defined by
\begin{equation*}
  \Copy_V : \left\{ \begin{aligned}
    V &\to V \oplus V \\
    v &\mapsto (v,v)
  \end{aligned} \right.
  \qquad\text{and}\qquad
  \Delete_V : \left\{ \begin{aligned}
    V &\to \{0\} \\
    v &\mapsto 0.
  \end{aligned} \right.
\end{equation*}
Given these definitions, the maps in $\VectRel_k$ are just linear maps, in the
usual sense. Be warned that the dagger is \emph{not} the matrix transpose, in
contrast to $\BoolMat$; indeed, the transpose of a linear map is always another
linear map, but $f^\dagger$ is not a linear map unless $f$ is invertible. As in
$\Rel$, the dagger simply effects a formal exchange of inputs and outputs.

However, the linear transpose leads to interesting structure not present in
$\Rel$. The transpose of the duplication map $\Copy_V$ is the \emph{addition}
map
\begin{equation*}
  \varMerge_V := \Copy_V^\top : \left\{ \begin{aligned}
    V \oplus V &\to V \\
    (v_1,v_2) &\mapsto v_1 + v_2
  \end{aligned} \right.
\end{equation*}
and the transpose of the deletion map $\Delete_V$ is the \emph{zero} map
\begin{equation*}
  \varCreate_V := \Delete_V^\top : \left\{ \begin{aligned}
    \{0\} &\to V \\
    0 &\mapsto 0.
  \end{aligned} \right.
\end{equation*}
The linear relations $\varCopy_V := \varMerge_V^\dagger: V \to V \oplus V$ and
$\varDelete_V := \varCreate_V^\dagger: V \to \{0\}$ are called \emph{coaddition}
and \emph{cozero}, respectively.

The family of maps $(\varMerge_V,\varCreate_V)$ form a codiagonal structure on
$\VectRel_k$. Moreover, every linear relation $L: U \to V$ is a \emph{lax monoid
homomorphism} with respect to this structure, meaning that
\begin{equation*}
  (L \oplus L) \varMerge_V \implies \varMerge_U L
  \qquad\text{and}\qquad
  \varCreate_V \implies \varCreate_U L.
\end{equation*}
The duality exhibited here motivates the following definition.

\begin{definition}[{\cite[\S 5]{carboni1987}}]
  An \emph{abelian bicategory of relations} is locally posetal 2-category
  $\cat{B}$ that is also a symmetric monoidal category $(\cat{B},\otimes,I)$
  with diagonals $(X,\Copy_X,\Delete_X)$ and codiagonals
  $(X,\varMerge_X,\varCopy_X)$ such that
  \begin{itemize}
    \item every morphism $R: X \to Y$ is a lax comonoid homomorphism and a lax
    monoid homomorphism;
    \item the morphisms $\Copy_X,\Delete_X,\varMerge_X,\varCreate_X$ have right
    adjoints, denoted $\Merge_X,\Create_X,\varCopy_X,\varDelete_X$;
    \item both pairs $(\Copy_X,\Merge_X)$ and $(\varCopy_X,\varMerge_X)$ obey
    the Frobenius equations.
  \end{itemize}
\end{definition}

\begin{remark}
  An equivalent, more succinct definition is that an \emph{abelian bicategory of
  relations} is a locally posetal 2-category $\cat{B}$ such that both $\cat{B}$
  and $\cat{B}^\oppositeTwo$ are bicategories of relations with respect to the
  same monoidal product. Here $\cat{B}^\oppositeTwo$ is the 2-category $\cat{B}$
  with all 2-morphisms reversed. See \cref{sec:expressive-olog} for further
  discussion.
\end{remark}

The category of linear relations, $\VectRel_k$ is an abelian bicategory of
relations. The usual category of relations, $\Rel$, is \emph{not} an abelian
bicategory of relations.

The category of linear relations would obviously be an inappropriate data
category for the FOAF ontology. Linear relations are useful for representing
systems of linear ordinary differential equations (ODEs), as argued by Baez and
Erbele \cite{baez2015a}. The graphical language of monoidal categories then
formalizes the signal flow diagrams that appear in control theory and other
engineering fields. In this setting, one takes the field $k = \R(s)$, the real
numbers $\R$ with a formally adjoined indeterminate $s$. Upon taking Laplace
transforms, differentiation becomes the linear operation of scalar
multiplication by $s$ and integration becomes multiplication by $1/s$. A linear
relation $L: k^m \to k^n$ is a system of linear, constant-coefficient ODEs with
$m$ input signals and $n$ output signals.

The damped, driven harmonic oscillator provides a simple, one-dimensional
example. The equation of motion is
\begin{equation*}
  m \frac{d^2 x}{dt^2} + \beta \frac{dx}{dt} + \kappa x(t) = F(t),
\end{equation*}
where $x$ is a position or angle and $F$ is a driving force. Provided the
oscillations are not too large, this equation accurately describes a mass on a
spring or a pendulum under gravity, subject to an additional driving force. The
system can be represented by a linear relation $k \to k$ with input $F$ and
output $x$:
\begin{center}
  \includegraphics{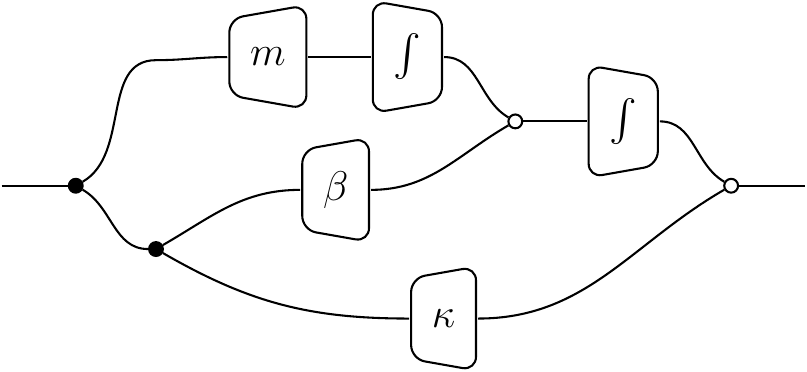}
\end{center}
The dark nodes are coaddition and the light nodes are coduplication. It is
perhaps easier to read the diagram from right to left, noting that the formal
inverse of the integration map $\int$ is the differentiation map $\frac{d}{dx}$.
The same diagram is drawn in conventional engineering notation in \cite[Fig.\
2.1]{friedland2012}. A string diagram for a more complicated system, the
``inverted pendulum,'' is presented in \cite[\S 5]{baez2015a}.

Like the previous data categories, the category of linear relations deserves a
more comprehensive treatment than space here permits. We include it mainly as a
concrete example of the ``non-standard'' instance data enabled by functorial
semantics. The possibility of instance data with extra algebraic or topological
structure is a distinctive---and sometimes useful---feature of categorical
knowledge representation that is not easily replicated in a purely logical
system.

\section{Types and the open-world assumption} \label{sec:types}

The use of \emph{types} is another distinctive feature of categorical knowledge
representation. Unlike instance data, types can be added to existing logical
systems without too much difficulty; we shall do so in \cref{sec:logic}.
Nonetheless, we regard types as distinctive because category theory is typed
``by default,'' while logic is not, and because, as a practical matter, the
knowledge representation systems in common use are untyped. In this section, we
discuss the significance of types for knowledge representation.

The reader may be puzzled by the claim that knowledge representation frameworks
based on description logic are untyped. Isn't the assignment of individuals to
concepts a form of typing? Indeed, isn't a primary purpose of description logic
to taxonomize the types of things existing in a given domain via a hierarchy of
interrelated concepts? If so, what could be the purpose of adding a
\emph{second}, explicit form of typing to the system? These questions have merit
and we shall try to answer them in this section.

First, we explain why description logic is untyped. Consider the relation
``friend of'' from the FOAF ontology (\cref{sec:foaf}). We want to express that
only people can be friends. In description logic, we would use a value
restriction (see \cref{sec:DL,sec:Rel}) to ensure that any two individuals in
the ``friend of'' relation belong to the concept ``Person.'' Then, given any two
individuals belonging to the disjoint concept ``Organization,'' the answer to
the question ``Are the two entities friends?'' would be ``No.'' On one view that
is a perfectly reasonable answer. But on another it is confused. We might argue
that the answer to the question is neither ``yes'' nor ``no'' because the
question does not make sense. Organizations are simply not the kind of things
that can be friends with each other. Merely by asking whether two organizations
are friends, we commit a \emph{category mistake}---using ``category'' in the
sense of Gilbert Ryle, not Eilenberg and Mac Lane. A programmer would call it a
\emph{type error}. These two possible responses illustrate the philosophical
difference between classes and types.

More prosaically, in description logic any concepts are comparable, while in
relational ologs only relations with the same domain and codomain are
comparable. Thus, in description logic, there is a \emph{universal concept}, to
which all individuals belong, and it is possible to take the intersection of any
two concepts. By contrast, a relational olog has only \emph{local} maxima and
\emph{local} intersections within each collection of typed relations $X \to Y$.

These distinctions have practical implications for knowledge representation. To
bring this out, we consider two different methods of constructing a
\emph{taxonomy} of entities in a relational olog. This task, although probably
overrepresented in KR research, is important in many applications. The first
method, effectively untyped, is based on subsumption of concepts with a single
domain type. The second method creates a hierarchy of different types connected
by inclusion maps.

We shall see that the difference between the two methods is related to the open
and closed world assumptions in database theory and knowledge representation
\cite{reiter1978}. Under the \emph{open world assumption}, any statements that
are not deducible from a knowledge base are not assumed to be either true or
false. Under the \emph{closed world assumption}, certain statements that are not
deducible are assumed to be false. The open world assumption is the standard
mode of reasoning in logical systems, including first-order logic and
description logic. The closed world assumption is commonly used in databases and
ruled-based knowledge representation. For example, in the database for the FOAF
ontology (\cref{sec:relational-db}), the absence of a row in the ``friend of''
table for Alice and Carol is interpreted as the absence of friendship between
Alice and Carol---not the absence of \emph{knowledge} of whether Alice and Carol
are friends. More generally, closed world reasoning makes assumptions about what
\emph{is} true based on what is \emph{not} explicitly stated.

To create a taxonomy in the style of description logic, we work with concepts
$C: X \to I$ over a fixed domain type $X$. The subsumption axiom $C \Rightarrow
D$ asserts that every instance of concept $C: X \to I$ is also an instance of
concept $D: X \to I$. A collection of such axioms implicitly defines a hierarchy
of concepts and sub-concepts. Inferences about concepts are made under the open
world assumption. In particular, no two concepts $C,D$ are provably disjoint
unless there is a disjointness axiom $C \cap D \Rightarrow \bot$ (or
disjointness can be inferred from other axioms). Relational ologs support
disjointness axioms through the extensions of \cref{sec:expressive-olog}.

In a relational olog, it is also possible to represent a taxonomy as a hierarchy
of types. For each kind in the taxonomy we define a type $X$. To declare that
type $X$ is a subtype of type $Y$, we add an inclusion map $\iota: X \to Y$. An
\emph{inclusion} of $X$ into $Y$ is simply a morphism $\iota: X \to Y$ that is
an injective map. Injectivity is asserted by the axiom
\begin{equation*}
  \vcenter{\hbox{\includegraphics{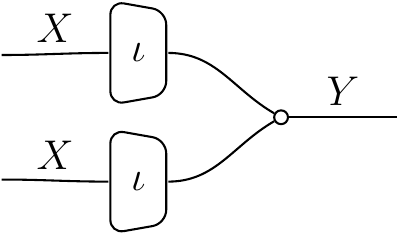}}}
  \quad=\quad
  \vcenter{\hbox{\includegraphics{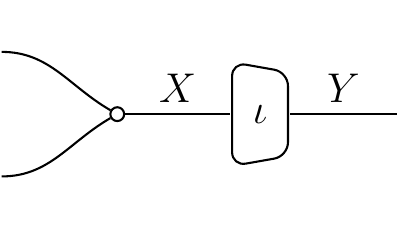}}}.
\end{equation*}
Given instance data $F: \cat{B} \to \Rel$, the function $F(\iota)$ associates
each element of $F(X)$ with a unique element of $F(Y)$, thereby identifying
$F(X)$ with a subset of $F(Y)$. In this way we interpret $X$ as a subtype of
$Y$. Inferences about types in the hierarchy are made under what amounts to a
closed world assumption: unless explicitly stated, distinct types are unrelated.
Two types unconnected by any morphism are not merely ``disjoint,'' they occupy
different universes. It is not permitted to contemplate their intersection.

These design patterns are not mutually exclusive, and we expect that they can be
profitably combined. On one hand, the open world assumption enables inference
about concepts that are not recorded by, or even anticipated by, the creator of
the ontology. It embodies the ethos of the Semantic Web, which, like the World
Wide Web, allows ``anyone to say anything about anything.'' On the other hand,
it can be inconvenient in scientific domains with a large and tightly controlled
vocabulary, such as biology and biomedicine. The Gene Ontology, for example,
contains tens of thousands of concepts related to genes and their biological
functions \cite{ashburner2000}. Many of them are disjoint, and all these
constraints must be recorded. Graphical ontology editors like Prot{\'e}g{\'e}
simplify these tasks \cite{musen2015}, but omissions are still easy to make. We
suggest that a judicious use of typing could eliminate the most embarrassing
errors of this kind.

Some authors have tried to augment description logics like OWL with closed world
reasoning \cite{knorr2011,sengupta2011}. This work is partly motivated by the
need for closed world reasoning when an ontology is used ``like a database'' to
make inferences about particular individuals. In this context, it is instance
data, not typing, that offers a simple solution. Given a relational olog
$\cat{B}$ and instance data $D: \cat{B} \to \Rel$, inference in $\cat{B}$ is
open world (modulo constraints imposed by the type system). By contrast,
inference in the bicategory of relations $D(\cat{B})$, a subcategory of $\Rel$,
is closed world. For any two relations $R,S: X \to Y$, there is a subsumption
$D(R) \Rightarrow D(S)$ in the database if and only if the table $D(R)$ is a
subset of the table $D(S)$. Assuming the database has complete information with
respect to a given population, if a subsumption $D(R) \Rightarrow D(S)$ holds at
the instance level, then we can regard the subsumption $R \Rightarrow S$ as
valid for that population, even if $R \Rightarrow S$ cannot be deduced at the
knowledge level (in $\cat{B}$). This is a form of closed world reasoning. Again,
we see the utility of a clean separation between universal and particular
knowledge.

Another important use of types is to represent ``concrete data types'' like
integers, real numbers, and strings within an ontology. All practical
description logic systems, including OWL, support data types, but only though ad
hoc extensions of the logical language \cite[\S 6.2]{baader2007}. We grant that
software implementations may need to handle primitive data types specially, but
think it inelegant to distinguish in the mathematical formalism between
``abstract'' types like ``Person'' and ``Organization'' and ``concrete'' types
like numbers and strings.

As others have observed, category theory builds a bridge between traditional
mathematical logic and programming language theory. In this setting, it connects
description logic with type theory. Relational ologs are based on simple type
theory. Of the two basic algebraic data types, we already have \emph{product
types}, and we shall introduce \emph{sum types} in \cref{sec:expressive-olog}.
We could conceivably use a more sophisticated type system. For example, instead
representing subtypes by inclusions, we could add first-class subtypes and
polymorphism. Polymorphic and other type theories have been extensively
investigated in the context of categorical logic \cite{crole1993,jacobs1999}.

In the programming language community, it is generally accepted that some amount
of typing increases the robustness and maintainability of software systems
(although opinions differ greatly as to how much typing is desirable). Apart
from low-level assembly languages, there are virtually no programming languages
in common use that are completely untyped. In the same spirit, we argue that at
least some typing is desirable in knowledge representation systems. The extent
and sophistication of the typing will depend on the application and on personal
preferences.

\section{Categorical logic} \label{sec:logic}

There is a fundamental connection between relational ologs and logical
formalisms for knowledge representation. We foreshadowed this connection in
\cref{sec:Rel} by defining the structures that make $\Rel$ into a bicategory of
relations using logical, rather than set-theoretic, notation. In fact, it is
possible to reason about \emph{any} bicategory of relations using first-order
logic, in effect pretending that it is $\Rel$. This conclusion is perhaps
surprising, since some bicategories of relations, such as the category
$\VectRel$ of linear relations (\cref{sec:VectRel}), look ``from the outside''
quite different than $\Rel$.

The purpose of this section, and the attendant \cref{app:regular}, is to make
precise the connection between relational ologs and first-order logic. Our
results belong to \emph{categorical logic}, which represents both syntax and
semantics as categories and interpretations of logical theories as functors. The
field was initiated by Lawvere's seminal thesis on the functorial semantics of
algebraic theories \cite{lawvere1963}. An important result of categorical logic,
perhaps the best known application of category theory to computer science, is
that the simply-typed lambda calculus is the \emph{internal language} of
cartesian closed categories \cite{lambek1988,crole1993,jacobs1999}. In a similar
spirit, we prove that a certain fragment of first-order logic, called ``regular
logic,'' is the internal language of bicategories of relations.

We now explain semiformally the correspondence between regular logic and
bicategories of relations. For details and proofs, we refer to
\cref{app:regular}.

\emph{Regular logic} is the fragment of first-order logic with connectives
$\exists,\wedge,\top,=$. Unlike traditional first-order logic, regular logic is
\emph{typed} (cf.\ \cref{sec:types}). Every variable $x$, free or bound, is
assigned a type $A$, expressed by writing $x: A$. To indicate the types of free
variables, every formula of regular logic is associated with a list of type
assignments, called a \emph{context}. Here are some examples of formulas in
context:
\begin{gather*}
  x:A, y:B \given R(x,y) \wedge S(x,y) \\
  x:A, z:C \given \exists y:B . (R(x,y) \wedge S(y,z)) \\
  x:A, x':A, x'':A \given (x = x') \wedge (x = x'') \\
  x:A \given \top
\end{gather*}
In general, a formula in context has form $\Gamma \given \varphi$, where
$\varphi$ is a formula and $\Gamma$ is a context containing all the free
variables of $\varphi$.

A theory in regular logic, or \emph{regular theory}, is defined by the following
data. There is a set of basic types $A,B,\dots$ and a set of relation symbols
$R,S,\dots$. Each relation symbol $R$ has a fixed signature $(A_1,\dots,A_n)$
that determines its arity and argument types. A regular theory also has a set of
\emph{axioms} of form $\Gamma \given \varphi \vdash \psi$, which we interpret
as: ``by assumption, $\varphi$ implies $\psi$ in the context $\Gamma$.''
Given a regular theory $\theory{T}$, we say that $\Gamma \given \varphi$
\emph{entails} $\Gamma \given \psi$ under the theory $\theory{T}$, or that
$\Gamma \given \varphi \vdash \psi$ is a \emph{theorem} of $\theory{T}$, if
$\Gamma \given \psi$ can be deduced from $\Gamma \given \varphi$ using the
axioms of $\theory{T}$ and the inference rules of the proof system for regular
logic.

As an example, we define a regular theory capturing a fragment of the ``friend
of a friend'' ontology (\cref{sec:foaf}). The types of the theory are
``Person,'' ``Organization,'' ``Number,'' and ``String.'' Its relation symbols
include
\begin{gather*}
  \text{knows}: (\text{Person}, \text{Person}), \qquad
  \text{friend of}: (\text{Person}, \text{Person}), \\
  \text{works at}: (\text{Person}, \text{Organization}), \qquad
  \text{salary}: (\text{Person}, \text{Organization}, \text{Number}).
\end{gather*}
The symmetry of the ``knows'' relation is expressed by the axiom
\begin{equation*}
  x: \text{Person}, y: \text{Person} \given
    \text{knows}(x,y) \vdash \text{knows}(y,x).
\end{equation*}
The ``works at'' relation is determined by
\begin{equation*}
  x: \text{Person}, y: \text{Organization} \given
   \exists z: \text{Number}.(\text{salary}(x,y,z)) \dashv\vdash \text{works at}(x,y),
\end{equation*}
where, as usual, $\dashv\vdash$ is shorthand for $\vdash$ and $\dashv$ (two
axioms).

We establish a correspondence between regular theories and bicategories of
relations. First, to every regular theory $\theory{T}$, we associate a
bicategory of relations $\Cl(\theory{T})$, called the \emph{classifying
category} of $\theory{T}$. The classifying category is constructed directly from
the syntax of regular logic. Its objects are finite lists of types $\tvec{A} =
(A_1,\dots,A_n)$, which can be regarded as $\alpha$-equivalence classes
$[\Gamma]$ of contexts $\Gamma = \tvec{x}: \tvec{A}$. Its morphisms are
equivalence classes of formulas in context $[\Gamma; \Delta \given \varphi]:
[\Gamma] \to [\Delta]$, where equality of formulas is up to $\alpha$-equivalence
(renaming of variables) and deducible logical equivalence under the axioms of
$\theory{T}$. For instance, in the ``friend of a friend'' theory, we have
\begin{equation*}
  [x:\text{Person}; y:\text{Org} \given
    \exists z:\text{Number} . (\text{salary}(x,y,z))]
   = [a:\text{Person}; b:\text{Org} \given \text{works at}(a,b)].
\end{equation*}
The semicolon in the context partitions the free variables into domain and
codomain; it serves no logical purpose. To make $\Cl(\theory{T})$ into a
bicategory of relations, we define composition, products, and diagonals
analogously to $\Rel$ (\cref{sec:Rel}).

Conversely, to every small bicategory of relations $\cat{B}$, we associate a
regular theory $\Lang(\cat{B})$, called the \emph{internal language} of
$\cat{B}$. Its types are the objects of $\cat{B}$ and its relation symbols are
the morphisms of $\cat{B}$. Note that $\Lang(\cat{B})$ necessarily has
infinitely many types and relation symbols! The axioms of $\Lang(\cat{B})$ are
chosen to guarantee that formulas corresponding to equal morphisms are provably
equivalent. To make sense of this statement, we must explain how an arbitrary
formula in context $\Gamma \given \varphi$ of $\Lang(\cat{B})$ can be
\emph{interpreted} as a morphism $\sem{\Gamma \given \varphi}$ of $\cat{B}$. The
mapping $\sem{\cdot}$ is essentially inverse to the constructions making the
classifying category into a bicategory of relations. We prove a \emph{soundness
theorem} for general interpretations of a regular theory $\theory{T}$ in a
bicategory of relations $\cat{B}$, yielding a \emph{categorical semantics} for
regular logic.

The main result of this section, proved in \cref{app:regular}, is that $\Cl$ is
inverse to $\Lang$ in an appropriate sense.

\begin{theorem}
  With respect to typed regular logic, for every small bicategory of relations
  $\cat{B}$, there is an equivalence of categories
  \begin{equation*}
    \Cl(\Lang(\cat{B})) \simeq \cat{B}
    \qquad\text{in}\qquad \BiRel.
  \end{equation*}
\end{theorem}

Consequently, we can regard regular theories and small bicategories of relations
as ``the same.'' Besides enriching our understanding of relational ologs, this
result is potentially practically useful, as it enables the transfer of tools
and techniques between category theory and logic. What should it mean to give
``instance data'' for a description logic knowledge base, assuming it can be
expressed as a regular theory $\theory{T}$? We simply ask what data is required
to give a structure-preserving functor $D: \Cl(\theory{T}) \to \Rel$. What
should be meant by a ``translation'' $\theory{T} \to \theory{T}'$ between
knowledge bases $\theory{T}$ and $\theory{T}'$? Again, we need only ask what
data is needed to give a functor $F: \Cl(\theory{T}) \to \Cl(\theory{T}')$. We
will not carry out these exercises here but it is instructive to do so.

In the other direction, we can reason about a relational olog $\cat{B}$ by
performing logical inference in $\Lang(\cat{B})$. This observation is
significant because the computational aspects of category theory are highly
underdeveloped in comparison with logic. Much research on description logic is
directed towards its computability and complexity theory, and there is a long
tradition of computational first-order logic. Many mature inference engines and
theorem provers exist. By contrast, the theory and practice of computational
category theory, especially higher category theory, is only now emerging
\cite{mimram2014,kissinger2015,bar2016}.

\paragraph{Bibliographic remarks} Regular logic has been thoroughly studied by
categorical logicians as the simplest fragment of first-order logic with a
quantifier \cite{awodey2009,butz1998,van-oosten1995}. Our categorical semantics
of regular logic is quite different from the usual one. Conventionally, a
formula $\Gamma \given \varphi$ is interpreted as a \emph{subobject}
$\sem{\Gamma \given \varphi}$ of the object $\sem{\Gamma}$: an equivalence
class of monomorphisms into $\sem{\Gamma}$. A category suitable for such
interpretations, called a \emph{regular category}, has all finite limits and
``well-behaved'' subobjects. The classifying category of a regular theory is a
regular category whose objects are (equivalence classes of) formulas in context
and whose morphisms are (equivalence classes of) formulas in context that are
provably functional. In our framework, the classifying bicategory of relations
has as objects (equivalence classes of) contexts and as morphisms (equivalence
classes of) formulas in context. The latter perspective seems more natural to
us.
  
The germ of the above theorem is present already in the original paper of
Carboni and Walters \cite[Remark 2.9 (iii)]{carboni1987}, but to our knowledge
has never been carefully developed. There are also strong connections between
bicategories of relations and regular categories. Given a regular category
$\cat{C}$, there is a bicategory of relations $\RelC(\cat{C})$ with the same
objects as $\cat{C}$ and with morphisms $A \to B$ equal to the subobjects of $A
\times B$---a construction that predates and motivates Carboni and Walters'
paper. Yet not every bicategory of relations arises in this way. Conversely, if
a bicategory of relations $\cat{B}$ is functionally complete, then
$\Map(\cat{B})$ is a regular category \cite[Theorem 3.5]{carboni1987}. Yet in
general $\Map(\cat{B})$ need not be a regular category. Thus, to a limited
extent, it is possible to pass between bicategories of relations and regular
categories.

\section{More expressive relational ologs} \label{sec:expressive-olog}

Relational ologs, as developed so far, can express local intersections and
maxima. As proved in \cref{sec:logic}, their internal language is the regular
($\exists,\wedge,\top,=$) fragment of first-order logic. It is natural to ask
for more expressive relational ologs allowing local unions and minima. In
logical terms, they should correspond to the coherent
($\exists,\wedge,\vee,\top,\bot,=$) fragment of first-order logic. In this
section, we develop such highly expressive relational ologs, called
\emph{distributive relational ologs}. We follow the pattern established by
\cref{sec:Rel,sec:BiRel,sec:olog}: first, we present the relevant monoidal
structures on $\Rel$; next, we abstract from $\Rel$ to formulate a general
categorical structure, called a \emph{distributive bicategory of relations};
finally, we define a distributive relational olog to be a finitely presented
distributive bicategory of relations.

\subsection{The category of relations, revisited} \label{sec:Rel2}

The category of relations has another interesting monoidal product, besides the
Cartesian product: the disjoint union. In this section, we explain disjoint
unions from a categorical perspective. As in \cref{sec:Rel}, our presentation
draws on the survey \cite{coecke2010}, especially \S 3.5 on ``classical-like''
monoidal products.

The \emph{disjoint union} (or \emph{tagged union}) is defined on objects of
$\Rel$ by
\begin{equation*}
  X \oplus Y := \{(x,1): x \in X\} \cup \{(y,2): y \in Y\}.
\end{equation*}
An element of $X \oplus Y$ is either an element of $X$ or an element of $Y$,
plus a special tag to avoid ambiguity when $X$ and $Y$ intersect. Given
morphisms $R: X \to Y$ and $S: Z \to W$ of $\Rel$, the \emph{disjoint union} $R
\oplus S: X \oplus Z \to Y \oplus W$ is defined by
\begin{equation*}
  (t,i)(R \oplus S)(s,j) \qquad\text{iff}\qquad
  \begin{cases}
    R(t,s) &\text{if $i=j=1$} \\
    S(t,s) &\text{if $i=j=2$} \\
    \bot &\text{otherwise}
  \end{cases}
\end{equation*}
The monoidal unit is $O := \emptyset$, the empty set. Finally, the braiding
morphism $\bm{\sigma}_{X,Y}: X \oplus Y \to Y \oplus X$ exchanges the tags. With
these definitions, $(\Rel,\oplus,O)$ is a symmetric monoidal category.

The category of relations is now equipped with two monoidal products. In
general, when working with two monoidal products $\otimes$ and $\oplus$, we call
the first product $\otimes$ the \emph{tensor} and the second product $\oplus$
the \emph{cotensor}. To avoid confusion, we will always use ``light'' notation
for structures associated with the tensor and ``dark'' notation for structures
associated with the cotensor. What that means should become clear shortly.

We would like to reason about both monoidal products using a single graphical
language. Unfortunately, that is not entirely straightforward. The basic problem
is that we now have an effectively \emph{three}-dimensional language, with
dimensions corresponding to composition, the tensor, and the cotensor, but
drawing pictures in dimensions greater than two is highly impractical. We
discuss (two-dimensional!) graphical languages for multiple products below. For
the moment, we work exclusively with the cotensor and can therefore employ,
without ambiguity, the usual graphical language of monoidal categories.

We now consider structures derived from the disjoint union. Define the
codiagonals $\varMerge_X: X \oplus X \to X$ and $\varCreate_X: O \to X$ on
$\Rel$ by 
\begin{equation*}
  (x,i) \varMerge_X x' \qquad\text{iff}\qquad x = x'.
\end{equation*}
Note that the initial morphism $\varCreate_X$ must be the (typed) empty relation
$O \to X$. Define the diagonals $\varCopy_X := \varMerge_X^\dagger: X \to X
\oplus X$ and $\varDelete_X := \varCreate_X^\dagger: X \to O$ by duality. As in
$(\Rel,\otimes,I)$, these morphisms form special $\dagger$-Frobenius monoids; in
particular, they satisfy the Frobenius equations.

The disjoint union is logically dual to the Cartesian product. The union $R \cup
S$ of two relations $R,S: X \to Y$ is $\varCopy_X (R \oplus S) \varMerge_Y$ or,
in graphical language,
\begin{equation*}
  R \cup S \quad=\quad \vcenter{\hbox{\includegraphics{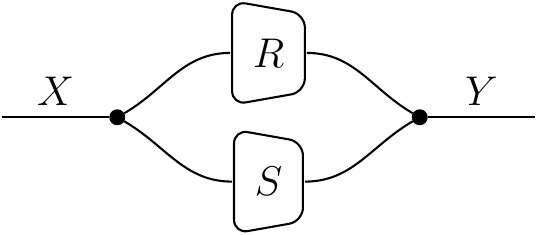}}}.
\end{equation*}
Similarly, the typed empty relation $\bot_{X,Y}: X \to Y$, or local minimum, is
$\varDelete_X \cdot \varCreate_Y$:
\begin{equation*}
  \bot_{X,Y} \quad=\quad \vcenter{\hbox{\includegraphics{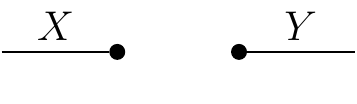}}}.
\end{equation*}
In particular, the boolean $\bot: I \to I$ is
$\bot_{I,I} = \varDelete_I \cdot \varCreate_I$.

There is a categorical interpretation of logical duality. The familiar principle
of 1-categorical duality establishes a correspondence between a category
$\cat{C}$ and its opposite $\cat{C}^\opposite$. In a bicategory of relations,
this form of duality is captured by the dagger functor. By analogy, in a
2-category $\cat{B}$, we can consider the 2-category $\cat{B}^\oppositeTwo$
obtained from $\cat{B}$ by reversing all 2-morphisms. The correspondence between
$\cat{B}$ and $\cat{B}^\oppositeTwo$ is duality at the level of 2-morphisms. If
$\cat{B}$ is a bicategory of relations, then 2-categorical duality is logical
duality. Consider the situation in $\Rel$. The diagonals $\Copy_X$ and
$\Delete_X$ are maps, and the codiagonals $\varMerge_X$ and $\varCreate_X$ are
also maps. Equivalently, the diagonals $\Copy_X$ and $\Delete_X$ are maps in
$\Rel$, while the diagonals $\varCopy_X$ and $\varDelete_X$ are maps in
$\Rel^\oppositeTwo$. Thus $(\Rel^\oppositeTwo,\oplus,O)$ is \emph{also} a
bicategory of relations, provided we verify the axiom on lax monoid
homomorphisms.

In fact, a stronger statement holds, breaking the symmetry between products and
sums. Unlike the diagonals $\Copy_X$ and $\Delete_X$ in $(\Rel,\otimes,I)$, the
codiagonals $\varMerge_X$ and $\varCreate_X$ in $(\Rel,\oplus,O)$ are
\emph{natural}. That is, for every relation $R: X \to Y$, we have $(R \oplus R)
\varMerge_Y = \varMerge_X R$ and $\varCreate_X R = \varCreate_Y$, or graphically
\begin{equation*}
  \vcenter{\hbox{\includegraphics{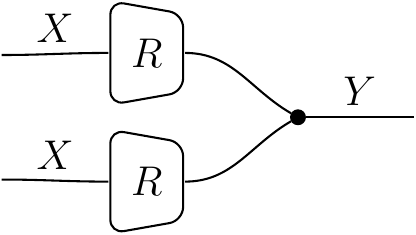}}}
  \quad=\quad
  \vcenter{\hbox{\includegraphics{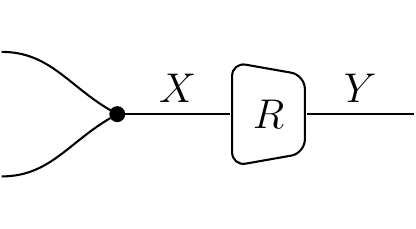}}}
\end{equation*}
and
\begin{equation*}
  \vcenter{\hbox{\includegraphics{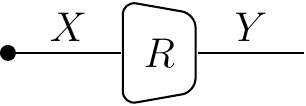}}}
  \quad=\quad
  \vcenter{\hbox{\includegraphics{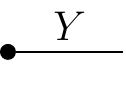}}}.
\end{equation*}
Likewise, we have $\varCopy_X (R \oplus R) = R \varCopy_Y $ and
$R \varDelete_Y = \varDelete_X$. This situation motivates the following
definition.

\begin{definition}
  Let $(\cat{C},\oplus,O)$ be a symmetric monoidal category. The monoidal
  product $\oplus$ on $\cat{C}$ is a
  \begin{itemize}
    \item a \emph{(categorical) product} if there exists a diagonal
    $(\varCopy_A,\varDelete_A)$, natural in $A$;
    \item a \emph{(categorical) coproduct} if there exists a codiagonal
    $(\varMerge_A,\varCreate_A)$, natural in $A$;
    \item a \emph{biproduct} if it both a product and coproduct, such that 
    for any objects $A_1,A_2$, the projection maps
    $\pi_1 = 1_{A_1} \oplus \varDelete_{A_2}$ and
    $\pi_2 = \varDelete_{A_1} \oplus 1_{A_2}$,
    the inclusion maps $\iota_1 = 1_{A_1} \oplus \varCreate_{A_2}$ and
    $\iota_2 = \varCreate_{A_1} \oplus 1_{A_2}$,
    and zero maps $0_{i,j} = \varDelete_{A_i} \varCreate_{A_j}$
    satisfy the equations
    \begin{equation*}
      \iota_i \cdot \pi_j = \delta_{i,j} := \begin{cases}
        1_{A_i} &\text{if $i = j$} \\
        0_{i,j} &\text{if $i \neq j$}
      \end{cases},
      \qquad i,j = 1,2.
    \end{equation*}
  \end{itemize}
\end{definition}

\begin{remark}
  Although it is not immediately obvious, these definitions of ``product'' and
  ``coproduct'' agree with the standard definitions via universal properties
  \cite{heunen2012}.
\end{remark}

The category of relations is a biproduct category with respect to the disjoint
union. Another prime example of a biproduct category is $(\Vect_k,\oplus)$, the
category of finite-dimensional vector spaces and linear maps, equipped with the
direct sum. The category of linear relations, $\VectRel_k$, is \emph{not} a
biproduct category.

\paragraph{Interactions between monoidal products}

We have hitherto studied the disjoint union only in isolation. We now consider
how the Cartesian product and disjoint union interact in $\Rel$. On the objects
of $\Rel$, there is a natural isomorphism
\begin{equation*}
  X \otimes (Y \oplus Z) \cong (X \otimes Y) \oplus (X \otimes Z),
\end{equation*}
given by $(x, (w,i)) \mapsto ((x,w),i)$, that expresses the distributivity of
products over sums. In words: having an element of $X$ and an element of $Y$ or
$Z$ is the same as having elements of $X$ and $Y$ or elements of $X$ and $Z$.
Here is one possible general definition of distributivity in a monoidal
category.

\begin{definition}[\cite{jay1993}]
  A \emph{distributive monoidal category} is a symmetric monoidal category
  $(\cat{C}, \otimes, I)$ with coproduct $\oplus$ that satisfies the
  \emph{distributive law}: for any objects $A,B,C$, the canonical
  distributivity morphism
  \begin{equation*}
    (A \otimes B) \oplus (A \otimes C) \to A \otimes (B \oplus C),
  \end{equation*}
  determined by the universal property of the coproduct, is an isomorphism.
\end{definition}

\begin{remark}
  A \emph{rig category}, or \emph{bimonoidal category}, retains the distributive
  law but relaxes the requirements that the tensor be symmetric and that the
  cotensor be the coproduct. It categorifies the classical algebraic structure
  known as a ``rig'' (ring without negatives) \cite{baez1998}. For the purposes
  of this paper, the extra generality of rig categories is unnecessary.
\end{remark}

Besides $\Set$ and $\Rel$, examples of distributive monoidal categories include
$\Ab$, the category of abelian groups, and $\Vect_k$, the category of
finite-dimensional vector spaces, both equipped the tensor product $\otimes$ and
the direct sum $\oplus$.

The distributive law extends to morphisms of $\Rel$ in a familiar way. For any
three relations $R,S,T: X \to Y$, we have
\begin{equation*}
  R \cap (S \cup T) = (R \cap S) \cup (R \cap T).
\end{equation*}
It is tempting to display this equation diagrammatically:
\begin{equation*}
  \vcenter{\smallhbox{\includegraphics{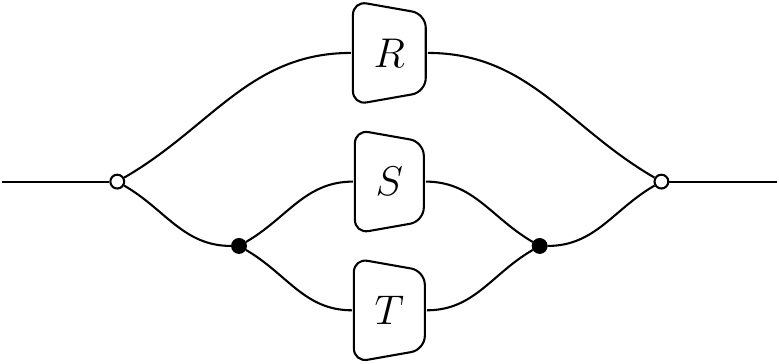}}}
  \quad=\quad
  \vcenter{\smallhbox{\includegraphics{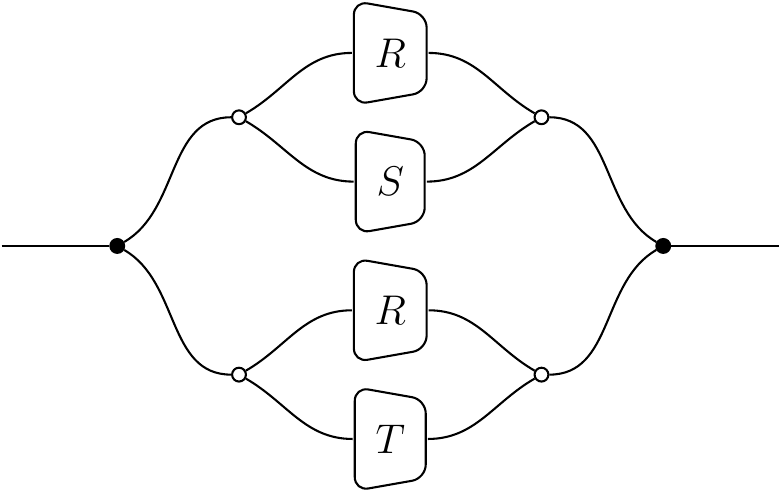}}}.
\end{equation*}
As noted at the beginning of this section, diagrams involving two monoidal
products take us beyond the firmly established graphical language of monoidal
categories. The picture above relies on \emph{context} derived from the copy and
merge nodes to determine which monoidal product is ``active'' at a given point.
In this case the notation is unambiguous, but in general one must take care to
avoid coherence problems, especially when working with the monoidal units. We
conjecture that soundness is maintained if the cotensor is restricted to forming
unions (via the morphisms $\varCopy_X$ and $\varMerge_X$) and local minima (via
$\varDelete_X$ and $\varCreate_X$). This restricted language is already
sufficient for applications that do not directly utilize sum types, a special
case of some practical interest.

In the literature, \emph{proof nets} are established as a graphical calculus for
categories with two monoidal products. Girad introduced proof nets in his
seminal paper on linear logic \cite{girard1987}. Blute et al generalized the
formalism to weakly distributive and $*$-autonomous categories (models of linear
logic), adopting a graphical style reminiscent of string diagrams
\cite{blute1996}. Unfortunately, proof nets are considerably more complicated
than string diagrams, accommodating monoidal units through special ``thinning
links.'' In our view it remains an open problem to define a graphical language
for categories with multiple monoidal products that is provably coherent---sound
and complete---but still simple enough for practical use by nonspecialists.

\subsection{Distributive bicategories of relations} \label{sec:DistBiRel}

Motivated by distributivity in the category of relations, we define a
categorical abstraction called a ``distributive bicategory of relations.''
We begin with the following more general definition.

\begin{definition}
  A \emph{union bicategory of relations} is a locally posetal 2-category
  $\cat{B}$, a symmetric monoidal category $(\cat{B},\otimes,I)$ with diagonals
  $(X,\Copy_X,\Delete_X)$, and a symmetric monoidal category
  $(\cat{B},\oplus,O)$ with codiagonals $(X,\varMerge_X,\varCopy_X)$, such that
  \begin{itemize}
    \item every morphism $R: X \to Y$ is a lax comonoid homomorphism and a lax
    monoid homomorphism;
    \item the morphisms $\Copy_X,\Delete_X,\varMerge_X,\varCreate_X$ have right
    adjoints, denoted $\Merge_X,\Create_X,\varCopy_X,\varDelete_X$;
    \item both pairs $(\Copy_X,\Merge_X)$ and $(\varCopy_X,\varMerge_X)$ obey
    the Frobenius equations.
  \end{itemize}
\end{definition}

Equivalently, a \emph{union bicategory of relations} is a locally posetal
2-category $\cat{B}$ such that both $\cat{B}$ and $\cat{B}^\oppositeTwo$ are
bicategories of relations (not necessarily with respect to the same monoidal
product).

\begin{remark}
  To our knowledge, this definition does not appear in the literature. Note that
  Johnstone's ``union allegory'' insists that unions are preserved by
  composition \cite[\S A3.2]{johnstone2002}; under our definition of ``union
  bicategory of relations,'' the strongest statement that can be made about the
  interaction between unions and composition is the logical dual of the modular
  law, $(R \cup TS^\dagger)S \subseteq RS \cup T$.
\end{remark}

The definition postulates no relationship whatsoever between the two monoidal
products. An abelian bicategory of relations (\cref{sec:VectRel}) is a union
bicategory of relations where the two products coincide. In our main example
$\VectRel_k$, the union $L \cup M$ of two linear relations $L,M \subseteq U
\oplus V$ is the vector space sum $L+M$. The other important special case of a
union bicategory of relations is the distributive bicategory of relations.

\begin{definition}
  A \emph{distributive bicategory of relations} is a union bicategory of
  relations where the cotensor is the categorical coproduct.
  
  We denote by $\DistBiRel$ the category of (small) distributive bicategories of
  relations and structure-preserving functors.
\end{definition}

\begin{remark}
  Carboni and Walters mention ``distributive'' bicategories of relations in
  passing, but do not clearly state a definition \cite[Remark 3.7]{carboni1987}.
  It seems likely that our definition is what they had in mind. Freyd and
  Scedrov utilize an analogous concept of ``distributive allegory''
  \cite{freyd1990}.
\end{remark}

Of course, $\Rel$ is a distributive bicategory of relations. Another example is
the category of boolean matrices, $\BoolMat$, introduced in \cref{sec:BoolMat}.

Several important properties are implicit in the definition. As the name
suggests, the distributive law holds automatically in a distributive bicategry
of relations. In fact, if $(\cat{C},\otimes,I)$ is any compact closed category
and $\oplus$ is the coproduct, then $\cat{C}$ is a distributive monoidal
category \cite{jay1993}. Moreover, it can be shown that intersections distribute
over unions in a distributive bicategory of relations. Also, the coproduct in a
distributive bicategry of relations is automatically a biproduct, by the
symmetry of a dagger category. (Alternatively, products or coproducts in a
compact closed category are always biproducts \cite{houston2008}.)
\cref{table:DistBiRel} summarizes the extra notation associated with a
distributive bicategory of relations, extending \cref{table:BiRel} in
\cref{sec:BiRel}.

\begin{table}
  \centering
  \begin{tabular}{llc}
    \toprule
    Structure & Name & Notation and definition \\
    \midrule
    biproduct category & biproduct & $R \oplus S$ \\
      & braiding & $\bm{\sigma}_{X,Y}$ \\
      & merge & $\varMerge_X$ \\
      & create & $\varCreate_X$ \\
      & copy & $\varCopy_X := \varMerge_X^* = \varMerge_X^\dagger$ \\
      & delete & $\varDelete_X := \varCreate_X^* = \varCreate_X^\dagger$ \\
    logical & union & $R \cup S := \varCopy_X (R \oplus S) \varMerge_Y$ \\
      & false & $\bot := \varDelete_I \cdot \varCreate_I $ \\
      & local minimum & $\bot_{X,Y} := \varDelete_X \cdot \varCreate_Y$ \\
    \bottomrule
  \end{tabular}
  \caption{Summary of extra morphisms in a distributive bicategory of relations}
  \label{table:DistBiRel}
\end{table}

Lastly, following the established pattern, we define a corresponding notion of
olog. All remarks made in \cref{sec:olog} about the meaning of ``finitely
presented'' remain in force.

\begin{definition}
  A \emph{distributive relational olog} is a finitely presented distributive
  bicategory of relations.
\end{definition}

Distributive relational ologs are very expressive. The only connectives of
first-order logic not directly expressible are negation and universal
quantification. However, the \emph{negation}, or \emph{complement}, of a
relation $R: X \to Y$ can be implicitly defined by introducing another relation
$S: X \to Y$ together with the two axioms $R \cap S \Rightarrow \bot_{X,Y}$ and
$\top_{X,Y} \Rightarrow R \cup S$. In graphical language, the axioms are
beautifully symmetric:
\begin{align*}
  \vcenter{\smallhbox{\includegraphics{img/intersection-role}}}
  \quad&\implies\quad
  \vcenter{\smallhbox{\includegraphics{img/local-minimum}}} \\
  \vcenter{\smallhbox{\includegraphics{img/local-maximum}}}
  \quad&\implies\quad
  \vcenter{\smallhbox{\includegraphics{img/union}}}.
\end{align*}
This definition of negation makes sense in any union bicategory of relations. In
$\Rel$, negation is the usual set-theoretic complement; in $\VectRel_k$, it is
the subspace complement (internal direct sum).

\subsection{Categorical logic with product and sum types} \label{sec:logic2}

By now it should be evident that distributive relational ologs correspond, in
some sense, to the fragment of first-order logic with connectives
$\exists,\wedge,\vee,\top,\bot,=$. This fragment is called \emph{coherent logic}
or, in older literature, \emph{geometric logic}. Coherent logic, or variants
thereof, has been used in axiomatic geometry \cite{avigad2009} and in automated
theorem proving \cite{stojanovic2014,ganesalingam2016}, in part because it is
more readily interpretable by humans than richer logics. Coherent logic is
nonetheless as expressive as first-order logic, in the sense that any
first-order theory can be translated into an equivalent coherent theory called
its \emph{Morleyization} \cite[Lemma D1.5.13]{johnstone2002}. In the Morleyized
theory, negations are encoded by the two axioms shown above.

Despite the suggestive analogy, a direct translation of \cref{sec:logic}
founders, due to the presence of sum types. In a bicategory of a relations, any
two objects $X$ and $Y$ have a product $X \otimes Y$, but our system of typed
regular logic does not include product types. That is, the syntax does not
permit the construction of a type $A \times B$ from two basic types $A$ and $B$.
This mismatch is, however, not fatal because products are smuggled into the
logical system as \emph{contexts}: a context ($x:A$, $y:B$) amounts to a single
variable of type $A \times B$. Adding a second monoidal product puts this device
under considerable strain. One could conceivably extend the syntax of a context
to represent a fully ``destructured'' element of arbitrary compound type. A
better solution is to augment the logical language with product types $A \times
B$ and sum types $A+B$, as well as a unit (singleton) type $1$ and zero (empty)
type $0$. The role of contexts is downplayed accordingly.

Product and sum types are ubiquitous in programming language theory. The simply
typed lambda calculus is often treated with product and sum types, in both
classical and categorical settings \cite{selinger2013,lambek1988}. By contrast,
first-order logics with non-trivial type systems are rare. A proof system for
first-order logic with product and sum types does not, to our knowledge, appear
in the literature. We now present such a system, straightforwardly adapted from
the lambda calculus. As in \cref{sec:logic}, the treatment here is informal.
Details and proofs are deferred to \cref{app:coherent}.

\begin{remark}
  A clarification may be helpful to readers acquainted with the Curry-Howard
  correspondence \cite{howard1980,wadler2015}. The interpretation of types as
  propositions, and programs as proofs, is the subject of a large body of
  research. However, we are interested in types \emph{with} logic, not types
  \emph{as} logic. The former is sometimes called ``two-level type theory.''  An
  example is Gambino and Aczel's ``logic-enriched type theory,'' a system of
  first-order logic with dependent types \cite{gambino2006}.
\end{remark}

By using a non-trivial type system, we commit ourselves to a proper treatment of
terms. In \cref{sec:logic}, we did not bother to distinguish the ``terms'' of
regular logic, which are just typed variables $x:A$. We now formally distinguish
two kinds of expressions in context: \emph{formulas} (also called
\emph{propositions}) and \emph{terms}. As before, the formulas are generated by
the equality relation, relation symbols, and logical connectives. The terms are
generated by variables, function symbols, and term constructors for the product
and sum types.

The term constructors are familiar from typed lambda calculus. Given a term $t:
A \times B$, there are \emph{projection} terms $\pi_1(t): A$ and $\pi_2(t): B$,
and given two terms $t:A$ and $s:B$, there is a \emph{pair} term $\pair{t}{s}: A
\times B$. Dually, given terms $t:A$ and $s:B$, there are \emph{inclusion} terms
$\iota_1(t): A+B$ and $\iota_2(s): A+B$, and given terms $t:A+B$, $r:C$, $s:C$,
there is a \emph{copair} term $\copair{t}{x:A.r}{y:B.s}: C$. The copair term, or
``case statement,'' is interpreted as follows. If $t:A+B$ is a value of type
$A$, return the term $r$ with variable $x$ replaced by $t$. If $t:A+B$ is a
value of type $B$, return the term $s$ with variable $y$ replaced by $t$. In
either case, the result is a value of type $C$.

We now sketch the correspondence between coherent logic with product and sum
types and distributive bicategories of relations. To every coherent theory
$\theory{T}$ we associate the classifying category $\Cl(\theory{T})$, a
distributive bicategory of relations. Its objects are the types of $\theory{T}$.
Note the difference from \cref{sec:logic}, where the objects of
$\Cl(\theory{T})$ are finite \emph{lists} of types. The morphisms of
$\Cl(\theory{T})$ are equivalence classes of formulas in context with exactly
two free variables. The types of these variables are the domain and codomain of
the morphism. We make the classifying category into a distributive bicategory of
relations analogously to $\Rel$ (\cref{sec:Rel,sec:Rel2}).

Conversely, every small distributive bicategory of relations $\cat{B}$ has its
internal language $\Lang(\cat{B})$, a coherent theory. As before, we interpret
an arbitrary formula in context $\Gamma \given \varphi$ of $\Lang(\cat{B})$ as a
morphism $\sem{\Gamma \given \varphi}$ of $\cat{B}$. The construction proceeds
analogously to \cref{sec:logic}, with one important addition: terms are
interpreted as \emph{maps}. More precisely, a term in context $\Gamma \given
t:A$ is interpreted as a morphism $\sem{\Gamma \given t:A}: \sem{\Gamma} \to
\sem{A}$ of $\Map(\cat{B})$. Although the connection between terms and maps is
interesting its own right, in the present proof, the interpretation of terms as
maps serves only to establish the base case in the inductive interpretation of
formulas as morphisms.

Our main result, proved in \cref{app:coherent}, is stated below.

\begin{theorem}
  With respect to coherent logic with product and sum types, for every small
  distributive bicategory of relations $\cat{B}$, there is an equivlance of
  categories
  \begin{equation*}
    \Cl(\Lang(\cat{B})) \simeq \cat{B}
    \qquad\text{in}\qquad \DistBiRel.
  \end{equation*}
\end{theorem}

It is also possible to realize the correspondence between regular logic and
bicategories of relations using a richer type system than in \cref{sec:logic}.
The appropriate logic is regular logic with product types and a singleton type
(but not sum types or an empty type).

\begin{theorem}
  With respect to regular logic with product types, for every small bicategory
  of relations $\cat{B}$, there is an equivalence of categories
  \begin{equation*}
    \Cl(\Lang(\cat{B})) \simeq \cat{B}
    \qquad\text{in}\qquad \BiRel.
  \end{equation*}
\end{theorem}

\section{Conclusion and outlook} \label{sec:conclusion}

In this paper, we have propounded a categorical framework for knowledge
representation centered around bicategories of relations. We emphasized three
important features that emerge automatically from category theory: instance
data, types, and graphical syntax. We compared our framework informally to
description logic and formally to the regular and coherent fragments of typed
first-order logic. In this final section, we offer a general perspective on
categorical knowledge representation. We also suggest directions for future
research.

We have extensively discussed the relationship between the algebraic and logical
approaches to knowledge representation, but have said comparatively little about
how the two categorical frameworks---functional and relational ologs---are
related. Although a complete answer is beyond the scope of this work, we will
suggest a ``pattern'' or ``template'' for defining categorical ontologies that
is general enough to encompass these and other frameworks. This template can
perhaps serve as a first step towards a unified methodology of categorical
knowledge representation.

Doctrines are a useful organizing principle for category theory \cite{kock1977}.
Informally, a \emph{doctrine} is a family of categories or higher categories
with extra structure. The most basic doctrine is the doctrine of categories
(with no extra structure). There are also doctrines of categories with finite
products, of symmetric monoidal categories, of compact closed categories, of
2-categories, of bicategories of relations, etc. Besides the categories
themselves, a doctrine specifies the relevant kind of ``structure-preserving''
functors, and the natural transformations between these. The concept of doctrine
can be formalized, but for us this informal understanding is perfectly adequate.

Here is a general recipe for constructing a categorical knowledge base. First,
choose a doctrine. This choice should be informed by the phenomena being
modeled; we expand on this idea below. Next, define a finitary specification
language for the doctrine. As above, the basic strategy is finite presentation,
a.k.a.\ the method of generators and relations, which works in considerable
generality. If the doctrine supports arbitrary limits or colimits, the method of
\emph{sketches} can be used instead \cite{wells1993,makkai1997}. Alternatively,
if the doctrine has as its internal language some well-known logical system,
that system could serve as a specification language. Finally, use the
specification language to define an ontology. As a mathematical object, the
ontology is simply a finitely generated category of the doctrine. Instance data
for the ontology is derived from the ``prototype'' category of the doctrine,
such as $\Set$ or $\Rel$, or possibly from a more exotic category.

Some doctrines relevant to knowledge representation are listed in
\cref{table:doctrine}. Functional ologs arise from the doctrine of categories
with finite limits and colimits; relational ologs from the doctrine of
bicategories of relations. The typed lambda calculus is the internal language of
the doctrine of cartesian closed categories. Besides its central role in
programming language theory, the lambda calculus has been used to model natural
languages \cite{heim1998}. This list by no means exhausts the doctrines that are
potentially useful for knowledge representation.

\begin{table}
  \centering
  \begin{tabular}{lp{0.85in}p{2in}}
    \toprule
    Doctrine & Prototype(s) & Internal language \\
    \midrule
    category & $\Set$ & \\
    category with finite (co)limits & $\Set$ & \\
    bicategory of relations & $\Rel$
      & regular logic with product types \\
    distributive bicategory of relations & $\Rel$
      & coherent logic with product and sum types \\
    cartesian closed category & $\Set$, $\Cpo$
      & typed lambda calculus with product types \\
    bicartesian closed category & $\Set$
      & typed lambda calculus with product and sum types \\
    \bottomrule
  \end{tabular}
  \caption{Selected doctrines relevant to knowledge representation}
  \label{table:doctrine}
\end{table}

Different doctrines are appropriate for different applications. Bicategories of
relations are designed to model classes of entities (concepts) and the
relationships between them (roles). Description logic shares this orientation.
However, creating taxonomies of concepts is hardly the only worthwhile
application of knowledge representation. As an example, the original impetus for
this work was our need to model knowledge about computer programs used in data
analysis. Description logics (and relational ologs) are ill-suited to this
project; a doctrine related to the lambda calculus is much more appropriate. In
general, we worry that the mainstream of KR research has unjustifiably
privileged taxonomies over other kinds of knowledge. Our philosophy is that
category theory is a universal modeling language enabling a more expansive
understanding of knowledge representation. In the future, we hope to see
practical, flexible knowledge representation systems that allow doctrines to be
rapidly assembled from the categorical toolbox to meet the needs of particular
applications.

There are myriad directions for future research on categorical knowledge
representation. We mention a few that are directly relevant to relational ologs.

A glaring omission in this work is any discussion of automated inference. By
contrast, computationally tractable inference has been the prime directive of
the description logic community. The first step for our project is to
acknowledge that inference in a relational olog is undecidable. This is true
even without the extensions of \cref{sec:expressive-olog}. (One can see this
algebraically, by reduction from the word problems for monoids or groups, or
logically, by reduction from the decision problem for regular theories.)

There are two possible responses to the problem of undecidability. We could
follow the DL community in imposing language restrictions to achieve provable
computational tractability. The extensive DL literature would doubtless be very
helpful in carrying out this program. However, we worry that imposing ad hoc
restrictions would do irredeemable violence to the formalism's elegance and
expressivity. A second approach is to allow an unrestricted language and settle
for approximate inference. We share with Doyle and Patil the opinion that this
approach is undervalued by the description logic community \cite{doyle1991}. In
the statistics and machine learning communities, the need for approximate
inference in complex models is an accepted fact of life. The contrast is
especially stark because inference in first-order theories is \emph{harder} than
inference in probabilistic models. In any event, developing inference algorithms
for relational ologs, exact or approximate, is an important prerequisite for
practical applications.

Another problem, already raised in \cref{sec:expressive-olog}, is to define a
graphical language for distributive bicategories of relations that is coherent,
yet intuitive. (We think that proof nets fall short on the second count.) The
graphical language of string diagrams is a very appealing feature of relational
ologs. We hope that a satisfactory extension of string diagrams to categories
with multiple monoidal products will be discovered.

\addsec{Acknowledgments}

I thank John Baez, David Spivak, and Ryan Wisnesky for helpful comments on the
manuscript. I am grateful to Thomas Icard for his encouragement of this project.
Finally, I thank the many contributors to the $n$Lab for creating an invaluable
online resource for aspiring and professional category theorists.

\printbibliography[heading=bibintoc]

\appendix
\section{Regular logic and bicategories of relations} \label{app:regular}

The objective of this appendix to state carefully and prove the theorem of
\cref{sec:logic}, establishing a correspondence between regular logic and
bicategories of relations. The development will be detailed yet terse, as we
have already explained the main ideas behind the theorem in \cref{sec:logic}.

We first define a formal system for regular logic. The syntax is borrowed from
\cite{awodey2009} and the proof system from \cite{johnstone2002}. We depart from
the standard formulation of regular logic only by dispensing with function
symbols. This convention is merely a convenience; \cref{app:coherent} shows how
to incorporate function symbols.

\begin{definition}
  A \emph{(multisorted) signature} consists of
  \begin{itemize}
    \item a set of \emph{sorts} or \emph{basic types}, which we write
    generically as $A,B,A_1,B_1,\dots$, and
    \item a set of \emph{relation symbols}, which we write as $R,S,\dots$, where
    each relation symbol $R$ is associated with a (possibly empty) ordered list
    of types $(A_1,\dots,A_n)$.
  \end{itemize}
\end{definition}

For brevity, we often use the vector notation $\tvec{A} := (A_1,\dots,A_n)$.
To express that relation symbol $R$ has types $\tvec{A}$, we write $R:
\tvec{A}$ or $R: (A_1,\dots,A_n)$.

There is a countably infinite set of \emph{variables} $x,y,z,\dots$. Unlike in
some typed logical calculi, the variable symbols do not have fixed types.
Instead, we write $x: A$ to express that $x$ is a variable of type $A$. A
\emph{context} is a (possibly empty) finite, ordered list of form
\begin{equation*}
  \Gamma = \tvec{x}: \tvec{A} = (x_1: A_1, \dots, x_n: A_n),
\end{equation*}
where the $x_i$'s are distinct variables and the $A_i$'s are basic types.

\begin{figure}
  \centering
  {\def\arraystretch{2.5}\tabcolsep=4ex
  \begin{tabular}{cr}
    $\vcenter{\prftree{\Gamma \given x_1: A_1}{\cdots}{\Gamma \given x_n: A_n}
                      {\Gamma \given R(x_1,\dots,x_n)}}$
      & (relation symbol $R: \tvec{A}$) \\
    $\vcenter{\prftree{\Gamma \given x:A}{\Gamma \given y:A}
                      {\Gamma \given x=y}}$
      & (equality) \\
    $\vcenter{\prftree{\Gamma \given \varphi}{\Gamma, x:A \given \varphi}}$
      & (weakening) \\
    $\Gamma \given \top$
      & (truth) \\
    $\vcenter{\prftree{\Gamma \given \varphi}{\Gamma \given \psi}
                      {\Gamma \given \varphi \wedge \psi}}$
      & (conjunction) \\
    $\vcenter{\prftree{\Gamma, x:A \given \varphi}
                      {\Gamma \given \exists x:A.\varphi}}$
      & (existential quantifier)
  \end{tabular}}
  \caption{Formation rules for regular logic}
  \label{fig:regular-formation}
\end{figure}

With respect to a fixed signature, the \emph{formulas in context} of the regular
language are expressions of form $\Gamma \given \varphi$, where $\Gamma$ is a
context and $\varphi$ is a formula, as defined inductively by the formation
rules in \cref{fig:regular-formation}. Formulas outside of a context have no
definite meaning. Note that the only \emph{terms in context} of the regular
language are variables, for which there is a single formation rule:
\begin{equation*}
  \Gamma, x:A, \Gamma' \given x:A.
\end{equation*}
A \emph{sequent} is an expression of form $\Gamma \given \varphi \vdash
\psi$, where $\Gamma \given \varphi$ and $\Gamma \given \psi$ are both formulas
in context. The inference rules for sequents are listed in
\cref{fig:regular-inference}.

\begin{figure}
  \centering
  {\def\arraystretch{2.5}\tabcolsep=0ex
  \begin{tabular}{cr}
    $\varphi \vdash \varphi$
      & (identity) \\
    $\vcenter{\prftree{\varphi \vdash \chi}{\chi \vdash \psi}
                      {\varphi \vdash \psi}}$
      & (cut) \\
    $\vcenter{\prftree
      {\tvec{x}:\tvec{A} \given \varphi \vdash \psi}
      {\Gamma \given \tvec{t}: \tvec{A}}
      {\Gamma \given \varphi[\tvec{t}/\tvec{x}] \vdash \psi[\tvec{t}/\tvec{x}]}}$
      & (substitution) \\
    $\varphi \vdash x=x
     \qquad
     (\tvec{x}=\tvec{y}) \wedge \varphi \vdash \varphi[\tvec{y}/\tvec{x}]$
      & (equality) \\
    $\varphi \vdash \top$
      & (truth) \\
    $\vcenter{\prftree{\varphi \vdash \psi}{\varphi \vdash \chi}
                      {\varphi \vdash \psi \wedge \chi}}
     \qquad
     \varphi \wedge \psi \vdash \varphi
     \qquad
     \varphi \wedge \psi \vdash \psi$
      & (conjunction) \\
    $\vcenter{\prftree[double]{\Gamma, x:A \given \varphi \vdash \psi}
                      {\Gamma \given (\exists x:A.\varphi) \vdash \psi}}$
      & (existential quantifier) \\
    $\Gamma \given \varphi \wedge (\exists x:A.\psi)
      \vdash \exists x:A.(\varphi \wedge \psi) \qquad [x \notin \Gamma]$
      & (Frobenius)
  \end{tabular}}
  \caption{Inference rules for regular logic}
  \label{fig:regular-inference}
\end{figure}

The statement of the inference rules assumes the following notational
conventions. Contexts that are constant across the premises and conclusion of a
rule are omitted. The formulas in context appearing in the rules are implicitly
assumed to be well-formed. For example, the existential quantifier rule assumes
that $x$ does not appear freely in $\psi$ because $\Gamma \given \psi$ must be
well-formed. The vector notation $\tvec{x} = \tvec{y}$ is shorthand for $x_1 =
y_1 \wedge \cdots \wedge x_n = y_n$; likewise, $\exists \tvec{x}: \tvec{A}$ is
shorthand for $\exists x_1:A_1 \cdots \exists x_n:A_n$. The expression
$\varphi[\tvec{y}/\tvec{x}]$ denotes the simultaneous substitution of $y_i$ for
$x_i$, for $1 \leq i \leq n$, in the formula $\varphi$.

\begin{remark}
  Several inference rules deserve further comment.
  \begin{itemize}
    \item \emph{Substitution}: Useful special cases of the substitution rule
    include the \emph{weakening} and \emph{strengthening} rules
    \begin{equation*}
      \vcenter{\prftree{\Gamma \given \varphi \vdash \psi}
                       {\Gamma, x:A \given \varphi \vdash \psi}}
      \qquad\qquad
      \vcenter{\prftree{\Gamma, x:A \given \varphi \vdash \psi}{\Gamma \given t:A}
                       {\Gamma \given \varphi \vdash \psi}}.
    \end{equation*}
    
    \item \emph{Existential quantifier}: Given the other rules, the
    bidirectional existential quantifier rule is equivalent to the
    \emph{$\exists$-introduction} and \emph{$\exists$-elimination} rules
    \cite[Lemma 4.1.8]{jacobs1999}
    \begin{equation*}
      \vcenter{\prftree{\Gamma \given t:A}{\Gamma \given \varphi \vdash \psi[t/x]}
                       {\Gamma \given \varphi \vdash \exists x:A.\psi}}
      \qquad\qquad
      \vcenter{\prftree{\Gamma \given \varphi \vdash \exists x:A.\psi}
                       {\Gamma, x:A \given \psi \vdash \chi}
                       {\Gamma \given \varphi \vdash \chi}}.
    \end{equation*}
    
    \item \emph{Frobenius}: The so-called ``Frobenius axiom,'' linking
    conjunction and existential quantification, is superfluous in full
    first-order logic with implication but is \emph{not} deducible from the
    other rules of regular logic \cite[p.\ 831]{johnstone2002}. The converse
    rule \emph{is} deducible \cite[p.\ 832]{johnstone2002}. The omission of the
    Frobenius rule in some standard texts on regular logic, such as
    \cite{van-oosten1995} and \cite{butz1998}, is apparently an error.
  \end{itemize}
\end{remark}

\begin{definition}
  A \emph{regular theory} (with respect to a fixed signature) is defined by a
  set of sequents in the signature, not necessarily finite, called the
  \emph{axioms} of theory. Under a regular theory $\theory{T}$, a formula
  $\varphi$ \emph{entails} $\psi$, written $\Gamma \given \varphi
  \vdash_{\theory{T}} \psi$, if the sequent $\Gamma \given \varphi \vdash \psi$
  is deducible from the axioms of $\theory{T}$ using the inference rules of
  regular logic (\cref{fig:regular-inference}). In this case we say that the
  sequent $\Gamma \given \varphi \vdash \psi$ is an \emph{entailment} or
  \emph{theorem} of $\theory{T}$.
\end{definition}

We now begin to establish the correspondence between bicategories of relations
and regular logic by constructing the classifying category of a regular theory.

\begin{definition}
  The \emph{classifying category} of a regular theory $\theory{T}$, denoted
  $\Cl(\theory{T})$, is the bicategory of relations defined as follows. Its
  objects are finite lists of basic types $\tvec{A}: (A_1,\dots,A_n)$. Given a
  context $\Gamma = \tvec{x}: \tvec{A}$, we also write $[\Gamma] := \tvec{A}$.
  Its morphisms $\tvec{A} \to \tvec{B}$ are equivalence classes of formulas in
  context,
  \begin{equation*}
    [\tvec{x}:\tvec{A}; \tvec{y}:\tvec{B} \given \varphi],
  \end{equation*}
  where the equivalence relation $\sim$ is defined by
  \begin{equation*}
    (\tvec{x}:\tvec{A}, \tvec{y}:\tvec{B} \given \varphi) \sim
      (\tvec{x}':\tvec{A}, \tvec{y}':\tvec{B} \given \varphi')
    \qquad\text{iff}\qquad
    \tvec{x}:\tvec{A}, \tvec{y}:\tvec{B} \given
      \varphi \dashv \vdash_{\theory{T}}
      \varphi'[\tvec{x}/\tvec{x}', \tvec{y}/\tvec{y}'].
  \end{equation*}
  In other words, the morphisms of $\Cl(\theory{T})$ are formulas in context up
  to $\alpha$-equivalence and logical equivalence under $\theory{T}$. Here
  $\dashv \vdash_{\theory{T}}$ is shorthand for $\dashv_{\theory{T}}$ and
  $\vdash_{\theory{T}}$, and the semicolon in a context $(\Gamma;\Gamma')$ is an
  extralogical marker that partitions the context into the domain $[\Gamma]$ and
  codomain $[\Gamma']$ of the morphism $[\Gamma;\Gamma' \given \phi]$. The
  2-morphisms of $\Cl(\theory{T})$ are the entailments of $\theory{T}$:
  \begin{equation*}
    [\Gamma;\Gamma' \given \varphi] \implies [\Gamma;\Gamma' \given \psi]
    \qquad\text{iff}\qquad
    \Gamma,\Gamma' \given \varphi \vdash_{\theory{T}} \psi.
  \end{equation*}
  
  We now define the requisite structures to make the classifying category into a
  bicategory of relations. Composition of morphisms is given by
  \begin{equation*}
    [\tvec{x}:\tvec{A}; \tvec{y}:\tvec{B} \given \varphi]
      \cdot
    [\tvec{y}:\tvec{B}; \tvec{z}:\tvec{C} \given \psi]
      := 
    [\tvec{x}:\tvec{A}; \tvec{z}:\tvec{C} \given
      \exists \tvec{y}: \tvec{B} . \varphi \wedge \psi],
  \end{equation*}
  and the identity morphisms are
  \begin{equation*}
    1_{\tvec{A}} := [\tvec{x}:\tvec{A}; \tvec{x}':\tvec{A} \given
      \tvec{x} = \tvec{x}'].
  \end{equation*}
  The monoidal product is defined on objects by $\tvec{A} \otimes \tvec{B} :=
  (\tvec{A},\tvec{B})$ and on morphisms by
  \begin{equation*}
    [\Gamma;\Delta \given \varphi] \otimes [\Gamma';\Delta' \given \psi]
      := [\Gamma,\Gamma';\Delta,\Delta' \given \varphi \wedge \psi].
  \end{equation*}
  The monoidal unit is the empty list $I := ()$. The braidings are
  \begin{equation*}
    \sigma_{\tvec{A},\tvec{B}}
      := [\tvec{x}:\tvec{A},\tvec{y}:\tvec{B};
          \tvec{y}':\tvec{B},\tvec{x}':\tvec{A} \given
          (\tvec{x} = \tvec{x}') \wedge (\tvec{y} = \tvec{y}')].
  \end{equation*}
  Finally, the diagonals are
  \begin{align*} 
    \Copy_{\tvec{A}}
      &:= [\tvec{x}:\tvec{A}; \tvec{x}':\tvec{A}, \tvec{x}'':\tvec{A} \given
          (\tvec{x} = \tvec{x}') \wedge (\tvec{x} = \tvec{x}'')] \\
    \Delete_{\tvec{A}} &:= [\tvec{x}:\tvec{A} \given \top].
  \end{align*}
\end{definition}

\begin{lemma}
  The classifying category $\Cl(\theory{T})$ of a regular theory $\theory{T}$ is
  a bicategory of relations.
\end{lemma}
\begin{proof}
  
  We must check that every axiom of a bicategory of relations can be deduced
  from the inference rules of regular logic. The proofs are tedious but
  mechanical; we sketch a few to illustrate what is involved and leave the rest
  to the reader.
  
  First, we show that $\Cl(\theory{T})$ is a category. To prove that
  composition is associative, we must show that the two formulas
  \begin{gather*}
    \tvec{x}:\tvec{A}; \tvec{w}:\tvec{D} \given
      \exists \tvec{y}:\tvec{B}.
        (\varphi \wedge \exists \tvec{z}:\tvec{C}.(\psi \wedge \chi)) \\
    \tvec{x}:\tvec{A}; \tvec{w}:\tvec{D} \given
      \exists \tvec{z}:\tvec{C}.
        (\exists \tvec{y}:\tvec{B}.(\varphi \wedge \psi) \wedge \chi)
  \end{gather*}
  are equivalent in regular logic. In fact, both formulas are equivalent to
  \begin{equation*}
    \tvec{x}:\tvec{A}; \tvec{w}:\tvec{D} \given
      \exists \tvec{y}:\tvec{B}.\exists \tvec{z}:\tvec{C}.
      (\varphi \wedge \psi \wedge \chi).
  \end{equation*}
  The derivation relies crucially on the Frobenius rule. We omit the details
  and the proof of the identity axiom of a category.

  Next, we prove that $\Cl(\theory{T})$ is a locally posetal 2-category. It is
  obvious that $\Cl(\theory{T})$ is locally posetal. The vertical composition
  axiom is immediate from the cut rule. To prove the horizontal composition
  axiom, suppose we have two entailments
  \begin{equation*}
    \tvec{x}:\tvec{A}, \tvec{y}:\tvec{B} \given \varphi \vdash \theta
    \qquad\text{and}\qquad
    \tvec{y}:\tvec{B}, \tvec{z}:\tvec{C} \given \psi \vdash \chi.
  \end{equation*}
  By the conjunction and weakening rules,
  \begin{equation*}
    \tvec{x}:\tvec{A}, \tvec{y}:\tvec{B}, \tvec{z}:\tvec{C} \given
      \varphi \wedge \psi \vdash \theta \wedge \chi.
  \end{equation*}
  Using the $\exists$-introduction rule and then the bidirectional
  $\exists$-rule, we obtain
  \begin{equation*}
    \tvec{x}:\tvec{A}, \tvec{z}:\tvec{C} \given
      \exists \tvec{y}:\tvec{B}.(\varphi \wedge \psi) \vdash
      \exists \tvec{y}:\tvec{B}.(\theta \wedge \chi),
  \end{equation*}
  proving the validity of horizontal composition.
  
  The axiom on lax comonoid homomorphisms amounts to two easily proved
  entailments, namely
  \begin{equation*}
    \tvec{x}:\tvec{A}; \tvec{y}':\tvec{B},\tvec{y}'':\tvec{B} \given
      \exists \tvec{y}:\tvec{B}.
        (\varphi \wedge (\tvec{y} = \tvec{y}') \wedge (\tvec{y} = \tvec{y}''))
    \vdash \varphi[\tvec{y}'/\tvec{y}] \wedge \varphi[\tvec{y}''/\tvec{y}]
  \end{equation*}
  and
  \begin{equation*}
    \tvec{x}:\tvec{A} \given \exists \tvec{y}:\tvec{B}.\varphi \vdash \top.
  \end{equation*}
  Finally, we mention that the adjoints of $\Copy_{\tvec{A}}$ and 
  $\Delete_{\tvec{A}}$ exist and are equal to
  \begin{equation*}
    \Merge_{\tvec{A}} :=
      [\tvec{x}':\tvec{A},\tvec{x}'':\tvec{A}; \tvec{x}:\tvec{A} \given
        (\tvec{x}'=\tvec{x}) \wedge (\tvec{x}''=\tvec{x})]
    \qquad\text{and}\qquad
    \Create_{\tvec{A}} := [;\tvec{x}:\tvec{A} \given \top]. \qedhere
  \end{equation*}
\end{proof}

Next, we construct the internal language of a bicategory of relations. A
preliminary definition is:
\begin{definition}
  An \emph{interpretation} or \emph{model} of a signature in a bicategory of
  relations $\cat{B}$ is specified by
  \begin{itemize}
    \item for every basic type $A$, an object $\sem{A}$ of $\cat{B}$;
    \item for every relation symbol $R:\tvec{A}$, a morphism
    $\sem{R}: \sem{\tvec{A}} \to I$ of $\cat{B}$, where we define
    $\sem{\tvec{A}} := \sem{A_1} \otimes \cdots \otimes \sem{A_n}$.
  \end{itemize}
  If $\Gamma = \tvec{x}: \tvec{A}$ is a context, we also write
  $\sem{\Gamma} := \sem{\tvec{A}}$.
  
  An interpretation of a signature extends to the full regular language in that
  signature. By induction on the formation rules of regular logic
  (\cref{fig:regular-formation}), we assign to each formula in context
  $\Gamma,\Gamma' \given \varphi$ a morphism $\sem{\Gamma;\Gamma' \given
  \varphi}: \sem{\Gamma} \to \sem{\Gamma'}$ of $\cat{B}$. In interpreting a
  rule, we allow ourselves to order the context variables and place the context
  semicolon however is most convenient, with the understanding that any other
  arrangement can be achieved by a suitable braiding and bending of wires. This
  flexibility greatly simplifies the notation.
  
  \begin{itemize}
    \item \emph{Relation symbol}: Given a relation symbol $R: \tvec{A}$,
    \begin{equation*}
      \sem{\tvec{x}:\tvec{A}, \Gamma \given R(\tvec{x})}
        := \sem{R} \otimes \Delete_{\sem{\Gamma}}.
    \end{equation*}
    
    \item \emph{Equality}: There are two cases: when the variables are distinct,
    set
    \begin{equation*}
      \sem{\Gamma, x:A; y:A \given x=y}
        := \Delete_{\sem{\Gamma}} \otimes 1_{\sem{A}};
    \end{equation*}
    when the variables are equal, set
    \begin{equation*}
      \sem{\Gamma \given x=x} := \Delete_{\sem{\Gamma}}.
    \end{equation*}
    
    \item \emph{Weakening}: Given a morphism
    $\sem{\Gamma \given \varphi}: \sem{\Gamma} \to I$,
    \begin{equation*}
      \sem{\Gamma, x:A \given \varphi}
        := \sem{\Gamma \given \varphi} \otimes \Delete_{\sem{A}}.
    \end{equation*}
    
    \item \emph{Truth}: $\sem{\Gamma;\Gamma' \given \top} := 
      \top_{\sem{\Gamma},\sem{\Gamma'}} = 
      \Delete_{\sem{\Gamma}} \Create_{\sem{\Gamma'}}$.
    
    \item \emph{Conjunction}: Given morphisms
    $\sem{\Gamma;\Gamma' \given \varphi}: \sem{\Gamma} \to \sem{\Gamma'}$ and
    $\sem{\Gamma;\Gamma' \given \psi}: \sem{\Gamma} \to \sem{\Gamma'}$,
    \begin{equation*}
      \sem{\Gamma;\Gamma' \given \varphi \wedge \psi}
        := \Copy_{\sem{\Gamma}} \cdot
           (\sem{\Gamma;\Gamma' \given \varphi} \otimes 
            \sem{\Gamma;\Gamma' \given \psi}) \cdot
           \Merge_{\sem{\Gamma'}}.
    \end{equation*}
    
    \item \emph{Existential quantifier}: Given a morphism
    $\sem{\Gamma; x:A \given \varphi}: \sem{\Gamma} \to \sem{A}$,
    \begin{equation*}
      \sem{\Gamma \given \exists x:A.\varphi}
        := \sem{\Gamma; x:A \given \varphi} \cdot \Delete_{\sem{A}}.
    \end{equation*}
  \end{itemize}
\end{definition}

\begin{definition}
  An interpretation $\sem{\cdot}$ of the signature of a regular theory
  $\theory{T}$ in a bicategory of relations $\cat{B}$ is an
  \emph{interpretation} or \emph{model} of $\theory{T}$ in $\cat{B}$ if it
  satisfies all the axioms of $\theory{T}$, i.e, for every axiom
  \begin{equation*}
    \Gamma \given \varphi \vdash \psi \quad\text{of}\quad \theory{T},
  \end{equation*}
  there is a 2-morphism
  \begin{equation*}
    \sem{\Gamma \given \varphi} \implies \sem{\Gamma \given \psi}
    \quad\text{in}\quad \cat{B}.
  \end{equation*}
\end{definition}

The previous lemma can be interpreted as a \emph{completeness} theorem, stating
that the inference rules of regular logic are sufficient to prove every axiom of
a bicategory of relations. The next lemma is a \emph{soundness} theorem: it says
that every inference rule of regular logic is valid in an arbitrary bicategory
of relations.
\begin{lemma}
  Let $\sem{\cdot}$ be an interpretation of a regular theory $\theory{T}$ in a
  bicategory of relations $\cat{B}$. For every theorem
  \begin{equation*}
    \Gamma \given \varphi \vdash_{\theory{T}} \psi
    \quad\text{of}\quad \theory{T},
  \end{equation*}
  there is a 2-morphism
  \begin{equation*}
    \sem{\Gamma \given \varphi} \implies \sem{\Gamma \given \psi}
    \quad\text{in}\quad \cat{B}.
  \end{equation*}
\end{lemma}
\begin{proof}
  The proof is by induction on the derivation of a theorem of $\theory{T}$. By
  the definition of an interpretation, every axiom of $\theory{T}$ holds in
  $\cat{B}$. Therefore, it suffices to show that every inference rule of regular
  logic (\cref{fig:regular-inference}) is valid in $\cat{B}$. We sketch these
  proofs below.
  \begin{itemize}
    \item \emph{Identity}: Existence of identity 2-morphism.
    \item \emph{Cut}: Vertical composition of 2-morphisms.
    \item \emph{Substitution}: Omitted.
    \item \emph{Equality}: Omitted.
    
    \item \emph{Truth}: By the lax comonoid homomorphism axiom,
    \begin{equation*}
      \sem{\Gamma \given \varphi}
       = \sem{\Gamma \given \varphi} \cdot \Delete_I
       \implies \Delete_{\sem{\Gamma}}
       = \sem{\Gamma \given \top}.
    \end{equation*}
    
    \item \emph{Conjunction}: If $\sem{\Gamma;\Gamma' \given \varphi} \Rightarrow
    \sem{\Gamma;\Gamma' \given \psi}$ and $\sem{\Gamma;\Gamma' \given \varphi}
    \Rightarrow \sem{\Gamma;\Gamma' \given \chi}$, then
    \begin{align*}
      \sem{\Gamma;\Gamma' \given \varphi}
        &= \Copy_{\sem{\Gamma}}
            (\sem{\Gamma;\Gamma' \given \varphi} \otimes 
             \sem{\Gamma;\Gamma' \given \varphi})
            \Merge_{\sem{\Gamma'}} \\
        &\Rightarrow \Copy_{\sem{\Gamma}}
            (\sem{\Gamma;\Gamma' \given \psi} \otimes 
             \sem{\Gamma;\Gamma' \given \chi})
            \Merge_{\sem{\Gamma'}} \\
        &= \sem{\Gamma;\Gamma' \given \psi \wedge \chi}.
    \end{align*}
    That proves the first conjunction rule. For the second, calculate
    \begin{align*}
      \sem{\Gamma;\Gamma' \given \varphi \wedge \psi}
        &= \Copy_{\sem{\Gamma}}
            (\sem{\Gamma;\Gamma' \given \varphi} \otimes 
             \sem{\Gamma;\Gamma' \given \psi})
            \Merge_{\sem{\Gamma'}} \\
        &\Rightarrow \Copy_{\sem{\Gamma}}
            (\sem{\Gamma;\Gamma' \given \varphi} \otimes 
             \Delete_{\sem{\Gamma}} \Create_{\sem{\Gamma'}})
            \Merge_{\sem{\Gamma'}} \\
        &= \sem{\Gamma;\Gamma' \given \varphi}.
    \end{align*}
    The proof of the third rule is similar.
    
    \item \emph{Existential quantifier}: Fix formulas $\sem{\Gamma; x:A \given
    \varphi}: \sem{\Gamma} \to \sem{A}$ and $\sem{\Gamma \given \psi}:
    \sem{\Gamma} \to I$. By the weakening formation rule, $\sem{\Gamma; x:A
    \given \psi} = \sem{\Gamma \given \psi} \cdot \Create_{\sem{A}}$. If there
    is an entailment $\sem{\Gamma; x:A \given \varphi} \Rightarrow \sem{\Gamma;
    x:A \given \psi}$, then
    \begin{align*}
      \sem{\Gamma \given \exists x:A.\varphi}
        &= \sem{\Gamma; x:A \given \varphi} \Delete_{\sem{A}} \\
        &\Rightarrow \sem{\Gamma; x:A \given \psi} \Delete_{\sem{A}} \\
        &= \sem{\Gamma \given \psi} \Create_{\sem{A}} \Delete_{\sem{A}} \\
        &= \sem{\Gamma \given \psi}.
    \end{align*}
    The proof of the converse rule is similar.
    
    \item \emph{Frobenius}: Given formulas $\sem{\Gamma \given \varphi}:
    \sem{\Gamma} \to I$ and $\sem{\Gamma, x:A \given \psi}: \sem{\Gamma} \to
    \sem{A}$, compute
    \begin{align*}
      \sem{\Gamma \given \varphi \wedge (\exists x:A.\psi)}
        &= \Copy_{\sem{\Gamma}}
          (\sem{\Gamma \given \varphi} \otimes 
           \sem{\Gamma \given \exists x:A.\psi}) \\
        &= \Copy_{\sem{\Gamma}}
          (\sem{\Gamma \given \varphi} \otimes 
           \sem{\Gamma, x:A \given \psi} \Delete_{\sem{A}}) \\
        &= \Copy_{\sem{\Gamma}}
          (\sem{\Gamma \given \varphi} \Create_{\sem{A}} \otimes
           \sem{\Gamma, x:A \given \psi})
          \Merge_{\sem{A}} \Delete_{\sem{A}} \\
        &= \Copy_{\sem{\Gamma}}
          (\sem{\Gamma, x:A \given \varphi} \otimes
           \sem{\Gamma, x:A \given \psi})
          \Merge_{\sem{A}} \Delete_{\sem{A}} \\
        &= \sem{\Gamma; x:A \given \varphi \wedge \psi} \Delete_{\sem{A}} \\
        &= \sem{\Gamma \given \exists x:A.(\varphi \wedge \psi)}. \qedhere
    \end{align*}
  \end{itemize}
\end{proof}

\begin{definition}
  The \emph{internal language} of a small bicategory of relations $\cat{B}$ is
  the regular theory $\Lang(\cat{B})$ defined as follows. Its signature consists
  of
  \begin{itemize}
    \item a basic type $A$ for every object $A$ of $\cat{B}$, and
    \item a relation symbol $R: (A_1,\dots,A_n)$ for every morphism 
    $R: A_1 \otimes \cdots \otimes A_n \to I$ of $\cat{B}$.
  \end{itemize}
  A sequent $\Gamma \given \varphi \vdash \psi$ is an axiom of $\Lang(\cat{B})$
  if and only if $\sem{\Gamma \given \varphi} \Rightarrow \sem{\Gamma \given
  \psi}$ in $\cat{B}$, where $\sem{\cdot}$ is the obvious interpretation of the
  signature of $\Lang(\cat{B})$ in $\cat{B}$.
\end{definition}

By the lemma, $\Lang(\cat{B})$ is interpretable in $\cat{B}$ and the theorems of
$\Lang(\cat{B})$ are exactly the 2-morphism of $\cat{B}$.

\begin{remark}
  In general, a single morphism of $\cat{B}$ gives rise to many relation symbols
  of $\Lang(\cat{B})$, e.g., if $R: A \to I$ is morphism and $A = A_1 \otimes
  A_2$, then there are relation symbols $R: (A)$ and $R: (A_1,A_2)$.
\end{remark}

We have now developed the machinery to state and prove the main theorem of
\cref{sec:logic}.

\begin{theorem}
  For every small bicategory of relations $\cat{B}$, there is an equivalence
  of categories
  \begin{equation*}
    \Cl(\Lang(\cat{B})) \simeq \cat{B}
    \qquad\text{in}\qquad \BiRel.
  \end{equation*}
\end{theorem}
\begin{proof}
  To prove the equivalence, it suffices to construct a structure-preserving
  functor $F: \Cl(\Lang(\cat{B})) \to \cat{B}$ that is full, faithful, and
  essentially surjective on objects \cite[Theorem 1.5.9]{riehl2016}. Define the
  functor $F$ on objects by
  \begin{equation*}
    F(\tvec{A}) := \sem{\tvec{A}} = \sem{A_1} \otimes \cdots \otimes \sem{A_n},
  \end{equation*}
  where $\tvec{A} = (A_1,\dots,A_n)$ and each $A_i$ is a basic type of
  $\Lang(\cat{B})$. If $\Gamma = \tvec{x}: \tvec{A}$ is a context, we also write
  $F(\Gamma) := \sem{\Gamma}$. Define $F$ on morphisms $[\Gamma; \Gamma' \given
  \varphi]: [\Gamma] \to [\Gamma']$ by
  \begin{equation*}
    F([\Gamma; \Gamma' \given \varphi]) := \sem{\Gamma; \Gamma' \given \varphi}:
      F(\Gamma) \to F(\Gamma').
  \end{equation*}
  By the construction of the classifying category and the internal language, we
  have the fundamental equivalence
  \begin{align*}
    [\Gamma;\Gamma' \given \varphi] \Rightarrow
      [\Gamma;\Gamma' \given \psi] \ \text{in}\ \Cl(\Lang(\cat{B}))
    &\qquad\text{iff}\qquad
    \Gamma,\Gamma' \given \varphi \vdash_{\Lang(\cat{B})} \psi \\
    &\qquad\text{iff}\qquad
    \sem{\Gamma;\Gamma' \given \varphi} \Rightarrow
      \sem{\Gamma;\Gamma' \given \psi}\ \text{in}\ \cat{B}.
  \end{align*}
  In particular, the functor $F$ is well-defined and faithful. It is full
  because if $R: A \to B$ is a morphism of $\cat{B}$, then there exists a
  relation symbol $R:(A,B)$ of $\Lang(\cat{B})$ such that
  $F(\sem{x:A;y:B \given R(x,y)}) = R$. Clearly, $F$ is (essentially) surjective
  on objects.
  
  It remains to prove that $F$ is a structure-preserving functor. The
  fundamental equivalence says that $F$ preserves 2-morphisms. We must show that
  $F$ also preserves composition, identities, monoidal products, and all the
  other structures of a bicategory of relations. We prove that $F$ preserves
  composition and products of morphisms and omit the other straightforward
  verifications. First, given morphisms $[\tvec{x}:\tvec{A}; \tvec{y}:\tvec{B}
  \given \varphi]$ and $[\tvec{y}:\tvec{B}; \tvec{z}:\tvec{C} \given \psi]$ of
  $\Cl(\Lang(\cat{B}))$, calculate
  \begin{align*}
    & F([\tvec{x}:\tvec{A}; \tvec{y}:\tvec{B} \given \varphi] \cdot
        [\tvec{y}:\tvec{B}; \tvec{z}:\tvec{C} \given \psi]) \\
    &\qquad= F([\tvec{x}:\tvec{A}; \tvec{z}:\tvec{C} \given
              \exists \tvec{y}:\tvec{B}.(\varphi \wedge \psi)]) \\
    &\qquad= \sem{\tvec{x}:\tvec{A}; \tvec{z}:\tvec{C} \given
              \exists \tvec{y}:\tvec{B}.(\varphi \wedge \psi)} \\
    &\qquad= (1_{\sem{\tvec{A}}} \otimes \eta_{\sem{\tvec{C}}})
             \sem{\tvec{x}:\tvec{A},\tvec{z}:\tvec{C}; \tvec{y}:\tvec{B} \given
                  \varphi \wedge \psi} \Delete_{\sem{\tvec{B}}} \\
    &\qquad= (1_{\sem{\tvec{A}}} \otimes \eta_{\sem{\tvec{C}}})
             (\sem{\tvec{x}:\tvec{A}; \tvec{y}:\tvec{B} \given \varphi} \otimes
              \sem{\tvec{z}:\tvec{C}; \tvec{y}:\tvec{B} \given \psi})
             \Merge_{\sem{\tvec{B}}} \Delete_{\sem{\tvec{B}}} \\
    &\qquad= (1_{\sem{\tvec{A}}} \otimes \eta_{\sem{\tvec{C}}})
             (\sem{\tvec{x}:\tvec{A}; \tvec{y}:\tvec{B} \given \varphi} \otimes
              \sem{\tvec{y}:\tvec{B}; \tvec{z}:\tvec{C} \given \psi}^\dagger)
             \epsilon_{\sem{\tvec{B}}} \\
    &\qquad= \sem{\tvec{x}:\tvec{A}; \tvec{y}:\tvec{B} \given \varphi} \cdot
             \sem{\tvec{y}:\tvec{B}; \tvec{z}:\tvec{C} \given \psi} \\
    &\qquad = F([\tvec{x}:\tvec{A}; \tvec{y}:\tvec{B} \given \varphi]) \cdot
              F([\tvec{y}:\tvec{B}; \tvec{z}:\tvec{C} \given \psi]).
  \end{align*}
  Given morphisms $[\tvec{x}:\tvec{A}; \tvec{y}:\tvec{B} \given \varphi]$ and
  $[\tvec{z}:\tvec{C}; \tvec{w}:\tvec{D} \given \psi]$, calculate
  \begin{align*}
    & F([\tvec{x}:\tvec{A}; \tvec{y}:\tvec{B} \given \varphi] \otimes
        [\tvec{z}:\tvec{C}; \tvec{w}:\tvec{D} \given \psi]) \\
    &\qquad= F([\tvec{x}:\tvec{A},\tvec{z}:\tvec{C}; 
                \tvec{y}:\tvec{B},\tvec{w}:\tvec{D} \given
                \varphi \wedge \psi]) \\
    &\qquad= \sem{\tvec{x}:\tvec{A},\tvec{z}:\tvec{C}; 
                  \tvec{y}:\tvec{B},\tvec{w}:\tvec{D} \given
                  \varphi \wedge \psi]} \\
    &\qquad= \Copy_{(\tvec{A},\tvec{C})}
             (\sem{\tvec{x}:\tvec{A},\tvec{z}:\tvec{C}; 
                   \tvec{y}:\tvec{B},\tvec{w}:\tvec{D} \given \varphi} \otimes
              \sem{\tvec{x}:\tvec{A},\tvec{z}:\tvec{C}; 
                   \tvec{y}:\tvec{B},\tvec{w}:\tvec{D} \given \psi})
             \Merge_{(\tvec{B},\tvec{D})} \\
    &\qquad= (\Copy_{\tvec{A}} \otimes \Copy_{\tvec{C}})
             (1_{\tvec{A}} \otimes \sigma_{\tvec{A},\tvec{C}} \otimes 1_{\tvec{C}}) \\
      &\qquad\phantom{{}={}}
          (\sem{\tvec{x}:\tvec{A}; \tvec{y}:\tvec{B} \given \varphi} \otimes
                \Delete_{\tvec{C}} \Create_{\tvec{D}} \otimes
                \Delete_{\tvec{A}} \Create_{\tvec{B}} \otimes
                \sem{\tvec{z}:\tvec{C}; \tvec{w}:\tvec{D} \given \psi}) \\
      &\qquad\phantom{{}={}}
          (1_{\tvec{B}} \otimes \sigma_{\tvec{D},\tvec{B}} \otimes 1_{\tvec{D}})
          (\Merge_{\tvec{B}} \otimes \Merge_{\tvec{D}}) \\
    &\qquad= (\Copy_{\tvec{A}} \otimes \Copy_{\tvec{C}})
             (\sem{\tvec{x}:\tvec{A}; \tvec{y}:\tvec{B} \given \varphi} \otimes
                   \Delete_{\tvec{A}} \Create_{\tvec{B}} \otimes
                   \Delete_{\tvec{C}} \Create_{\tvec{D}} \otimes
                   \sem{\tvec{z}:\tvec{C}; \tvec{w}:\tvec{D} \given \psi})
             (\Merge_{\tvec{B}} \otimes \Merge_{\tvec{D}}) \\
    &\qquad= \sem{\tvec{x}:\tvec{A}; \tvec{y}:\tvec{B} \given \varphi} \otimes
             \sem{\tvec{z}:\tvec{C}; \tvec{w}:\tvec{D} \given \psi} \\
    &\qquad= F([\tvec{x}:\tvec{A}; \tvec{y}:\tvec{B} \given \varphi]) \otimes
             F([\tvec{z}:\tvec{C}; \tvec{w}:\tvec{D} \given \psi]).
  \end{align*}
  The first calculation is obscured by the wire bending needed to unpack the
  definitions, but becomes transparent when rendered as a string diagram.
\end{proof}

\begin{remark}
  The inverse functor $G: \cat{B} \to \Cl(\Lang(\cat{B}))$ in the equivalence is
  defined on objects by $G(A) := A$ and on morphisms $R: A \to B$ by $G(R) :=
  [x:A; y:B \given R(x,y)]$. As an alternate proof, it is possible to explicitly
  construct the natural isomorphisms $F \cdot G \cong 1_{\Cl(\Lang(\cat{B}))}$
  and $G \cdot F \cong 1_{\cat{B}}$.
\end{remark}

\section{Coherent logic and distributive bicategories of relations} \label{app:coherent}

In this appendix we prove the main theorem of \cref{sec:logic2}, establishing a
correspondence between coherent logic with product and sum types and
distributive bicategories of relations. We present the logical system carefully
because we cannot find a comparable logic in the literature. However, where the
proof overlaps with \cref{app:regular}, we provide less detail.

We first define a formal system for coherent logic with product and sum types.
For the sake of brevity, we refer to this system simply as ``coherent logic.''
We maintain the syntactic conventions of \cref{app:regular}, including the
vector notation. Our proof system is an amalgamation of the usual proof systems
for typed lambda calculus \cite{jacobs1999,fiore2006} and coherent logic
\cite{johnstone2002} but with several important differences that we shall point
out.

\begin{definition}[cf.\ \cite{jacobs1995,jacobs1999}]
  A \emph{distributive signature} consists of
  \begin{itemize}
    \item a set of \emph{basic types}, which generates a set of \emph{types}
    according to the BNF grammar
    \begin{equation*}
      A,B ::= C\ |\ A \times B\ |\ 1\ |\ A + B\ |\ 0,
    \end{equation*}
    where $C$ ranges over basic types;
    
    \item a set of \emph{function symbols} $f,g,h,\dots$, each with fixed domain
    type $A$ and codomain type $B$, written $f: A \to B$;
    
    \item a set of \emph{relation symbols} $R,S,\dots$, each with fixed domain
    types $A$ and $B$, written $R: (A,B)$.
  \end{itemize}
\end{definition}

Variables and contexts are defined as in \cref{app:regular}. With respect to a
fixed distributive signature, a \emph{term in context} is an expression of form
$\Gamma \given t:A$, as defined inductively by the formation rules in
\cref{fig:coherent-formation-terms}. Likewise, a \emph{formula in context} is an
expression $\Gamma \given \varphi$ defined by the formation rules in
\cref{fig:coherent-formation-formulas}.

\begin{figure}
  \centering
  {\def\arraystretch{2.5}\tabcolsep=0ex
  \begin{tabular}{cr}
    $\Gamma, x:A, \Gamma' \given x:A$
      & (variable) \\
    $\vcenter{\prftree{\Gamma \given t: A}{\Gamma \given f(t): B}}$
      & (function symbol $f: A \to B$) \\
    $\vcenter{\prftree{\Gamma \given t:A}{\Gamma \given s:B}
                      {\Gamma \given \pair{t}{s}: A \times B}}$
      & (pair) \\
    $\vcenter{\prftree{\Gamma \given t: A \times B}
                      {\Gamma \given \pi_1(t): A}}
     \qquad
     \vcenter{\prftree{\Gamma \given t: A \times B}
                       {\Gamma \given \pi_2(t): B}}$
      & (projection) \\
    $\Gamma \given *: 1$
      & (singleton) \\
    $\vcenter{\prftree{\Gamma \given t:A+B}
                      {\Gamma, x:A \given r:C}{\Gamma, y:B \given s:C}
                      {\Gamma \given \copair{t}{x:A.r}{y:B.s}: C}}$
      & (case) \\
    $\vcenter{\prftree{\Gamma \given t: A}
                      {\Gamma \given \iota_1(t): A+B}}
     \qquad
     \vcenter{\prftree{\Gamma \given t: B}
                       {\Gamma \given \iota_2(t): A+B}}$
      & (inclusion)
  \end{tabular}}
  \caption{Formation rules for terms of coherent logic}
  \label{fig:coherent-formation-terms}
\end{figure}

\begin{figure}
  \centering
  {\def\arraystretch{2.5}\tabcolsep=4ex
  \begin{tabular}{cr}
    $\vcenter{\prftree{\Gamma \given t: A}{\Gamma \given s:B}
                      {\Gamma \given R(t,s)}}$
      & (relation symbol $R: (A,B)$) \\
    $\vcenter{\prftree{\Gamma \given t:A}{\Gamma \given s:A}
                      {\Gamma \given t=s}}$
      & (equality) \\
    $\vcenter{\prftree{\Gamma \given \varphi}{\Gamma, x:A \given \varphi}}$
      & (weakening) \\
    $\Gamma \given \top$
      & (truth) \\
    $\Gamma \given \bot$
      & (falsity) \\
    $\vcenter{\prftree{\Gamma \given \varphi}{\Gamma \given \psi}
                      {\Gamma \given \varphi \wedge \psi}}$
      & (conjunction) \\
    $\vcenter{\prftree{\Gamma \given \varphi}{\Gamma \given \psi}
                      {\Gamma \given \varphi \vee \psi}}$
      & (disjunction) \\
    $\vcenter{\prftree{\Gamma, x:A \given \varphi}
                      {\Gamma \given \exists x:A.\varphi}}$
      & (existential quantifier)
  \end{tabular}}
  \caption{Formation rules for formulas of coherent logic}
  \label{fig:coherent-formation-formulas}
\end{figure}

A \emph{sequent} is an expression $\Gamma \given \varphi \vdash \psi$, where
$\Gamma \given \varphi$ and $\Gamma \given \psi$ are both formulas in context.
The sequent $\Gamma \vdash \varphi$ is shorthand for $\Gamma \given \top \vdash
\varphi$. The inference rules for sequents are listed in
\cref{fig:coherent-inference}. As before, we omit the context when it is the
same in the premises and conclusion of a rule.

\begin{remark}
  We comment on the inference rules that differ from both first-order logic and
  equational type theory. Most importantly, product types are treated exactly as
  in type theory but sum types have a stronger axiomatization.
  \begin{itemize}
    \item \emph{Distributivity}: Like the Frobenius rule, the distributivity
    axiom linking conjunction and disjunction is superfluous in full first-order
    logic but is \emph{not} deducible from the other rules of coherent logic
    \cite[p.\ 831]{johnstone2002}. The converse of the distributivity axiom
    \emph{is} deducible.
  
    \item \emph{Case}: Our case rules cannot be expressed in simple type theory.
    Instead, type theory postulates
    \begin{equation*}
      \vdash \copair{t}{x:A.s[\iota_1(x)/z]}{y:B.s[\iota_2(y)/z]} = s[t/z].
    \end{equation*}
    Given the other axioms, this rule can be deduced from our first case rule
    but not conversely.
    
    \item \emph{Empty}: In simple type theory, the empty type is axiomatized by
    an \emph{empty function} $\bot_A: A$ obeying the formation and inference
    rules
    \begin{equation*}
      \prftree{t:0}{\bot_A(t): A}
      \qquad\qquad
      \prftree{t:0}{s:A}{\vdash \bot_A(t) = s}.
    \end{equation*}
    In particular, one can derive $\Gamma, x:0 \given r=s$ for any terms $r,s$.
    We strengthen that result to $\Gamma, x:0 \given \bot$ and discard the empty
    function $\bot_A$.
  \end{itemize}
\end{remark}

\begin{figure}
  \centering
  {\def\arraystretch{3}\tabcolsep=0ex
  \begin{tabular}{cr}
    $\varphi \vdash \varphi$
      & (identity) \\
    $\vcenter{\prftree{\varphi \vdash \chi}{\chi \vdash \psi}
                      {\varphi \vdash \psi}}$
      & (cut) \\
    $\vcenter{\prftree
      {\tvec{x}:\tvec{A} \given \varphi \vdash \psi}
      {\Gamma \given \tvec{t}: \tvec{A}}
      {\Gamma \given \varphi[\tvec{t}/\tvec{x}] \vdash \psi[\tvec{t}/\tvec{x}]}}$
      & (substitution) \\
    $\vdash x=x
     \qquad
     (\tvec{x}=\tvec{y}) \wedge \varphi \vdash \varphi[\tvec{y}/\tvec{x}]
     \qquad
     \tvec{x}=\tvec{y} \vdash t = t[\tvec{y}/\tvec{x}]$
      & (equality) \\
    $\vdash \pi_1(\pair{t}{s}) = t
     \qquad
     \vdash \pi_2(\pair{t}{s}) = s$
      & (projection) \\
    $\vdash \pair{\pi_1(t)}{\pi_2(t)} = t$
      & (pair) \\
    $\vcenter{\prftree[noline]
      {\vdash \copair{\iota_1(t)}{x:A.r}{y:B.s} = r[t/x]}
      {\vdash \copair{\iota_2(t)}{x:A.r}{y:B.s} = s[t/y]}}$
      & (inclusion) \\
    $\vcenter{\prftree[noline]
      {\Gamma \given \vdash
       (\exists x:A. \iota_1(x) = t) \vee (\exists x:B. \iota_2(x) = t)}
      {\Gamma \given
       (\exists x:A. \iota_1(x) = t) \wedge (\exists x:B. \iota_2(x) = t)
       \vdash \bot}}
      \qquad [x \notin \Gamma]$
      & (case) \\
    $\vcenter{\prftree{\Gamma \given t:1}{\Gamma \given \vdash t = *}}
     \qquad
     \Gamma, x:0 \given \vdash \bot$
      & (singleton \& empty) \\
    $\varphi \vdash \top
     \qquad
     \bot \vdash \varphi$
      & (truth \& falsity) \\
    $\vcenter{\prftree{\varphi \vdash \psi}{\varphi \vdash \chi}
                      {\varphi \vdash \psi \wedge \chi}}
     \qquad
     \varphi \wedge \psi \vdash \varphi
     \qquad
     \varphi \wedge \psi \vdash \psi$
      & (conjunction) \\
    $\vcenter{\prftree{\varphi \vdash \chi}{\psi \vdash \chi}
                      {\varphi \vee \psi \vdash \chi}}
     \qquad
     \varphi \vdash \varphi \vee \psi
     \qquad
     \psi \vdash \varphi \vee \psi$
      & (disjunction) \\
    $\varphi \wedge (\psi \vee \chi)
       \vdash (\varphi \wedge \psi) \vee (\varphi \wedge \chi)$
      & (distributivity) \\
    $\vcenter{\prftree[double]{\Gamma, x:A \given \varphi \vdash \psi}
                      {\Gamma \given (\exists x:A.\varphi) \vdash \psi}}$
      & (existential quantifier) \\
    $\Gamma \given \varphi \wedge (\exists x:A.\psi)
       \vdash \exists x:A.(\varphi \wedge \psi)
       \qquad [x \notin \Gamma]$
      & (Frobenius)
  \end{tabular}}
  \caption{Inference rules for coherent logic}
  \label{fig:coherent-inference}
\end{figure}

A \emph{coherent theory} is defined analogously to a regular theory.

\begin{definition}
  The \emph{classifying category} of a coherent theory $\theory{T}$, denoted
  $\Cl(\theory{T})$, is the distributive bicategory of relations whose objects
  are the types of $\theory{T}$ (basic or compound); whose morphisms $A \to B$
  are equivalence classes of formulas in context $[x:A,y:B \given \varphi]$,
  where equivalence is up to $\alpha$-equivalence and provable logical
  equivalence under $\theory{T}$; and whose 2-morphisms are the theorems of
  $\theory{T}$.

  The structures of a distributive bicategory of relations are defined as
  follows.
  \begin{itemize}
    \item \emph{Category}: Composition is defined by
    \begin{equation*}
      [x:A,y:B \given \varphi] \cdot [y:B,z:C \given \psi]
        := [x:A,z:C \given \exists y:B.(\varphi \wedge \psi)]
    \end{equation*}
    and the identity morphisms are $1_A := [x:A,x':A \given x=x']$.
    
    \item \emph{Tensor}: On objects, $A \otimes B := A \times B$ and $I := 1$;
    on morphisms,
    \begin{align*}
      [x_1:A_1,& y_1:B_1 \given \varphi] \otimes [x_2:A_2, y_2:B_2 \given \psi] \\
       :=&\ [x:A_1 \times A_2, y: B_1 \times B_2 \given \\
         &\quad \exists x_1:A_1.\exists y_1:B_1.
                   (\varphi \wedge \pi_1 x = x_1 \wedge \pi_1 y = y_1) \wedge \\
         &\quad \exists x_2:A_2.\exists y_2:B_2.
                   (\psi \wedge \pi_2 x = x_2 \wedge \pi_2 y = y_2)] \\
         =&\ [x:A_1 \times A_2, y: B_1 \times B_2 \given
          \varphi[\pi_1 x / x_1,\pi_1 y / y_1] \wedge
          \psi[\pi_2 x / x_2,\pi_2 y / y_2]].
    \end{align*}
    The braidings are
    \begin{equation*}
      \sigma_{A,B} := [x: A \times B, y: B \times A \given
        (\pi_1 x = \pi_2 y) \wedge (\pi_2 x = \pi_1 y)].
    \end{equation*}
    
    \item \emph{Cotensor}: On objects, $A \oplus B := A+B$ and $O := 0$;
    on morphisms,
    \begin{align*}
      [x_1:A_1,& y_1:B_1 \given \varphi] \oplus [x_2:A_2, y_2:B_2 \given \psi] \\
        :=&\ [x:A_1 + A_2, y: B_1 + B_2 \given \\
         &\qquad \exists x_1:A_1.\exists y_1:B_1.
                   (\varphi \wedge \iota_1 x_1 = x \wedge \iota_1 y_1 = y) \vee \\
         &\qquad \exists x_2:A_2.\exists y_2:B_2.
                   (\psi \wedge \iota_2 x_2 = x \wedge \iota_2 y_2 = y)].
    \end{align*}
    The braidings are
    \begin{align*}
      \bm{\sigma}_{A,B} &:= [x: A+B, y: B+A \given \\
        &\qquad \exists x':A. \exists y':B.(
          (\iota_1 x' = x \wedge \iota_2 x' = y) \vee
          (\iota_2 y' = x \wedge \iota_1 y' = y))].
    \end{align*}
    
    \item \emph{Diagonals and codiagonals}:
    \begin{align*}
      \Copy_A &:=
        [x:A, y:A \times A \given \pair{x}{x} = y] \\
      \Delete_A &:= [x:A, y:1 \given \top] \\
      \varMerge_A &:=
        [x:A+A, y:A \given \copair{x}{x_1:A.x_1}{x_2:A.x_2} = y] \\
      \varCreate_A &:= [x:0, y:A \given \bot].
    \end{align*}
  \end{itemize}
\end{definition}

\begin{lemma}
  The classifying category $\Cl(\theory{T})$ of a coherent theory $\theory{T}$
  is a distributive bicategory of relations.
\end{lemma}

The proof of the lemma is monstrously long but similar in many respects to the
completeness proof in \cref{app:regular}. To exemplify the new considerations
posed by product and sum types, we prove a different, ``obvious'' fact about
local unions in the classifying category. We expect this will exhaust the
reader's appetite for such calculations. We then proceed directly to the
construction of the internal language of a distributive bicategory of relations.

\begin{proposition}
  Local unions in the classifying category of a coherent theory are given by
  logical disjunction:
  \begin{equation*}
    [x:A,y:B \given \varphi] \cup [x:A,y:B \given \psi]
      = [x:A,y:B \given \varphi \vee \psi].
  \end{equation*}
\end{proposition}
\begin{proof}
  In an abuse of notation, we write $[\varphi] := [x:A,y:B \given \varphi]$ and
  $[\psi] := [x:A,y:B \given \psi]$. By definition,
  \begin{equation*}
    [\varphi] \cup [\psi] = \varCopy_A ([\varphi] \oplus [\psi]) \varMerge_B.
  \end{equation*}
  First, use the inclusion rules and the first case rule to show that
  \begin{gather*}
    \varCopy_A = [x:A, y:A+A \given (\iota_1 x = y) \vee (\iota_2 x = y)] \\
    \varMerge_B = [x:B+B, y:B \given (x = \iota_1 y) \vee (x = \iota_2 y)].
  \end{gather*}
  Therefore,
  \begin{align*}
    ([\varphi] \oplus [\psi]) \varMerge_B
      = [&x':A+A, y:B \given \exists y':B+B.( \\
        & ((\exists x:A. \exists \tilde y:B.(
            \varphi[\tilde y/y] \wedge
            \iota_1 x = x' \wedge \iota_1 \tilde y = y')) \vee \\
        & \phantom{(}(\exists x:A. \exists \tilde y:B.(
            \psi[\tilde y/y] \wedge
            \iota_2 x = x' \wedge \iota_2 \tilde y = y'))) \\
        & \wedge (y' = \iota_1 y \vee y' = \iota_2 y) )].
  \end{align*}
  Use the Frobenius rule and the distributivity of existential quantifiers over
  disjunctions to put the formula into prenex normal form, then distribute the
  conjunction over the disjunctions to put the body into disjunctive normal
  form:
  \begin{align*}
    ([\varphi] \oplus [\psi]) \varMerge_B
      = [&x':A+A, y:B \given \exists x:A. \exists \tilde y:B. \exists y':B+B.( \\
        & (\varphi[\tilde y/y] \wedge \iota_1 x = x' \wedge
           \iota_1 \tilde y = y' \wedge \iota_1 y = y') \vee \\
        & (\varphi[\tilde y/y] \wedge \iota_1 x = x' \wedge
           \iota_1 \tilde y = y' \wedge \iota_2 y = y') \vee \\
        & (\psi[\tilde y/y] \wedge \iota_2 x = x' \wedge
           \iota_2 \tilde y = y' \wedge \iota_1 y = y') \vee \\
        & (\psi[\tilde y/y] \wedge \iota_2 x = x' \wedge
           \iota_2 \tilde y = y' \wedge \iota_2 y = y') )].
  \end{align*}
  The first and last disjuncts are handled by the injectivity of the inclusions
  (deducible from the inclusion rules), e.g.,
  \begin{equation*}
    \iota_1 \tilde y = y' \wedge \iota_1 y = y'
      \vdash \iota_1 \tilde y = \iota_1 y
      \vdash \tilde y = y.
  \end{equation*}
  Eliminate the two cross terms using the second case rule, e.g.,
  \begin{equation*}
    \iota_1 \tilde y = y' \wedge \iota_2 y = y'
      \vdash \iota_1 \tilde y = \iota_2 y
      \vdash \bot.
  \end{equation*}
  Upon simplification the result is
  \begin{equation*}
    ([\varphi] \oplus [\psi]) \varMerge_B
      = [x':A+A, y:B \given \exists x:A.(
          (\varphi \wedge \iota_1 x = x') \vee
          (\psi \wedge \iota_2 x = x'))].
  \end{equation*}
  The second half of the calculation is very similar and yields
  \begin{equation*}
    [\varphi] \cup [\psi]
      = \varCopy_A ([\varphi] \oplus [\psi]) \varMerge_B
      = [x:A,y:B \given \varphi \vee \psi]. \qedhere
  \end{equation*}
\end{proof}

\begin{definition}
  An \emph{interpretation} or \emph{model} of a distributive signature in a
  distributive bicategory of relations $\cat{B}$ is specified by
  \begin{itemize}
    \item for every basic type $C$, an object $\sem{C}$ of $\cat{B}$;
    \item for every function symbol $f:A \to B$, a morphism $\sem{f}: \sem{A}
    \to \sem{B}$ of $\Map(\cat{B})$;
    \item for every relation symbol $R: (A,B)$, a morphism $\sem{R}: \sem{A} \to
    \sem{B}$ of $\cat{B}$.
  \end{itemize}
  The extension of the interpretation to any type of the signature is implicit
  in the definition:
  \begin{equation*}
    \sem{A \times B} = \sem{A} \otimes \sem{B}, \quad
    \sem{1} = I, \quad
    \sem{A+B} = \sem{A} \oplus \sem{B}, \quad
    \sem{0} = O.
  \end{equation*}
  We also maintain the convention that $\sem{\tvec{x}:\tvec{A}} :=
  \sem{\tvec{A}} := \sem{A_1} \otimes \cdots \otimes \sem{A_n}$.
  
  An interpretation of a distributive signature extends to the full coherent
  language in that signature by induction on the term and formula formation
  rules. Each term in context $\Gamma \given t:A$ is interpreted as a morphism
  $\sem{\Gamma \given t}: \sem{\Gamma} \to \sem{A}$ of $\Map(\cat{B})$ as
  follows.
  \begin{itemize}
    \item \emph{Variable}: $\sem{\Gamma,x:A,\Gamma' \given x:A} := 
      \Delete_{\sem{\Gamma}} \otimes 1_{\sem{A}} \otimes \Delete_{\sem{\Gamma'}}$.
      
    \item \emph{Function symbol}: Given a function symbol $f: A \to B$ and a map
    $\sem{\Gamma \given t: A}$,
    \begin{equation*}
      \sem{\Gamma \given f(t): B} := \sem{\Gamma \given t: A} \cdot \sem{f}.
    \end{equation*}
    
    \item \emph{Pair}: Given maps $\sem{\Gamma \given t:A}$ and
    $\sem{\Gamma \given s:B}$,
    \begin{equation*}
      \sem{\Gamma \given \pair{t}{s}: A \times B} :=
        \Copy_{\sem{\Gamma}}
        (\sem{\Gamma \given t:A} \otimes \sem{\Gamma \given s:B}).
    \end{equation*}
        
    \item \emph{Projection}: Given a map $\sem{\Gamma \given t: A \times B}$,
    \begin{align*}
      \sem{\Gamma \given \pi_1(t): A}
        &:= \sem{\Gamma \given t: A \times B} 
            (1_{\sem{A}} \otimes \Delete_{\sem{B}}) \\
      \sem{\Gamma \given \pi_2(t): B}
        &:= \sem{\Gamma \given t: A \times B}
            (\Delete_{\sem{A}} \otimes 1_{\sem{B}}).
    \end{align*}
    
    \item \emph{Singleton}: $\sem{\Gamma \given *:I} := \Delete_{\sem{\Gamma}}$.
    
    \item \emph{Case}: Given maps $\sem{\Gamma \given t: A+B}$,
    $\sem{\Gamma, x:A \given r:C}$, and $\sem{\Gamma, y:B \given s:C}$, define
    $\sem{\Gamma \given \copair{t}{x:A.r}{y:B.s}: C}$ to be
    \begin{equation*}
      \Copy_{\sem{\Gamma}} \cdot
      (1_{\sem{\Gamma}} \otimes \sem{\Gamma \given t: A+B}) \cdot
      d_{\sem{\Gamma},\sem{A},\sem{B}} \cdot
      (\sem{\Gamma, x:A \given r} \oplus \sem{\Gamma, y:B \given s}) \cdot
      \varMerge_{\sem{C}},
    \end{equation*}
    where $d_{X,Y,Z}: X \otimes (Y \oplus Z) \to (X \otimes Y) \oplus (X \otimes
    Z)$ is the family of distributivity isomorphisms in $\cat{B}$.
    
    \item \emph{Inclusion}: Given maps $\sem{\Gamma \given t:A}$ and
    $\sem{\Gamma \given s:B}$,
    \begin{align*}
      \sem{\Gamma \given \iota_1(t): A+B}
        &:= \sem{\Gamma \given t:A} \oplus \varCreate_B \\
      \sem{\Gamma \given \iota_2(s): A+B}
        &:= \varCreate_A \oplus \sem{\Gamma \given s:B}.
    \end{align*}
  \end{itemize}
  
  Each formula in context $\Gamma;\Gamma' \given \varphi$ is interpreted as a
  morphism $\sem{\Gamma;\Gamma' \given \varphi}: \sem{\Gamma} \to
  \sem{\Gamma'}$ of $\cat{B}$ as follows.
  \begin{itemize}
    \item \emph{Relation symbol}: Given a relation symbol $R: (A,B)$ and two
    maps $\sem{\Gamma \given t:A}$ and $\sem{\Gamma \given s:B}$, define
    \begin{equation*}
      \sem{\Gamma \given R(t,s)} :=
        \Copy_{\sem{\Gamma}}
        (\sem{\Gamma \given t: A} \sem{R} \otimes \sem{\Gamma \given s: B})
        \epsilon_{\sem{B}}.
    \end{equation*}
    where, as usual, $\epsilon_{\sem{B}} = \Merge_{\sem{B}} \Delete_{\sem{B}}$.
    
    \item \emph{Equality}: Given maps $\sem{\Gamma \given t:A}$ and
    $\sem{\Gamma \given s:A}$,
    \begin{equation*}
      \sem{\Gamma \given t=s} :=
        \Copy_{\sem{\Gamma}}
        (\sem{\Gamma \given t:A} \otimes \sem{\Gamma \given s:A})
        \epsilon_{\sem{A}}
    \end{equation*}
    
    \item \emph{Falsity}: $\sem{\Gamma;\Gamma' \given \bot} :=
      \bot_{\sem{\Gamma},\sem{\Gamma'}} =
      \varDelete_{\sem{\Gamma}} \varCreate_{\sem{\Gamma'}}$.
    
    \item \emph{Disjunction}: Given morphisms
    $\sem{\Gamma;\Gamma' \given \varphi}: \sem{\Gamma} \to \sem{\Gamma'}$ and
    $\sem{\Gamma;\Gamma' \given \psi}: \sem{\Gamma} \to \sem{\Gamma'}$,
    \begin{equation*}
      \sem{\Gamma;\Gamma' \given \varphi \vee \psi}
        := \varCopy_{\sem{\Gamma}} \cdot
           (\sem{\Gamma;\Gamma' \given \varphi} \oplus 
            \sem{\Gamma;\Gamma' \given \psi}) \cdot
           \varMerge_{\sem{\Gamma'}}.
    \end{equation*}
  \end{itemize}
  The other formation rules---weakening, truth, conjunction, and
  existential quantifier---are interpreted exactly as in \cref{app:regular}.
\end{definition}

An \emph{interpretation} or \emph{model} of a coherent theory in a distributive
bicategory of relations is defined analogously to an interpretation of a regular
theory in a bicategory of relations. There is also a corresponding soundness
theorem.

\begin{lemma}
  Let $\sem{\cdot}$ be an interpretation of a coherent theory $\theory{T}$ in a
  distributive bicategory of relations $\cat{B}$. For every theorem
  \begin{equation*}
    \Gamma \given \varphi \vdash_{\theory{T}} \psi
    \quad\text{of}\quad \theory{T},
  \end{equation*}
  there is a 2-morphism
  \begin{equation*}
    \sem{\Gamma \given \varphi} \implies \sem{\Gamma \given \psi}
    \quad\text{in}\quad \cat{B}.
  \end{equation*}
\end{lemma}
\begin{proof}
  As before, we must show that every inference rule of coherent logic
  (\cref{fig:coherent-inference}) is valid in $\cat{B}$. We sketch the proofs
  for the rules that have not already been treated in \cref{app:regular}.
  \begin{itemize}
    \item \emph{Projection}: The first projection axiom
    $\vdash \pi_1\pair{t}{s} = t$ holds because
    \begin{equation*}
      \Copy_X (f \otimes g) \pi_{Y,Z}
        = \Copy_X (f \otimes g) (1_Y \otimes \Delete_Z)
        = \Copy_X (f \otimes g \Delete_Z)
        = \Copy_X (f \otimes \Delete_X)
        = f
    \end{equation*}
    whenever $f: X \to Y$ and $g: X \to Z$ are maps. The second projection axiom
    $\vdash \pi_2\pair{t}{s} = s$ is proved similarly.
    
    \item \emph{Pair}: The pair axiom 
    $\vdash \pair{\pi_1(t)}{\pi_2(t)} = t$ amounts to the equality
    \begin{align*}
      \Copy_{X \otimes Y} (\pi_{X,Y} \otimes \pi_{X,Y}')
        &= (\Copy_X \otimes \Copy_Y) (1_X \otimes \sigma_{X,Y} \otimes 1_Y)
           (1_X \otimes \Delete_Y \otimes \Delete_X \otimes 1_Y) \\
        &= (\Copy_X \otimes \Copy_Y)
           (1_X \otimes \sigma_{X,Y} \Delete_{Y \otimes X} \otimes 1_Y) \\
        &= (\Copy_X \otimes \Copy_Y)
           (1_X \otimes \Delete_X \otimes \Delete_Y \otimes 1_Y) \\
        &= 1_X \otimes 1_Y = 1_{X \otimes Y}.
    \end{align*}
    
    \item \emph{Inclusion}: The two inclusion axioms are dual to the two
    projection axioms, e.g., $\iota_{X,Y} (f \oplus g) \varMerge_Z = f$ for any
    two morphisms $f: X \to Z$ and $g: Y \to Z$ (which need not be maps).
    
    \item \emph{Case}: The first case axiom
    $\vdash (\exists x:A. \iota_1 x = t) \vee (\exists x:B.\iota_2 x = t)$ 
    amounts to
    \begin{equation*}
      \Create_{X \oplus Y} =
        \varCopy_I
        (\Create_X \iota_{X,Y} \oplus \Create_Y \iota_{X,Y}')
        \varMerge_{X \oplus Y}.
    \end{equation*}
    By the dual of the pair axiom, $(\iota_{X,Y} \oplus \iota_{X,Y}')
    \varMerge_{X \oplus Y} = 1_{X \oplus Y}$, it suffices to show that
    $\Create_{X \oplus Y} = \varCopy_I (\Create_X \oplus \Create_Y)$, or 
    equivalently
    \begin{equation*}
      \Delete_{X \oplus Y}
        = (\Delete_X \oplus \Delete_Y) \varMerge_I
        =: [\Delete_X,\Delete_Y]
    \end{equation*}
    Because the inclusions are maps, we have $\iota_{X,Y} \Delete_{X \oplus Y} =
    \Delete_X$ and $\iota_{X,Y}' \Delete_{X \oplus Y} = \Delete_Y$, and hence
    the last equation holds by the universal property of the coproduct. We omit
    the proof of the second case axiom.
    
    \item \emph{Singleton}: The unit $I$ of the tensor is terminal in
    $\Map(\cat{B})$.
    
    \item \emph{Empty}: For any objects $X,Y$ of $\cat{B}$, there is at most one
    morphism $X \otimes O \to Y$, which must therefore be $\bot_{X \otimes
    O,Y}$. A similar result holds, with essentially the same proof, in a
    distributive category; we refer to \cite[Proposition 3.2]{carboni1993}.
    
    \item \emph{Falsity}: Dual to truth.
    
    \item \emph{Disjunction}: Dual to conjunction.
    
    \item \emph{Distributivity}: As mentioned in \cref{sec:DistBiRel}, the
    distributive law
    \begin{equation*}
      R \cap (S \cup T) = (R \cap S) \cup (R \cap T)
    \end{equation*}
    holds in any distributive bicategory of relations. The proof, which we omit,
    is a calculation involving the canonical distributivity isomorphism.
    \qedhere
  \end{itemize}
\end{proof}

\begin{definition}
  The \emph{internal language} of a small distributive bicategory of relations
  $\cat{B}$ is the coherent theory $\Lang(\cat{B})$ defined as follows. Its
  signature consists of
  \begin{itemize}
    \item for every object $A$ of $\cat{B}$, a basic type $A$;
    \item for every pair of types $(A,B)$ and every morphism $f: \sem{A} \to
    \sem{B}$ of $\Map(\cat{B})$, a function symbol $f: A \to B$; and
    \item for every pair of types $(A,B)$ and every morphism $R: \sem{A} \to
    \sem{B}$ of $\cat{B}$, a relation symbol $R:(A,B)$.
  \end{itemize}
  A sequent $\Gamma \given \varphi \vdash \psi$ is an axiom of $\Lang(\cat{B})$
  if and only if $\sem{\Gamma \given \varphi} \Rightarrow \sem{\Gamma \given
  \psi}$ in $\cat{B}$. Throughout the definition, $\sem{\cdot}$ is the obvious
  interpretation of the signature of $\Lang(\cat{B})$ in the category $\cat{B}$.
\end{definition}

\begin{remark}
  The expressivity of the internal language is not affected by including
  function symbols for the maps because every map is also associated with a
  relation symbol.

  As in \cref{app:regular}, a single morphism of $\cat{B}$ can give rise to many
  function and relation symbols of $\Lang(\cat{B})$. Moreover, despite replacing
  lists of types by compound types, it remains the case that $\Lang(\cat{B})$
  has ``too many'' types. For example, if $A = A_1 \otimes A_2$ is a product in
  $\cat{B}$, then both $A$ and $A_1 \times A_2$ are types in $\Lang(\cat{B})$.
  Although the types are isomorphic, they are formally distinct. This
  discrepancy explains why the equivalence of categories in the theorem below is
  not an isomorphism of categories.
\end{remark}

\begin{theorem}
  For every small distributive bicategory of relations $\cat{B}$, there is an
  equivalence of categories
  \begin{equation*}
    \Cl(\Lang(\cat{B})) \simeq \cat{B}
    \qquad\text{in}\qquad \DistBiRel.
  \end{equation*}
\end{theorem}
\begin{proof}
  It suffices to construct a structure-preserving functor $F:
  \Cl(\Lang(\cat{B})) \to \cat{B}$ that is full, faithful, and essentially
  surjective on objects. Define the functor $F$ on objects by $F(A) := \sem{A}$
  and on morphisms by
  \begin{equation*}
    F([x:A; y:B \given \varphi]) := \sem{x:A; y:B \given \varphi}: F(A) \to F(B).
  \end{equation*}
  The proof that $F$ is well-defined and has the requisite properties proceeds
  as in \cref{app:regular}. We leave the details to the reader.
\end{proof}

\end{document}